\documentclass[lettersize,journal]{IEEEtran} % Commen\textbf{\textbf{\textbf{}}}t this line out if you need a4paper
\pdfoutput=1
\usepackage{amsmath,amssymb}
\usepackage{hyperref}
\usepackage{cite}
\usepackage{graphicx}
\usepackage[caption=false,font=footnotesize]{subfig}
\usepackage{varwidth}
\graphicspath{ {./images/} }
\usepackage{afterpage}
\usepackage{placeins}
\usepackage{tikz}
\usetikzlibrary{fit}					% fitting shapes to coordinates
\usetikzlibrary{backgrounds}	% drawing the background after the foreground
\usetikzlibrary{arrows}
\usepackage{xcolor}
\usepackage{bm}
\usepackage{slashbox}
\usepackage{fancyvrb}
\usepackage{array}

\usepackage{enumitem}
\usepackage{soul,color}
\usepackage[ruled,vlined]{algorithm2e}
\usepackage{multirow}
\usepackage{adjustbox}
\usepackage[]{footmisc}
\usepackage{makecell}
\usepackage{textcomp}
\usepackage{lipsum}
\newcolumntype{?}{!{\vrule width 2pt}}
\newcolumntype{P}[1]{>{\centering\arraybackslash}p{#1}}
\usepackage{amsfonts,amsthm,cancel}
\newtheorem{proposition}{Proposition}
\newtheorem*{definition}{Definition}

\renewcommand{\Re}{\mathbb{R}}

\newcommand{\deltaX}[3]{\bm{p}_{#2}}
\newcommand{\deltaT}[3]{{\theta_{#2}}}

\newcommand{\Rot}[1]{R_{#1}}
\newcommand{\Trans}[1]{\bm{p}^{#1}}

\newcommand{\TMat}[2]{
\begin{bmatrix}
    #1 & #2\\
    \bm{0} & 1
    \end{bmatrix}
}

\newcommand{\Fig}{Fig. }
\newcommand{\Sec}{Section }
\newcommand{\Algo}{Algorithm }
\newcommand{\disp}[1]{\bm{p}_{#1}}
\newcommand{\Matlab}{MATLAB\textsuperscript{\textregistered}}
\newcommand{\Tab}{Table }
\newcommand{\TriSoRo}{TerreSoRo-III}
\newcommand{\TetraSoRo}{TerreSoRo-IV}

\newcommand{\Framework}{Model-Free Control Framework}

\newcommand\copyrighttext{%
\centering
  This paper has been accepted for publication in the \textit{IEEE Transactions on Robotics.}
  Please cite as: \\ C. Freeman, A.N. Mahendran, and V. Vikas, ``Environment-centric learning approach for gait synthesis in terrestrial soft robots," in \textit{IEEE Transactions on Robotics (T-RO)}, 2025.}
\newcommand\copyrightnotice{%
\begin{tikzpicture}[remember picture,overlay]
\node[anchor=north,yshift=-5pt] at (current page.north) {\fbox{\parbox{\dimexpr\textwidth-15\fboxsep-\fboxrule\relax}{\copyrighttext}}};
\end{tikzpicture}%
}
\newcommand\multimediatext{%
\centering
 The multimedia attachment can be accessed at \url{https://www.youtube.com/watch?v=ujxK8EWahjg}}
\newcommand\multimediaatt{%
\begin{tikzpicture}[remember picture,overlay]
\node[anchor=south,yshift=10pt] at (current page.south) {\fbox{\parbox{\dimexpr\textwidth-15\fboxsep-\fboxrule\relax}{\multimediatext}}};
\end{tikzpicture}%
}
\begin{document}

\title{Environment-Centric Learning Approach for Gait Synthesis in Terrestrial Soft Robots}

\author{Caitlin Freeman, Arun Niddish Mahendran, and Vishesh Vikas$^{1}$% <-this % stops a space
\thanks{*This work is supported by the National Science Foundation under Grant No. 1830432}% <-this % stops a space
\thanks{$^{1}$Caitlin Freeman, Arun Niddish Mahendran, and Vishesh Vikas are with the Agile Robotics Lab (ARL), University of Alabama, Tuscaloosa AL 36487, USA
        {\tt\small clfreeman@crimson.ua.edu, anmahendran@crimson.ua.edu, vvikas@ua.edu}}%
}

\maketitle
\copyrightnotice
\multimediaatt
\begin{abstract}
%Model-based approaches applied to soft locomotors face challenges pertaining to modeling their continuum nature and the robot-environment interaction. 
%Consequently, there is lack of methodical gait synthesis approaches, that are fundamental to their locomotion control.
% Model-based locomotion control for soft robots faces challenges pertaining to modeling their continuum nature and the robot-environment interaction. Consequently, it is difficult to methodically synthesize gaits, which are fundamental to their locomotion control. 
Locomotion gaits are fundamental for control of soft terrestrial robots. However, synthesis of these gaits is challenging due to modeling of robot-environment interaction and lack of a mathematical framework. This work presents an environment-centric, data-driven, and fault-tolerant probabilistic Model-Free Control (pMFC) framework that allows for soft multi-limb robots to learn from their environment and synthesize diverse sets of locomotion gaits for realizing open-loop control. Here, discretization of factors dominating robot-environment interactions enables an environment-specific graphical representation where the edges encode experimental locomotion data corresponding to the robot motion primitives. In this graph, locomotion gaits are defined as simple cycles that are transformation invariant, i.e., the locomotion is independent of the starting vertex of these periodic cycles. Gait synthesis, the problem of finding optimal locomotion gaits for a given substrate, is formulated as Binary Integer Linear Programming (BILP) problems with a linearized cost function, linear constraints, and iterative simple cycle detection. %Despite being NP-hard, this problem can be reliably solved using common branch-and-bound solvers, and results show that optimality is not significantly compromised. 
Experimentally, gaits are synthesized for varying robot-environment interactions. Variables include robot morphology - three-limb and four-limb robots, \TriSoRo~and \TetraSoRo; substrate - rubber mat, whiteboard and carpet; and actuator functionality - simulated loss of robot limb actuation.
 %Specifically, for different robot morphologies - three-limb and four-limb robots, \TriSoRo~and \TetraSoRo; different surfaces - rubber mat, whiteboard and carpet; and the scenario pertaining to loss of robot limb actuation.%For this scenario,  gait-synthesis does not require  re-learning of the graph, highlighting the fault-tolerant nature of the framework. 
%
%(three-limb and four-limb terrestrial soft robots, \TriSoRo~and \TetraSoRo, and three different surfaces) and in simulated loss of limb actuation. 
On an average, gait synthesis improves the translation and rotation speeds by 82\% and 97\% respectively. The results highlight that data-driven methods are vital to soft robot locomotion control due to complex robot-environment interactions and simulation-to-reality gaps, particularly when biological analogues are unavailable.

\end{abstract}


\section{Introduction}
% \hl{acknowledge thanks to Dr.Vela and Alex}
%\IEEEPARstart{I}{n} recent years, roboticists have been inspired by the ability of animals to leverage structural soft materials for terrestrial locomotion tasks. This has led to the design of soft material robots powered by a variety of actuators including pneumatics, shape memory alloys and motor-tendons\cite{kim_soft_2013,tolley_resilient_2014, vikas_design_2016}. Conceptually, robot locomotion results from combination of robot movements that alter its interaction with the environment \cite{radhakrishnan_locomotion:_1998, calisti_fundamentals_2017}. 
\IEEEPARstart{S}{oft} robots are promising for terrestrial locomotion applications given their inherent mechanical intelligence. This intelligence permits complex robot behavior and surprising adaptability in unstructured environments, even with relatively simple robot designs and controllers \cite{kim_soft_2013,tolley_resilient_2014, vikas_design_2016}. However, the locomotion of soft robots is significantly more challenging to model and control due to the complexity of their nonlinear kinematics and dynamics; this is further complicated by their greater dependence on environmental interactions \cite{environment}. Therefore, modeling of soft robot locomotion becomes difficult and computationally intensive, and significant simulation-to-reality gaps remain \cite{pinskier_bioinspiration_2022, faure_sofa:_2012}. This has motivated reduced-order modeling techniques, such as rigid link-based approximations with stick-slip contact transitions\cite{gamus_understanding_2020} and shape-centric geometric modeling with variable friction models \cite{chang_shape-centric_2021}. Alternatively, soft robots can be controlled without a kinematic model with methods ranging from simple trial-and-error approaches informed by intuition or biology \cite{doi:10.1073/pnas.1116564108} to more robust learning-based model-free algorithms \cite{umedachi_softworms:_2016}. 

Whether model-based or empirically derived, gaits remain the predominant basis of locomotion control for soft robots. Gaits are effective in reducing the control space to a smaller and tractable set of cyclic patterns that can be modulated and combined to realize complex tasks like path planning. In fact, multi-gait closed-loop path planning of one of the robots presented in this paper has recently been shown to be an effective locomotion control strategy \cite{Mahendran_IROS_2023}. However, the selection of soft robot locomotion gaits remains an open problem. 

This paper presents a generic data-driven approach for discovering a diverse set of open-loop gaits that can be used for the locomotion control of soft robots.
This method is particularly well-suited to soft robots for which there is a lack of biological analogue and/or kinematic model. The resulting open-loop gaits are useful not only as fundamental components of real-time soft robot locomotion path planning, but also as a vehicle to explore the degree to which the environment (e.g., substrate) and manufacturing inconsistencies affect the locomotion behavior of soft robots. This morphology-agnostic model-free framework overcomes challenges of current model-based methods by directly encoding robot-environment interactions, capturing difficult-to-model factors like stick-slip transitions and manufacturing inaccuracies. Furthermore, it addresses the fundamental problems relating to \textit{gait definition} (how to mathematically define a gait, e.g., crawl vs. inch), \textit{gait synthesis and exploration} (how to determine and evaluate the underlying useful gaits), and \textit{gait characterization} (what type of gait motion is produced, e.g., rotation or translation).
\subsection{Relation to Previous Work}
This research builds upon prior work on the \underline{M}odel-\underline{F}ree \underline{C}ontrol (MFC) framework\cite{vikas_model-free_2015, umedachi_softworms:_2016,  vikas_design_2016}. The MFC framework discretizes robot-environment frictional interactions and uses a data-driven approach to experimentally explore all possible motion primitives of a soft robot on a planar surface. Then, optimal gaits for locomotion along a line are identified by optimizing the translation of a sequence of motion primitives. In these previous papers, the MFC framework has been shown to be a robust gait synthesis method for soft robots with varying morphology (e.g., multi-limb vs. rectangular and symmetrical vs. asymmetrical), actuator type (motor \cite{vikas_model-free_2015} vs. shape memory alloy\cite{vikas_design_2016}), number of actuators (two vs. three and loss-of-limb scenarios), tendon routing and placement, and environment (carpet vs. hard floor). In these experiments, the optimal gait cycles varied depending on changing parameters such as different surface conditions and variations in the robot's morphology. Furthermore, voltage-varying experiments suggested that gait speed for a given synthesized gait pattern can be varied and optimized. However, the method in the current state does not scale to systems with more than three actuators, does not account for uncertainty/repeatability, and is not applicable to locomotion on a plane, i.e., $SE(2)$.

\subsection{Contributions} 
This research contains the following contributions.
\begin{enumerate}
    \item A new \textit{probabilistic MFC framework} (pMFC) is presented, which augments the previous work by constructing a more robust data collection process using the proposed stochastic Hierholzer's algorithm and probabilistic model of motion primitives. These changes account for transient effects and inherent variability, thereby enabling the selection of more predictable and precise gaits. 
    \item The graphical nature of the framework is used to define \textit{locomotion gaits}  in $SE(2)$ as periodic simple cycles that are transformation invariant, i.e., the rotation and translation of the simple cycle are the same irrespective of the starting vertex. It is proved that, under this definition, gaits can be fundamentally characterized as either translation gaits (where the cumulative rotation of the simple cycle vanishes) or rotation gaits (where all translations of the simple cycle edges are zero).
    \item \textit{Gait synthesis} is performed for both three-limb and four-limb motor-tendon actuated soft robots by optimizing a  weighted linear cost function formulated as a \underline{B}inary \underline{I}nteger \underline{L}inear \underline{P}rogramming (BILP) problem with linear constraints and iterative cycle detection. Despite being NP-hard, the problem is tractable and can be iteratively solved using linear programming relaxation sub-problems. Comparison of this approach with an exhaustive search of the nonlinear cost function demonstrates its \textit{scalabilty and near-optimality}. 
    \item The synthesized gaits are \textit{experimentally validated} for different robot-environment interactions, i.e., three-limb vs. four-limb morphologies and three different substrates. Additionally, the \textit{fault-tolerant} nature of the framework is validated by testing re-synthesized gaits in the simulated scenario of single-actuator failure without need for re-learning. 
\end{enumerate} 

% \hl{(4) Experimental validation, multi-robot platform and fault-tolerance.}

%\hl{Fundamentals of soft robot locomotion by Calisti et al: The survey is focused on two or three types of "typical gaits", e.g., peristaltic, etc} 

\subsection{Paper Outline}
The paper is structured as follows: Section \ref{Sec:Related}~provides a background overview of related approaches to soft robot gait control; Section \ref{Sec:Basics}~introduces the pMFC with an example;  Section \ref{Sec:GaitDefinition} discusses gait definition for soft robots with unknown body kinematics and dynamics; Section \ref{Sec:Synthesis} presents the gait synthesis algorithm using optimization and learning; Section \ref{Sec:ExpSetup} details the experimental setup and procedure, including robot fabrication and visual tracking; Section \ref{Sec:Results} shows the results of the learning experiments, including detailed gait sequence tables and trajectory plots, comparison of BILP algorithm to exhaustive nonlinear optimization, and gait characterization; and Section \ref{Sec:Conclusion} discusses the findings and future work.

%\hl{ Include different morphologies/ actuators / surfaces}
%\hl{Validity of the model assumption of problematically independent motion primitively can be quantified.}

%\hl{3-limb and 4-limb robot, fault-tolerance, comparison with exhaustive search}

%\hl{the proposed gait synthesis method reduces the search space to make the problem more tractable without overly compromising on the optimality.}
% The proposed environment-centric {p}robabilistic {MFC} framework includes mechanism to simultaneously incorporate events of varying likeliness (uncertainty that is inevitable in real-world applications). 

% When applied to terrestrial soft robots, it is generic - independent of robot material, actuation and friction mechanism; adaptable - possesses ability to learn in different environments by constructing environment-specific digraph; and fault-tolerant - the scenarios involving loss of actuation or failure of a robot limb do not require it to re-learn to result in locomotion. Furthermore, the physical locomotion gaits are mathematically analogous to the simple cycles of the digraph. 

% The problem of finding task-specific optimal locomotion gaits can be formulated as weighted sum of average locomotion, uncertainty and gait length. Evaluating this becomes computational efficient due to representation of arcs as normal distributions and simple cycles as integer combination. The resulting optimization problem is an Integer Programming Problem (IPL) with linear constraints that can be solved using commercially available optimization solvers.

 \section{Related Works: Learning Locomotion Gaits for Soft Robots}
 \label{Sec:Related}
Due to the strong influence of bioinspiration, the most common method for soft robot gait selection involves recreating the gait behavior exhibited in nature by the robot's biological analogue (i.e., biomimicry). Examples include snake-inspired locomotion with pneumatic actuation \cite{8593404}, earthworm-inspired peristalsis with shape-memory alloy actuation 
\cite{seok_meshworm:_2013} and inchworm-inspired locomotion with pneumatic actuation \cite{doi:10.1089/soro.2017.0042}. By using a combination of known biological gait behavior and intuitive manual control or tuning, researchers can simplify gait control. The gait parameters in these examples are tuned using sparse experimental sampling, iterative learning control, and neural networks, respectively. However, biomimetic and hand-picked intuitive gaits can naturally suffer from suboptimality and discrepancies between expected biomimetic behavior and actual behavior. Additionally, relying on these methods can artificially restrict the design space of soft robots to biomimicry or overly simple designs to facilitate intuitive or ad hoc gait selection. 
 
 On the other hand, finite element-based and model-based approaches are limited by their computational complexity and issues that arise from the high sensitivity of soft robots to manufacturing inaccuracies and environmental conditions, as well as the difficulty in capturing dynamic friction effects (e.g., sliding and stick-slip transitions) \cite{coevoet_soft_2019, bern_trajectory_2019, schaff_soft_2022, saunders_experimental_2011, https://doi-org.libdata.lib.ua.edu/10.1002/aisy.201900186}. Similar concerns, particularly the intractability of complex designs, also affect the more efficient but less accurate evolutionary simulators like VoxCAD \cite{1241678}. These issues of model fidelity and accuracy affect the resulting experimental performance when transferring the control policies to real soft robots, often cited as the simulation-to-reality gap \cite{pinskier_bioinspiration_2022}. In essence, model-based methods are currently insufficient to accurately capture the robots' interactions with the environment, despite the richness of such interactions being a key component of soft robot locomotion \cite{Mazzolai_2022}.

 Recently, many soft roboticists have urged the incorporation of data-driven methods, including machine learning approaches, to augment model-based control and overcome the challenges of the high dependence on the environment, manufacturing inconsistencies, and variation in material properties \cite{Della_2023_IEEECSM}. Modeling limitations have also inspired a large field of research into learning-based methods and data-informed model-based methods, such as reinforcement learning (RL) and central pattern generators (CPG). However, these learning-based approaches are largely focused on manipulators and can often suffer from high computational costs and large dataset requirements \cite{ https://doi.org/10.1002/aisy.202100165}. Additionally, many soft robots exhibit material fatigue, low repeatability, and instances of mechanical failure requiring intervention and maintenance, hindering experimental learning \cite{10136428}.

 \subsection{Reinforcement Learning}
Reinforcement learning (RL) typically involves formulating the robot-environment interaction as a Markov decision process, composed of: 1) robot states; 2) actions to transition between states; 3) the transition probabilities of each action; and 4) the corresponding rewards following each transition. A common goal is then to find an optimal policy via a balance of exploration and exploitation that maximizes the cumulative reward (i.e., the return). Although RL can be used as a model-free method to find optimal gaits, the large dataset and training times required often prevent experimental training of the RL parameters \cite{robotics8010004}.  RL with simulation training (i.e., requiring a robot model) has been implemented to optimize gaits for a quadruped with soft legs and a rigid frame \cite{JI2022102382}, a rolling tensegrity ball robot \cite{10.1109/ICRA.2017.7989079}, and a swimming snake-like soft robot \cite{9561145}. In these selected examples, RL convergence occurred in 500 episodes (i.e., simulated executions of a robot gait until the robot falls), a couple hundred episodes (equivalent to a few hours of continuous experimental training), and 200 episodes (each with 500 steps of 400 ms), respectively. Such extensive exploration results in the infeasibility of experimental training due to a combination of the time requirement and the inevitable effects of the wear and tear imposed on the robot. Experimental training remains impractical even when using expedited training (e.g., the soft actor-critic method). RL and other machine learning approaches have the potential to capture non-linearities and robot-environment interactions in a way that model-based methods currently cannot \cite{Mazzolai_2022}. Nevertheless, ``model-free" RL for gait optimization in practice largely requires a model of the soft robot (with simplifying assumptions for tractability) to perform training in simulation, resulting in simulation-to-reality gaps.

The presented research overcomes these challenges by discretizing the set of all robot states and assuming quasi-static transitions between states to facilitate exhaustive experimental learning of rewards. This simplified problem closely resembles an idealized reinforcement learning problem, where, unlike the previous examples, the experimental rewards for all states and actions can be exhaustively tabulated without any prior knowledge of the robot dynamics, such as simulation models for training or predefined stable gaits for initialization. From this, optimal periodic gaits (i.e., policies) can then be derived via offline optimization. The connections between this type of comprehensive offline approach and modern reinforcement learning algorithms are explored in \cite{sutton2018} and \cite{szepesvari2010}.

\subsection{Central Pattern Generators}
 A central pattern generator (CPG) is a neural circuit in which coupled oscillators produce a group periodic output in response to simple non-rhythmic inputs. CPG models applied to robot locomotion control are generally formulated as coupled differential equations that are specific to the robot type (e.g., fish, snake, lamprey, quadruped, hexapod, etc.) \cite{IJSPEERT2008642}. Although this method has been shown to exhibit many advantages (including stable limit cycles, resilience to disturbances, smooth gait modulations and transitions, and reduction of control parameters), they typically rely on either predefined knowledge of the model (e.g., known gait behavior and kinematic parameters of the animal from which the robot is inspired) or extensive evolution or learning in simulation to derive the model details \cite{IJSPEERT2008642}. For example, a CPG model for a 12-motor spine-inspired tensegrity robot was tuned in simulation (requiring a model of the robot) to achieve various gaits over irregular terrain after sweeping over 100 parameters in 24,000 Monte Carlo trials; experimental implementation then required hand tuning or online optimization \cite{doi:10.1089/soro.2015.0012, 7354134}. Ishige et al. proposed an RL-optimized CPG framework for caterpillar-like soft robots that can generate both periodic and aperiodic locomotion in a variety of robot-environment interactions and is scalable to a large number of sensors and actuators;  however, this method is only tested in simulation and requires a simplified model of the robot dynamics \cite{doi:10.1089/soro.2018.0126}. CPG therefore encounters limitations of extensive training requirements and/or prior knowledge of the robot dynamics.

% In another example, a five-link snake-like robot was induced to turn after following a straight-line trajectory by applying a simple Q-learning algorithm to the CPG model \cite{9147889}. This example requires no explicit dynamic model of the robot but requires a initial gait and is not experimentally validated. 
%In animals, periodic outputs (e.g., locomotion gaits) result from interacting oscillatory inputs called \underline{C}entral \underline{P}attern \underline{G}enerators (CPG) \cite{ijspeert_central_2008}. This has motivated the use of CPG-inspired locomotion controllers that typically rely on tunable predetermined gaits.

%Current gait definitions in the robotic community are not always readily applicable to soft robots as there is an overemphasis on footfalls and patterns observed in rigid robots or vertebrate animals. Further, most soft robot gaits are determined via trial-and-error or biomimetic approaches. This can lead to suboptimal locomotion behavior and a discrepancy between the expected behavior due to the bio-inspired design and the actual behavior. Simply put, the soft robotics literature lacks methods for robust and generic gait definition, synthesis and characterization. 

\subsection{Model-free Gait Optimization}
In the absence of more accurate and efficient soft robot simulation tools, there is a current need for truly model-free control methods with experimental training (i.e., black-box learning). One possible solution is to use evolutionary algorithms to tune sinusoidal control inputs. This was successfully implemented for the gait optimization of a six-actuator knife-fish inspired soft robot, greatly improving performance over the hand-designed controller; while having the advantage of capturing dynamic effects of the robot locomotion without a dynamic model, this method is limited by predefined gait functions and extensive training time (on the order of hours) leading to many instances of hardware failure (including the death of 17 servomotors) \cite{10.1145/3205455.3205583}. Hamill et al. discretized the motion primitives of a six-limb modular robot with binary fluidic actuators (assuming no-slip contacts) for straight-line motion gait synthesis; the resulting motions were estimated with a simple model and formulated as edges of a graph to find gaits as the shortest path cycles \cite{soft_modular}. However, this approach is limited by the lack of experimental training (leading to unexpected performance), suboptimal searches over the graph, and hardware failure during testing. Nevertheless, the use of graph-based approaches that discretize the motion primitives has the potential to simplify the model-free gait synthesis problem, especially for robots with a small number of actuators. Additionally, the states and actions presented in RL can be analogously represented in a graph to facilitate the search for optimal gaits.

\section{Probabilistic \Framework} \label{Sec:Basics}

As a requisite background for this paper, an informal and elementary introduction of required concepts for directed graphs is provided. The reader may refer to \cite{bollobas_modern_2013} for additional details. %
A directed graph (or a digraph) $G$ is defined as an ordered pair $(V(G), E(G))$ with a nonempty set $V(G)$ of $n$ vertices and a set $E(G)$ of $m$ directed edges. An incidence matrix $B=B(G)=(b_{ij})$ of $G$ is a ${n \times m}$ matrix that associates each edge to an ordered pair of vertices:
\begin{align}
    b_{ij} &=\left\{\begin{array}{rl}
    1 & v_i\mathrm{~is~initial~vertex~of~}e_i\\
    -1 &v_i\mathrm{~is~terminal~vertex~of~}e_i\\
    0 & \mathrm{otherwise.}
    \end{array}\right.
\end{align}
\indent A complete digraph is a digraph in which every pair of vertices is connected by bidirectional edge with no loops. In addition, we also define initial and terminal matrices $B^i, B^t$ as

\begin{align}
\scalebox{0.9}{
$B = B^i-B^t,~$}
{\footnotesize 
\begin{gathered}
   b^i_{ij} =\left\{\begin{array}{rl}
    1 & v_i\mathrm{~is~initial~vertex~of~}e_i\\
    0 & \mathrm{otherwise}
    \end{array}\right.\\
    \quad \, b^t_{ij} =\left\{\begin{array}{rl}
    1 & v_j\mathrm{~is~terminal~vertex~of~}e_i\\
    0 & \mathrm{otherwise.}
    \end{array}\right.
\end{gathered}
}
\end{align}

A closed walk consists of a sequence of vertices starting and ending at the same vertex, where consecutive vertices in the sequence are connected by a directed edge. A simple cycle is a closed walk where no vertices or directed edges are repeated, other than the start and the end vertex. Let $L$ be a simple cycle in $G$ with a given sequence of vertices $V(L)=v_1v_2\cdots v_lv_{(l+1)}$ s.t. $v_1=v_{l+1}$
%\hl{I changed these from u and q to v and e because it is currently not consistent in the text (later simple cycles use v). We can change this to where all simple cycles use the u and q notations, but I think it reads cleaner with e and v. I can be convinced otherwise so I commented out the original.}
%Let $L$ be a simple cycle in $G$ with a given sequence of vertices $V(L)=u_1u_2\cdots u_lu_{(l+1)}$ s.t. $u_1=u_{l+1}$. This can also be represented as a sequence of edges $E(L)=q_1q_2\cdots q_l$, where $q_i=u_iu_{i+1}$.
This can also be represented as a sequence of edges $E(L)=e_1e_2\cdots e_l$, where $e_i=v_iv_{i+1}$. While both of these representations contain the notion of a starting vertex $v_1$ or edge $e_1$ (i.e., a starting point to the simple cycle), a cycle can also be identified without this by a $(m\times 1)$ vector $\mathbf{z}=(z_i)$ where
\begin{equation} \label{Eqn:SimpleCyleVector}
z_i = \left\{\begin{array}{cc}
1 & \mathrm{if\ }z_i \in E(L)\\
0 & \mathrm{otherwise.}
\end{array}
\right.
\end{equation}
%\hl{put equations 3 and 4 together?}
The vector $\mathbf{z}$ is a simple cycle only if it satisfies the constraints that (1) at each vertex, the number of incoming edges is equal to the outgoing edges,  $B\mathbf{z}=0$, (2) the out-degree of each vertex (i.e., the number of edges for which each vertex is an initial vertex) is at most one, $B^i\mathbf{z}\leq 1$, and (3) it does not contain a sum of unconnected simple cycles. Hence, the simple cycle subspace $C_1(G)$ with linear constraints is defined as
% This simple cycle vector $\mathbf{z}$ satisfies the constraint that at each vertex, the incoming edges is equal to the outgoing edges, and the degree of each vertex is at most one. Additionally, the simple cycle is not a sum of unconnected simple cycles
\begin{equation}
{\small C_1(G)=\left\{
    \begin{gathered}
    \mathbf{z}\in \mathbb{Z}_2^{(m\times 1)} :\quad 
    B\mathbf{z}=0, \quad B^i\mathbf{z}\leq 1\\
    \nexists \mathbf{z}_1, \mathbf{z}_2 \mathrm{~s.t.~} \mathbf{z} = \mathbf{z}_1+\mathbf{z}_2,~B\mathbf{z}_1=B\mathbf{z}_2=0
    \end{gathered} 
    \right\}}.
    \label{Eqn:SimpleCycleCons}
\end{equation}

\noindent The weights associated with each of its edges $\bm{w}_i=\bm{w}(e_i)$ form $W(E)$, the set of edge weights.

\subsection{Soft Robot and Environment-specific Weighted Digraph} \label{subsec:digraph}
For a terrestrial soft robot, locomotion occurs when an imbalance of forces acting on the robot results in a net planar motion of the center of mass. This locomotion can be generally divided into three categories: static, quasi-static, and dynamic. In static locomotion, the robot is statically stable during all locomotion stages and will remain in its position if the actuation is suspended or locked. Dynamic locomotion generally exploits different strategies, including static instability, to achieve greater changes in inertia and higher magnitudes of speed. Accordingly, quasi-static locomotion can be described as locomotion that includes some dynamic behavior but is statically stable at a certain stage of motion. In the context of this research, both static and quasi-static locomotion can be modeled and explored using weighted digraphs; fully dynamic motion (e.g., rolling) is not currently explored in this work. 

\textit{Robot states} are defined as discrete statically stable robot configurations. Every soft robot thus possesses a multitude of robot states, each with varying shapes and postures that can change the distribution and magnitude of forces on the system. \textit{Motion primitives} refer to the possible transitions between these robot states. As these transitions can lead to force imbalances that result in locomotion, motion primitives effectively discretize the factors dominating the robot-environment interaction.  For this discussion, we focus on formulating a graphical framework that methodically associates knowledge of motion primitives with environment feedback. The identification and exploration of these robot states is outside the scope of this research.

To construct the digraph, robot states are modeled as vertices $V(G)$, and motion primitives are modeled as directional edges $E(G)$ that connect the vertices. The resulting translation $\disp{i}\in\Re^{2\times 1}$ and rotation $\theta_i$ associated with each motion primitive is then encoded in the corresponding edge weight $\bm{w}(e_i)$, recorded in the coordinate system of the initial vertex. These motions and edge weights are environment-specific and depend on the type of surface of interaction. Hence, \textit{these digraph weights are fundamental to robot adaptation and learning in different environments.} To account for variances in actuation and robot-environment interactions in addition to measurement noise, each edge weight $\bm{w}_i$ is modeled as a normal distribution with mean $\bm{\mu}_{i} \in \Re^{3\times 1}$ and covariance matrix ${\Sigma}_{i} \in \Re^{3 \times 3}$:
\begin{align} 
% \bm{w}_{i} = \vGaussian{\mu_{i}}{\Sigma_{i}} \label{Eqn:ArcWeight}
\bm{w}_i = \mathcal{N}\left(\bm{\mu}_i,\Sigma_i\right)
, \quad %
\bm{\mu}_i = \begin{bmatrix}\disp{i}\\ \theta_i
\end{bmatrix}, \quad 
{\Sigma}_i = \begin{bmatrix}
{\Sigma}_{pp} &\Sigma_{p\theta}\\
\Sigma_{\theta p} & \Sigma_{\theta\theta}
\end{bmatrix}
\label{Eqn:ArcWeight}.
\end{align}
%The probabilistic distribution of the edges may be chosen as desired.

The pMFC framework is summarized as a weighted robot-environment specific digraph $(G,W)$ with the following elements:
 \begin{enumerate}%[leftmargin=*,itemsep=0pt]
    \item The structure of digraph $G$ is determined by the robot:
    \begin{itemize}%[leftmargin=*,itemsep=0pt]
        \item $n$, number of robot states or graph vertices $V(G)$
        \item $m$, number of motion primitives or edges $E(G)$
        \item $B(E)$, incidence matrix of dimension $n \times m$
    \end{itemize}
    \item The probabilistic digraph edge weights, $W(E)=\{P(E),\Theta(E),S(E)\}$, correspond to the resulting locomotion of motion primitives and are environment-specific:
    \begin{itemize}%[leftmargin=*,itemsep=0pt]
        \item $P(E)$, mean displacement matrix of dimension $2 \times m$
        \item $\Theta(E)$, mean rotation matrix of dimension $1 \times m$
        \item $S(E)$, noise covariance matrix of dimension $ 3\times 3\times m$ where $S(e_i)=\Sigma_i(e_i)$
        \item $S_p(E)$, translation covariance trace matrix of dimension $1\times m$ where $\displaystyle S_p(e_i) = \mathrm{tr}\left(\Sigma_{pp}(e_i)\right)$
        \item $S_\theta(E)$, rotation covariance matrix of dimension $1\times m$ where $S_\theta(e_i) = \Sigma_{\theta \theta}(e_i)$.
        %where $S(e_i) = \mathrm{tr}(\bm{\Sigma}_i)$
    \end{itemize}
    % \item $l$, the length of the simple cycle.
\end{enumerate}
While discretizing the robot's states does not affect its physical deformability and adaptability, it is essential to select states that accurately capture the key factors governing its interaction with the environment to prevent oversimplification.
\subsection{Example Robot}
A deformable two-limb soft robot, shown in \Fig \ref{Fig:ExampleGraph}, is presented to gain a better understanding of the framework. %
The robot comprises a soft deformable body and a flexible hub that houses the motors of the motor-tendon actuators. The embedded tendons connected to the motors allow for controlled curling or uncurling of each limb. The limb states correspond to the discrete configurations when the limb is curled or uncurled, given as
\begin{equation}
\mathrm{limb~state} = \left\{ \begin{array}{ll}
0 & \mathrm{for\ limb\ uncurled}\\
1 & \mathrm{for\ limb\ curled.}
\end{array}\right.
\label{Eqn:RobotState}
\end{equation}

\begin{figure}[ht]
\centering
\subfloat[][]{\includegraphics[page=1,width=.53\columnwidth,trim= 3.3cm 2.6cm 5cm .5cm, clip=true]{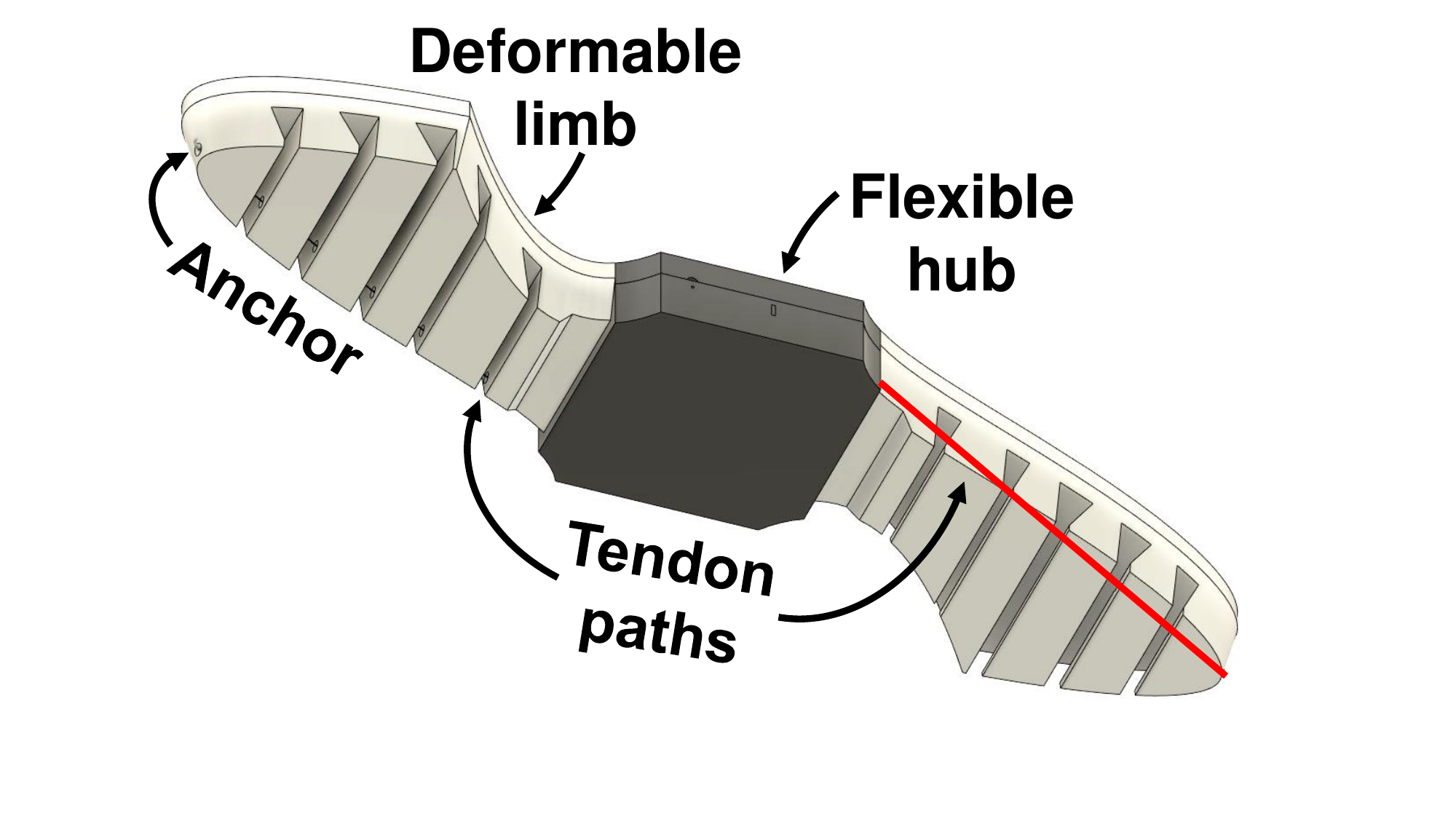}}\hfill
\subfloat[][]{\includegraphics[page=4,width=.38\columnwidth,trim= 7cm .7cm 7.5cm .7cm, clip=true]{figures/Example_Robot_Figures.pdf}}\par
\subfloat[][]{\includegraphics[page=2,width=.6\columnwidth,trim= .2cm 3.2cm .2cm .2cm, clip=true]{figures/Example_Robot_Figures.pdf}}\par
\label{Fig:ExampleRobot}
\caption{(a) Example two-limb soft robot CAD drawing with corresponding (b) digraph (four discrete robot states and 12 motion primitives) and (c) CAD-rendered robot configurations. }
\label{Fig:ExampleGraph}
\end{figure}
For a $n_l=2$-limb robot with $n_s=2$ states per limb, the robot can exist in $n = \left(n_s\right)^{n_l}=4$ discrete states which are different permutations of the limb states - $\{00\}, \{01\}, \{10\}, \{11\}$, or $V=\{v_1,v_2,v_3,v_4\}$. These states correspond to the different robot shapes and represent vertices of the digraph in \Fig \ref{Fig:ExampleGraph}. For this complete digraph, where transitions between all vertices are allowed, the number of motion primitives is $\displaystyle m=n(n-1)=12$. The edges $E=\{e_1,e_2,\cdots,e_{12}\}$ are defined using the incidence matrices $B,B^i$ as
\newcommand{\po}{\phantom{--}1}
\newcommand{\zz}{\phantom{--}0}
\newcommand{\mo}{\phantom{l}-1}
\begin{equation*}
{\scriptsize B = \left[\begin{array}{l}
    {\phantom{-}1 \po \po} {\mo \zz \zz} {\mo \zz \zz} {\mo \zz \zz}\\
    {-1 \zz \zz} {\po \po \po} {\zz \mo \zz} {\zz \mo \zz}\\
    {\phantom{-}0\mo \zz} {\zz \mo \zz} {\po \po \po} {\zz \zz \mo}\\
    {\phantom{-}0\zz \mo} {\zz \zz \mo} {\zz \zz \mo} {\po \po \po}
    \end{array}\right]},
\end{equation*}
\renewcommand{\mo}{\zz}
\begin{equation*}
{\scriptsize B^i = \left[\begin{array}{l}
    {\phantom{-}1 \po \po} {\mo \zz \zz} {\mo \zz \zz} {\mo \zz \zz}\\
    {\phantom{-}0 \zz \zz} {\po \po \po} {\zz \mo \zz} {\zz \mo \zz}\\
    {\phantom{-}0\mo \zz} {\zz \mo \zz} {\po \po \po} {\zz \zz \mo}\\
    {\phantom{-}0\zz \mo} {\zz \zz \mo} {\zz \zz \mo} {\po \po \po}
    \end{array}\right]}.
\end{equation*}

 For the example robot that exhibits translation and no rotation, a simple cycle, defined in \eqref{Eqn:SimpleCycleCons},  is the mathematical representation of a periodic gait. We shall later see, in Section \ref{Sec:GaitDefinition}, that a locomotion gait is defined as a simple cycle that is transformation invariant. It is important to note that gaits are commonly defined by not only a sequence of motions or footfalls, but also by parameters such as the duty cycle and relative phase, leading to complex mathematical formulations. However, in the context of the pMFC, a simple cycle vector is sufficient for representing a gait. In the current implementation for this paper, the components of the simple cycle vector do not have a temporal component; the motion primitives are of equal time duration and are the result of binary actuation. This simplification is considered acceptable due to the widespread use of binary actuation in soft robotics control \cite{https://doi.org/10.1002/aisy.202100165}.  
 
 The robot is inspired by a caterpillar which dominantly utilizes two locomotion gaits: inching and crawling. These are illustrated\footnote{There are a number of established methods to illustrate both biological and robotic gaits including but not limited to footfall diagrams, contact-force plots, joint angle plots, and relative phase diagrams. We choose in this work to use robot drawings with shaded limbs to de-emphasize footfalls and highlight the constant-length binary actuation in a compact and simple manner. This is because the sliding/inching motion of the limbs of this particular soft robot results in all limbs maintaining contact with the ground at all times. }
 below where the maroon shaded limbs represent the curled limb.
\begin{enumerate}[leftmargin=*]
\item {Crawling gait}: $V(L_{crawl})=v_1v_2v_3v_1$ \vspace{5pt}

\begin{center}
\fbox{
\parbox[c]{.68\columnwidth}{
\centering
\raisebox{-.3\totalheight}{\includegraphics[width=.95cm,trim=0cm 6cm 0cm 6cm, clip=true]{figures/MTA2_State_1}} \kern-.1em  $\xrightarrow[]{e_1}$ \kern-.1em \raisebox{-.3\totalheight}{\includegraphics[width=.95cm,trim=0cm 6cm 0cm 6cm, clip=true]{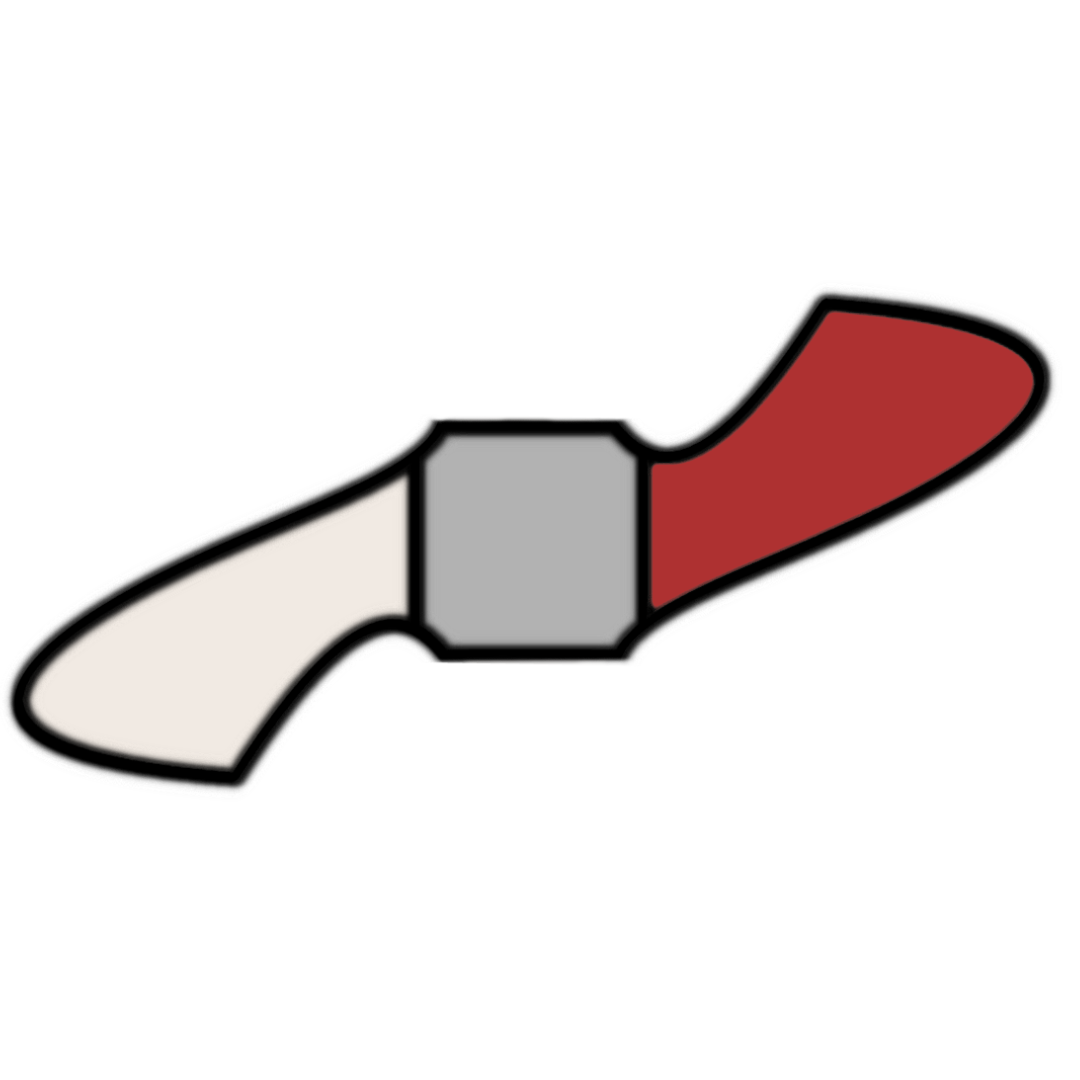}} \kern-.1em $\xrightarrow[]{e_5}$ \kern-.1em  \raisebox{-.3\totalheight}{\includegraphics[width=.95cm,trim=0cm 6cm 0cm 6cm, clip=true]{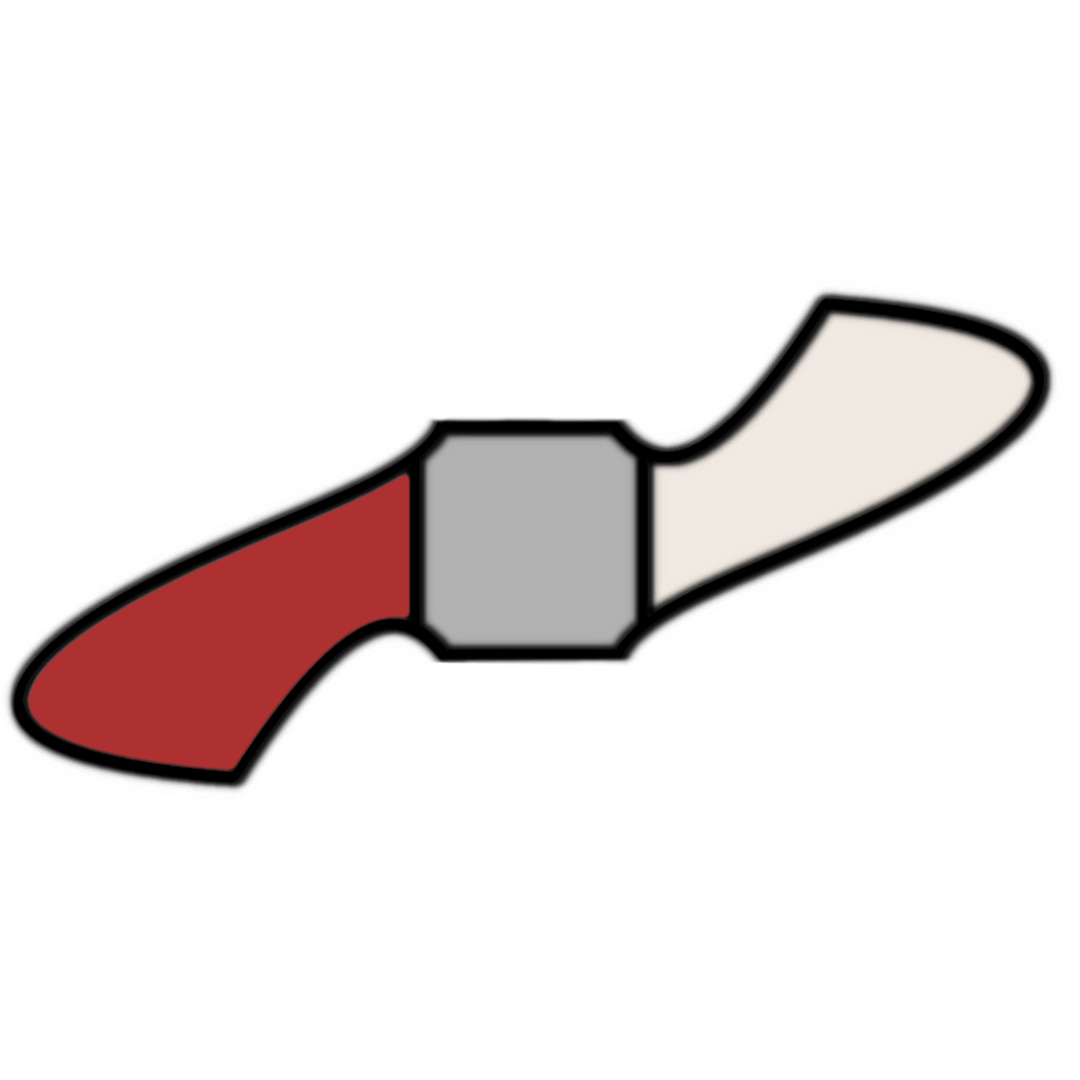}} \kern-.1em $\xrightarrow[]{e_7}$ \kern-.1em  \raisebox{-.3\totalheight}{\includegraphics[width=.95cm,trim=0cm 6cm 0cm 6cm, clip=true]{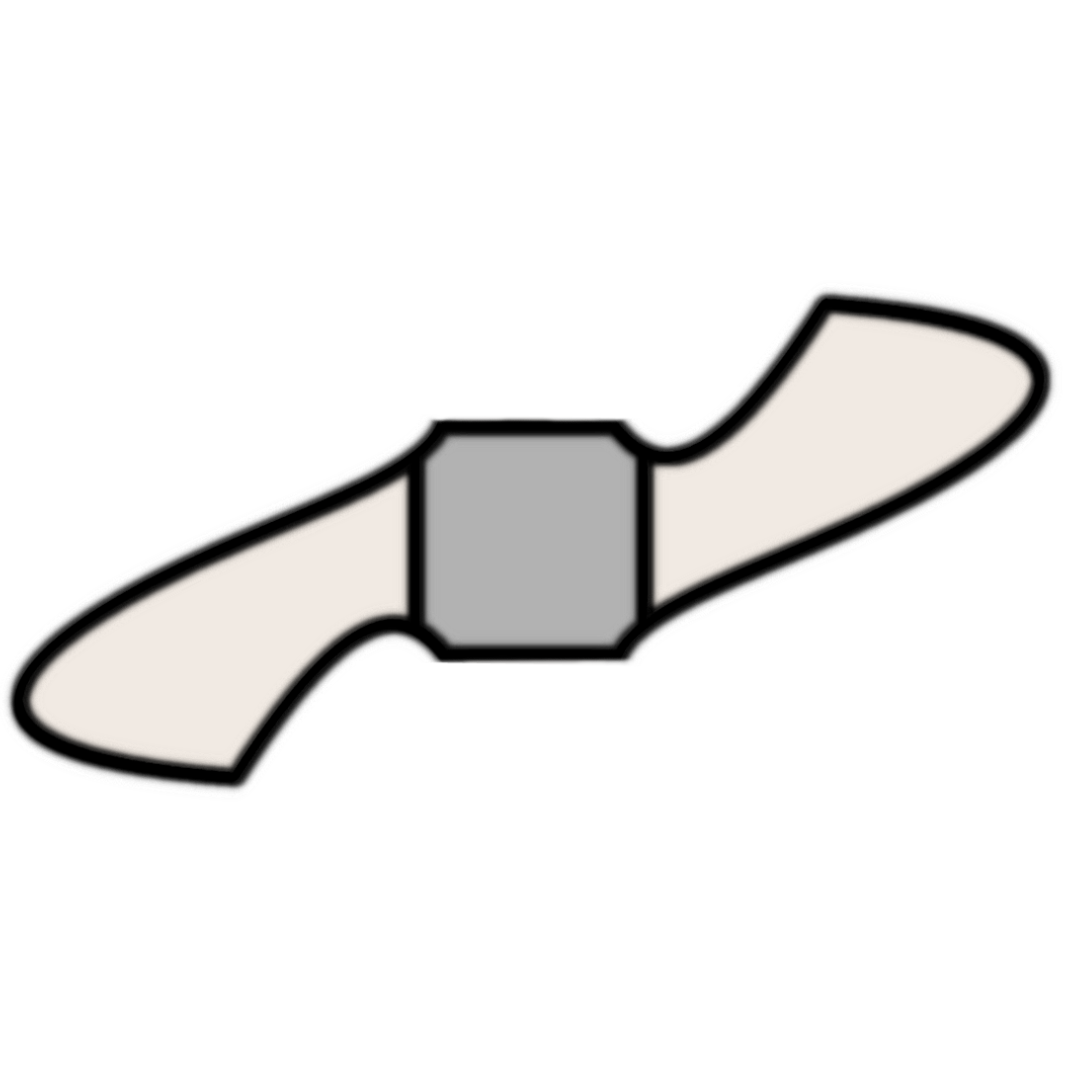}}
}
}
\end{center}

% Hence, crawl gait simple cycle is 
\begin{equation}
\mathbf{z}_{crawl} = [1,0,0,0,1,0,1,0,0,0,0,0]^\mathsf{T}
\end{equation}
\item {Inching gait}: $V(L_{inch})=v_1v_2v_4v_2v_1$ \vspace{5pt} 
\begin{center}
\fbox{
\parbox[c]{.85\columnwidth}{
\centering
\raisebox{-.3\totalheight}{\includegraphics[width=1cm,trim=0cm 6cm 0cm 6cm, clip=true]{figures/MTA2_State_1}} \kern-.1em  $\xrightarrow[]{e_1}$ \kern-.1em \raisebox{-.3\totalheight}{\includegraphics[width=.95cm,trim=0cm 6.5cm 0cm 6.5cm, clip=true]{figures/MTA2_State_2.png}} \kern-.1em $\xrightarrow[]{e_6}$ \kern-.1em  \raisebox{-.3\totalheight}{\includegraphics[width=.95cm,trim=0cm 6.5cm 0cm 6.5cm, clip=true]{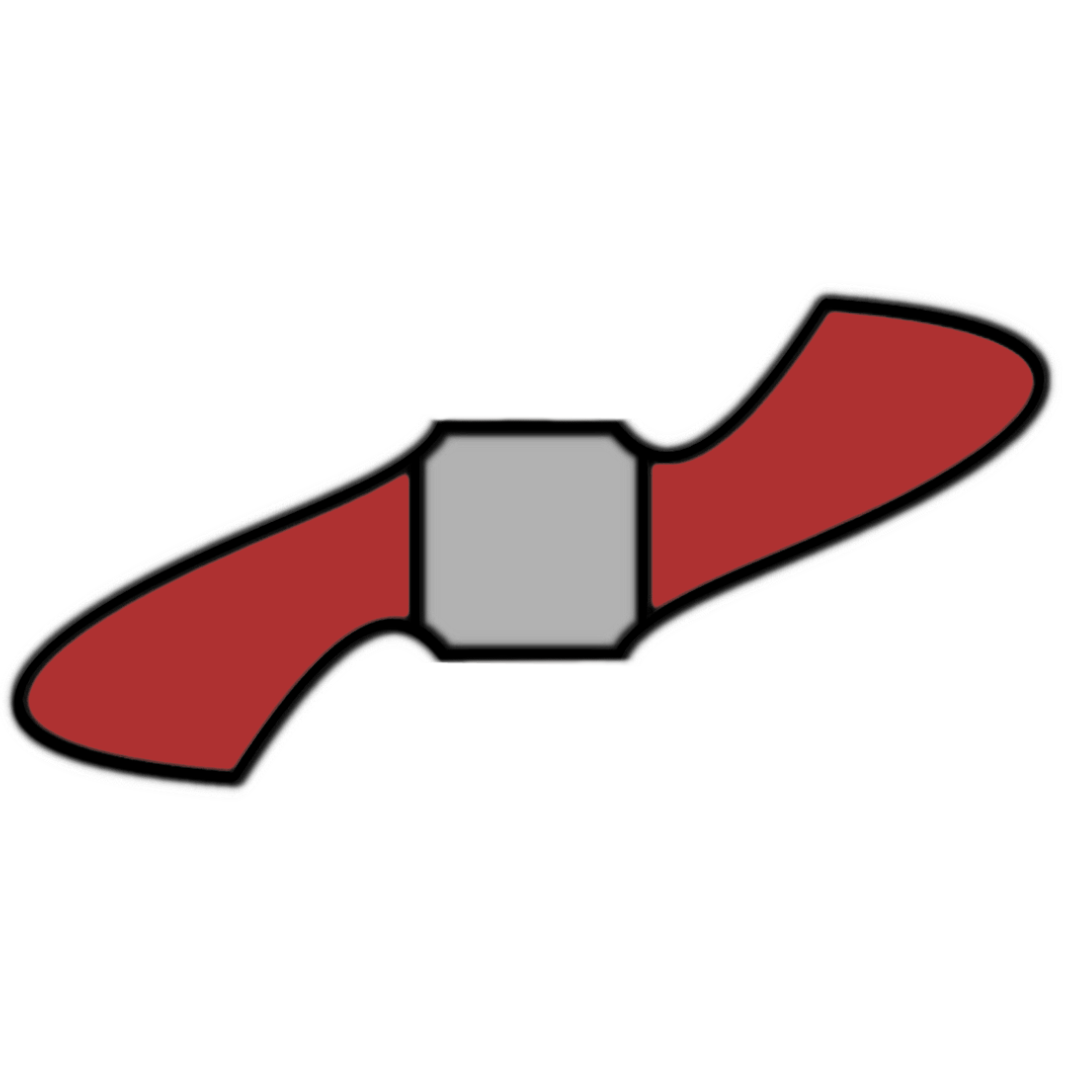}} \kern-.1em $\xrightarrow[]{e_{12}}$ \kern-.1em  \raisebox{-.3\totalheight}{\includegraphics[width=.95cm,trim=0cm 6.5cm 0cm 6.5cm, clip=true]{figures/MTA2_State_2.png}} \kern-.1em $\xrightarrow[]{e_7}$ \kern-.1em  \raisebox{-.3\totalheight}{\includegraphics[width=.95cm,trim=0cm 6.5cm 0cm 6cm, clip=true]{figures/MTA2_State_1.png}}
}
}
\end{center}
% Consequently, the inch gait simple cycle is 
\begin{equation}
\mathbf{z}_{inch} = [1,0,0,0,0,1,1,0,0,0,0,1]^\mathsf{T}
\end{equation}\end{enumerate}
% However, the number of simple cycles ${c}$ for a complete digraph increase exponentially with the size $n$, equivalently, the number of limbs and discrete limb states \cite{johnson_finding_1975} as
% \begin{equation}
% c = \sum_{i=1}^{n}\left(
% \begin{array}{c}
% n\\
% n-i+1
% \end{array}
% \right) \left(n-i\right)!
% \end{equation}
%%%%%%%%%%%%%%%%%%%%%

\subsection{Randomized Learning of the Environment-specific Digraph} \label{Subsec:Hierholzer}
The robot-environment interaction is encoded in the digraph edge weights $W(E)$ resulting from the motion primitives. Experimentally, the most efficient way to exhaustively learn the graph weights is to execute an Eulerian cycle, a closed walk that traverses every edge $E(G)$ exactly once. The  Hierholzer's algorithm is an efficient method to obtain the Eulerian cycle \cite{hierholzer_ueber_1873} and can be used to learn digraph weights for a specific environment. For the example robot shown in \Fig \ref{Fig:ExampleGraph}, one possible Eulerian cycle with starting and terminating vertex ${v_1}$ is $E(\mathrm{Eulerian~cycle})=e_1e_5e_9e_{10}e_3e_{11}e_{6}e_{12}e_7e_2e_8e_4$. We propose a modified {stochastic} Hierholzer's algorithm for complete digraphs, Algorithm  \ref{Algo:ModifiedHeriholzers}, that runs in $\mathcal{O}(m)$ time for $m = n(n-1)$ edges. The existence of such Eulerian cycles for digraphs is possible only when the digraph is connected and the in-degree of every vertex is equal to its out-degree. This condition is satisfied for the complete digraphs discussed in this work.
% $e_1 \rightarrow%
% e_5 \rightarrow%
% e_9 \rightarrow%
% e_{10} \rightarrow%
% e_3 \rightarrow%
% e_{11} \rightarrow%
% e_6 \rightarrow%
% e_{12} \rightarrow%
% e_7 \rightarrow%
% e_2 \rightarrow%
% e_8 \rightarrow%
% e_4%
% $. %
%In comparison, the Modified Hierholzer's Algorithm run in $\mathcal{O}(n_a)$ time.
%%%%%%%%%%
\begin{algorithm}
\SetAlgoLined
\KwResult{Eulerian cycle (list of ordered vertices)}
states = $[0, 1, \ldots, n]$ \;
random.shuffle(states) \;
 cycle = empty list  \;
 \For{i = $[n-1, \ldots, 1, 0]$}{
    j = 0 \;
    \While{$i\neq j$}{
        cycle.append(states[j])\;
        cycle.append(states[i])\;
        j = j+1 \;
    }
 }
 Return cycle
 \caption{{Stochastic} Hierholzer's for complete digraphs.}
 \label{Algo:ModifiedHeriholzers}
\end{algorithm}

The proposed algorithm leverages the structure of the complete digraph to reduce Hierholzer's algorithm to essentially a permutation enumeration problem. This exhaustive edge traversal and data collection process can be repeated for multiple trials in an experiment. The stochastic nature is introduced by a \texttt{random.shuffle} operation, which randomizes the order of the vertices for each trial; the returned cycle is one of $n!$ unique Eulerian cycles. The motion primitives in the pMFC framework are assumed to be quasi-static (i.e., statically stable for all discrete states) and independent of each other. Thus, stochastically selecting an Eulerian cycle for each data collection trial minimizes the biases that may be caused by transient effects and edge order.
\subsection{Framework Implementation}

The implementation of the pMFC can be divided into the following steps. 

\subsubsection{Digraph Construction} The soft robot's $n$ states and $m$ motion primitives are identified and used to build a robot-specific digraph with empty edge weights, as detailed in Section \ref{subsec:digraph}. 
\subsubsection{Edge Weight Learning Experiments} 
\Algo \ref{Algo:ModifiedHeriholzers} is used to generate randomized sequences of robot states to fully traverse the complete digraph. This determines the control sequence for the soft robot to experimentally execute sequentially on a given substrate. The resulting translation and rotation of each motion primitive is recorded and used to calculate the corresponding edge weight according to \eqref{Eqn:ArcWeight} in Section \ref{subsec:digraph}.
\subsubsection{Gait Synthesis}
Optimal translation and rotation gaits, represented by simple cycles in the graph, are then synthesized via the optimization algorithm presented in Section \ref{subseq:linearcost}.
\subsubsection{Gait Validation Experiments} The synthesized gaits are then tested experimentally to validate the performance.
\section{Gait Definition and Transformation Invariance} \label{Sec:GaitDefinition}
Locomotion gaits are generally classified into two categories: fixed and free \cite{wettergreen_gait_1992}. Fixed gaits follow the traditional definition of gaits, where an animal or robot employs a fixed periodic sequence of body movements to realize locomotion. For more versatile movement, roboticists later abstracted the idea of gaits to encompass free gaits, which comprise any sequence of motions, whether periodic or aperiodic, that move the robot between two specified locations, often considering obstacles or dynamically challenging environments. Free gaits therefore require much more computation, modeling, and sensing. We restrict our definition of gaits to the biological interpretation (i.e., fixed gaits), where gait synthesis is defined as the generation of periodic control sequences that correspond to a desired movement.

Terrestrial locomotion gaits in $SE(2)$ are expected to be repeated sequentially to realize planar motion. However, the definition of locomotion gaits becomes complicated when both translation and rotation are involved. While the simple vector cycle notation defined in \eqref{Eqn:SimpleCyleVector} affords simplicity in the gait definition and gait synthesis optimization problem, it does not encode a specific starting vertex $v_1$. Here, we show that such ambiguity can lead to inconsistent resultant locomotion. 

Let $L$ be a simple cycle with $l$ edges where each edge $e_i$ has an associated initial vertex %$u_j$, terminal vertex $u_{j+1}$
$v_i$, terminal vertex $v_{i+1}$ and edge weight $\bm{w}_i$ represented in the coordinate system $i$ of the initial vertex\footnote{The subscript $i$ in this section is meant to represent the relationship between vertices, edges, and coordinate systems to elucidate the subsequent proofs in \Sec~\ref{Sec:GaitDefinition} and does not correlate to the vertex/edge numbering convention used in other sections of the paper.}. Let the edge representation of $L$ be %$E(L) = q_1q_2\cdots q_l$ 
$E(L) = e_1e_2\cdots e_l$. %
For kinematic analysis, the body coordinate system is fixed on the robot which translates and rotates by $\bm{w}_i=\left[\bm{p}_i,\theta_i\right]^\mathsf{T}$ when executing a motion primitive, i.e., traversal along the edge $e_i$ as visualized in \Fig \ref{Fig:ExampleLocomotion} for a three-limb soft robot.  The transformation matrix  $g_{i} \in SE(2)$ for this motion is
\begin{align*}
    g_{i} &= \TMat{R(\theta_i)}{\disp{i}}, \quad R(\theta) = \begin{bmatrix}
    \cos\theta & -\sin\theta\\
    \sin\theta & \cos \theta
    \end{bmatrix}\in SO(2)
    %\bm{w}_i(e_i) = \begin{bmatrix}
    % \disp{i}\\
    % \theta_i
    % \end{bmatrix}
\end{align*}
The traversal of consecutive edges $e_ie_{i+1}\cdots e_j, s.t., i<j$ is captured in the transformation matrix $g_{i,j}$
\begin{align}
    \begin{gathered}
    g_{i,j}(e_ie_{i+1}\cdots e_{j}) = 
    \TMat{R_{i,j}}{\bm{p}_{i,j}}\\ \mathrm{s.t.}\quad 
    R_{i,j} = R\left(\sum_{k=i}^{j}\theta_k\right) \quad \mathrm{and}\\
    \bm{p}_{i,j} = \disp{i} + \Rot{i,i}\disp{i+1}\cdots+\Rot{i,j-1}\disp{j} = \sum_{k=i}^{j} \Rot{i,(k-1)}\disp{k} \\
    \mathrm{where}\quad \Rot{i,i}= R(\theta_i) \quad \mathrm{and} \quad \Rot{i,(i-1)}= I.
    \end{gathered}
\end{align}
Here, the displacement $\disp{{i,j}}$ and rotation $\Rot{i,j}$ are expressed in the coordinate system of the initial vertex of the edge $e_i$.
% These quantities are represented in the coordinate system fixed to the initial vertex of the edge. The shorthand notation $R_{i,j}\in SO(2)$ is defined as 
% \begin{align*}
% {\small
% \begin{gathered}
%     R_{i,j} = 
%     \left\{\begin{array}{l@{\quad}l}
%         R\left(\sum_{k=i}^{j}\theta_k\right) & i<j\\
%         R(\theta_i) &  i=j\\
%         I_{2\times 2} & i>j
%     \end{array}\right., \quad
%       R(\theta) = \begin{bmatrix}
%     \cos\theta & -\sin\theta\\
%     \sin\theta & \cos \theta
%     \end{bmatrix}
%     \end{gathered}
%     }.
% \end{align*}
% Geometrically, $R_{i,j}$ represents the rotation matrix between coordinate systems associated with initial vertex of edge $e_i$ and terminal vertex of edge $e_j$.
%\subsection{Locomotion and MFC}
\begin{figure}[ht]
\begin{center}
\newcommand{\coordSys}[3]{
\draw[fill=black] #1 node[below,left]{$O_{#2}$} circle (0.05);
\draw[thick,rotate=#3,->] #1 -- +(0,1) node[left]{$y_{#2}$};%
\draw[thick,rotate=#3,->] #1 -- ++(1,0) node[below]{$x_{#2}$};
}
\definecolor{lococolor}{RGB}{174,49,49}

\begin{tikzpicture}[transform shape, scale=0.8]
\pgfmathsetmacro{\rotationA}{0}
\pgfmathsetmacro{\rotationB}{-20}
\pgfmathsetmacro{\rotationC}{-65}
\pgfmathsetmacro{\rotationD}{20}
\pgfmathsetmacro{\scaleR}{0.28}
\draw (0,0) coordinate (originA) node{\includegraphics[scale=\scaleR,angle=\rotationA]{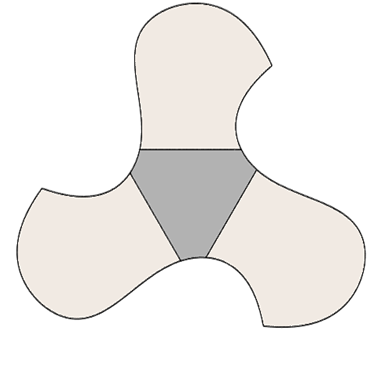}};
\coordSys{(originA)}{1}{\rotationA}
\draw (originA) ++(3,2) coordinate (originB) node{\includegraphics[scale=\scaleR,angle=\rotationB]{figures/MTA3.png}};
\coordSys{(originB)}{2}{\rotationB}
\draw (originB) ++(3,-3) coordinate (originC) node{\includegraphics[scale=\scaleR,angle=\rotationC]{figures/MTA3.png}};
\coordSys{(originC)}{3}{\rotationC}
\draw (originC) ++(1.75,3) coordinate (originD) node{\includegraphics[scale=\scaleR,angle=\rotationD]{figures/MTA3.png}};
\coordSys{(originD)}{4}{\rotationD}
%%%
\draw[very thick,->,lococolor] (originA) -- (originB) node[midway,above=3pt,left]{{\footnotesize $\deltaX{1}{1}{2}$}};
\draw[very thick,->,lococolor] (originB) -- (originC) node[midway,left=3pt]{{\footnotesize $\deltaX{2}{2}{3}$}};
\draw[very thick,->,lococolor] (originC) -- (originD) node[midway,right=3pt]{{\footnotesize $\deltaX{3}{3}{4}$}};
%%%%% Coordinate systems
\draw[thick, dashed,lococolor] (originB) -- ++(1.5,0) node[above]{$x_1$}++(-0.75,0) arc (0:(\rotationB-\rotationA):0.75) node[above=2pt,right=2pt]{$\deltaT{1}{1}{2}$};
\draw[thick, dashed,lococolor,rotate=\rotationB] (originC) -- ++(1.5,0) node[above]{$x_2$}++(-0.75,0) arc (0:(\rotationC-\rotationB):0.75) node[below=3pt,right=4pt,rotate=-\rotationB]{$\deltaT{2}{2}{3}$};
\draw[thick, dashed,lococolor,rotate=\rotationC] (originD) -- ++(1.5,0) node[above]{$x_3$}++(-0.75,0) arc (0:(\rotationD-\rotationC):0.75) node[midway,below=10pt,right=10pt,rotate=-\rotationC]{$\deltaT{3}{3}{4}$};
\draw (1,3) coordinate (originR) ;%node{\includegraphics[scale=0.07,angle=0]{figures/YbotTop.png}};
%\coordSys{(originR)}{1}{0}
%\coordSys{(originR)}{2}{-30}
%\coordSys{(originR)}{3}{-80}
%
\end{tikzpicture}
\end{center}
\caption{Example planar locomotion of a three-limb soft robot where the translations and rotations are observed in the coordinate system of the initial vertex of each edge $e_i$.}
\label{Fig:ExampleLocomotion}
\end{figure}
Consequently, the transformation matrix for simple cycle $L$ that begins at $e_1$ is
\begin{align*}
    g^1(L) &= g_{1} g_{2} \cdots g_{l} = %\\
    %&=
    \TMat{\Rot{1,l}}{\disp{1,l}}
\end{align*}
% where the superscript denotes the starting edge, and %The resulting translation $\Trans{1}$ (expressed in the coordinate system of initial vertex of $e_1$) and rotation $\Rot{1}$ are  
% \begin{align*}
%     \Trans{1} &= \disp{1} + \Rot{1}\disp{2} + %
%     \Rot{1,2}\disp{3}+\cdots %\\ & %
%     +\Rot{1,(l-1)}\disp{l}\\
%     &= \sum_{k=1}^{l-1} \Rot{1,(k-1)}\disp{k}, \\
%     R^1 &= \Rot{1,l}.
% \end{align*}
where the superscript denotes the starting edge. However, the simple cycle may start at any edge $e_k\in E(L)$. %  \\
% \vspace{5pt}\\
\begin{proposition} The transformation matrix of a simple cycle $L=e_1e_2\cdots e_l$ with start edge $e_i$ is
\begin{align}
    g^i(L) &= \left(g_ig_{i+1}\cdots g_l\right)\left(g_1\cdots g_{i-1}\right)%
    %\prod_{j=i}^{l}g_{j} \prod_{k=1}^{i-1}g_{k} %
    = \TMat{R^i}{\Trans{i}}
\end{align}
where the superscript denotes the starting edge. Two transformation matrices $g^i(L),g^j(L)$ with starting edges $e_i,e_j$ are related as
\begin{align}
    \begin{gathered}
    R^i=R^j=R_{1,l}\\
    \Trans{i} = \Rot{i,(j-1)}\Trans{j} 
     + \left(I-\Rot{1,l}\right)\bm{p}_{i,(j-1)}
    \end{gathered}
    \label{Eqn:TMatEquivalence}
\end{align}
\end{proposition}
\begin{proof}
Rotations are commutative in $SE(2)$, hence, $\forall i,j \leq l$ the rotation associated with simple cycle $L$ irrespective of the starting edge is $R^i=R^j=R_{1,l}$. The two transformation matrices can be written as
\begin{align*}
    g^i = \left(g_ig_{i+1}\cdots g_{j-1}\right) \left(g_{j}g_{j+1}\cdots g_l g_1 \cdots g_{i-1}\right)\\
    g^j = \left(g_j g_{j+1}\cdots g_l g_1 \cdots g_{i-1}\right)\left(g_ig_{i+1}\cdots g_{j-1}\right)
\end{align*}
Hence,
\begin{align*}
    \begin{gathered}
    g^i g_{i,j-1} = g_{i,j-1}g^j\\
{\scriptsize    \TMat{R_{1,l}}{\Trans{i}}\TMat{R_{i,(j-1)}}{\disp{{i,(j-1)}}} = \TMat{R_{i,(j-1)}}{\disp{{i,(j-1)}}}\TMat{R_{1,l}}{\Trans{j}}}\\
\Rightarrow \Trans{i} +\Rot{1,l}\disp{i,(j-1)} = \disp{i,(j-1)} +\Rot{i,(j-1)}\Trans{j}\\
\Trans{i} = \Rot{i,(j-1)}\Trans{j} + (I-\Rot{1,l})\disp{i,(j-1)}
    \end{gathered}
\end{align*}
\end{proof}

If the magnitude of the resultant translation or rotation for a gait depends on the starting vertex, the gait definition remains ambiguous. In context of the pMFC framework, this is resolved by ensuring that the resulting magnitude is independent of the starting vertex. 
Geometrically, $\Trans{i},\Trans{j}$ denote the displacements of robot as it traverses the same simple cycle but with different start edges. However, they are expressed in different coordinate systems -  in the initial vertices of their respective edges $e_i,e_j$. \eqref{Eqn:TMatEquivalence} provides the relationship where $\Trans{j}$ is transformed to the same coordinate system as $\Trans{i}$, i.e., the initial vertex $e_i$.
%\vspace{5pt}\\
\begin{definition}
Locomotion gaits are defined as simple cycles that are transformation invariant. The principle of transformation invariance implies that the displacement and rotation of a simple cycle are  preserved irrespective of starting vertex, i.e.,
\begin{align*}
    \Trans{i} = \Rot{i,(j-1)}\Trans{j} \quad \mathrm{and} \quad R^i=R^j, \qquad \forall 1<i,j\leq l
\end{align*}
\end{definition}
\begin{proposition}
Only two fundamental types of locomotion gaits in $SE(2)$ exist\footnote{This definition is specific to the pMFC framework and results from the assumptions of quasi-static motion and transformation invariance. In other contexts, coupled rotation and translation cycles may be considered. In fact, some researchers think of the rotation gait as being the only possible type of gait; here, pure translation results from an infinite turning radius \cite{rotation_gait}.}
\begin{enumerate}
    \item Translation gait: When the cumulative rotation of the simple cycle is zero, i.e., $\displaystyle \sum_{k=1}^l \theta_k = 0$
    \item Rotation gait: When the translation of all simple cycle motion primitives are zero, i.e., $\disp{k}=0~\forall~e_k\in E(L)$
\end{enumerate}
\end{proposition}
\begin{proof}
We have already seen that the resultant rotation of a simple cycle is the same irrespective of the starting vertex, i.e., $R^i=R^j$. Using \eqref{Eqn:TMatEquivalence}, translation equivalence implies
\begin{align*}
\left(I-\Rot{1,l}\right)\bm{p}_{i,(j-1)}=0
\end{align*}
This is possible only for two cases 
\begin{enumerate}
    \item $\displaystyle \left(\Rot{1,l}-I\right) =0 \Rightarrow \sum_{k=1}^l \theta_k = 0$
    \item $\displaystyle \disp{i,(j-1)}=0 \Rightarrow \sum_{k=i}^{j-1} \Rot{i,(k-1)}\disp{k} = 0~\forall~1<i,j\leq l$. Hence, $\disp{k}=0~\forall~e_k\in E(L)$.
\end{enumerate}
\end{proof}

\section{Gait Synthesis and Exploration} \label{Sec:Synthesis}

Gait synthesis is the process of determining a set of periodic sequences of motion primitives that satisfy the transformation invariance principle.

\subsection{Nonlinear Cost Functions} \label{subsec:nonlinear}
One approach to optimal gait synthesis would be to perform constrained optimization on nonlinear cost functions $J_{t,nl}$ and $J_{\theta,nl}$ that maximize a weighted sum of motion (translation or rotation, respectively) normalized by the length of the simple cycle: $\mathbf{z}^\mathsf{T}\mathbf{z}$. 
\begin{align}
    J_{t,nl}(\mathbf{z}) = \frac{\lVert \bm{p}(\mathbf{z}) \rVert + \lambda_t s_p(\mathbf{z})}{\mathbf{z}^\mathsf{T}\mathbf{z}}, ~ 
    J_{\theta,nl}(\mathbf{z}) = \frac{\lvert \Theta^\mathsf{T}\mathbf{z} \rvert + \lambda_\theta s_\theta(\mathbf{z})}{\mathbf{z}^\mathsf{T}\mathbf{z}}
    \label{Eqn:nlcostfuncs}
\end{align}
where $\lambda_t,\lambda_\theta$ are scalars, and {$s_p, s_\theta$} are the trace of the translation and rotation covariance matrices of the simple cycle. Here, the calculation of $\bm{p}(\mathbf{z})$ and the covariance matrix $\Sigma(\mathbf{z})$ can be approximated on Euclidean motion groups \cite{wang_nonparametric_2008}. Given the lack of a smooth expression for these cost functions, they need to be exhaustively calculated for each locomotion gait (constrained simple cycles).

A well-known algorithm for finding simple cycles has been proposed by Tarjan and later improved by Johnson \cite{johnson_finding_1975}. However, the number of simple cycles $n_{z}$ for a complete digraph increases exponentially with the size of the graph as defined by
% \begin{align*}
$ \displaystyle n_z = \sum_{i=1}^{n}\left(
\begin{array}{c}
n\\
n-i+1
\end{array}
\right) \left(n-i\right)! $ 
% \end{align*}
Consequently, the associated time complexity $\mathcal{O}((n + m)(n_z+1))$ limits the usefulness of optimizing these cost functions. For example, this approach results in a very tractable lookup table for small graphs, e.g.,  $n_z = 24$ for the two-limb example robot. However, larger graphs quickly result in combinatorial explosion, e.g., $n_z \approx 3.81 \times 10^{12}$ for the four-limb experiment robot with $n=16$ vertex digraph. As a result, the nonlinear cost function for this combinatorial (discrete) problem is not tractable for robots with more than three actuators with traditional solvers. 

The pMFC-based gait synthesis can be considered as an asymmetric traveling salesman problem, a famous NP-hard example in combinatorial optimization. Modern solutions include stochastic sampling solvers like simulated annealing, ant colony optimization, genetic algorithms, reinforcement learning (e.g., simple Q-learning or more complex deep learning approaches), and neural networks \cite{GOPALAKRISHNAN201551, 10.5555/3618408.3618810, ARORA2017683}. However, drawbacks to these approaches include unknown optimality and the black-box nature, which makes tuning of these solvers difficult to understand in physical terms. Therefore, a common approach is to linearize the cost function and use well-known methods like branch-and-bound which have guarantees on optimality \cite{ARORA2017683}. Although this approach does not entirely eliminate the issue of combinatorial explosion, it provides a good balance of tractability and optimality, as discussed in further detail in Section \ref{subsec:tract} and Section \ref{subsubsec:optimality}.

\subsection{Linearized Cost Functions} \label{subseq:linearcost}
We propose modified cost functions $J_t$ and $J_\theta$, respectively, that linearly weight the locomotion, variance and  gait length as 
\begin{align*}
    J_t(\mathbf{z}) &=  \bm{\alpha}_t^\mathsf{T}\underbrace{P\mathbf{z}}_{\mathrm{translation}} + \beta_t \underbrace{S_{p}^\mathsf{T}\mathbf{z}}_{\mathrm{variance}} + \gamma_t \underbrace{{1}_{1\times m}\mathbf{z}}_{\mathrm{length}}, \\
    J_\theta(\mathbf{z}) &= \alpha_\theta\underbrace{\Theta\mathbf{z}}_{\mathrm{rotation}} + \beta_{\theta} \underbrace{S_{\theta}^\mathsf{T}\mathbf{z}}_{\mathrm{variance}} + \gamma_\theta \underbrace{{1}_{1\times m}\mathbf{z}}_{\mathrm{length}}
\label{Eqn:costfuncs}
\end{align*}
where $P,\Theta,S_p,S_\theta$ are the graph's translation, rotation and covariance trace matrices as summarized in \Sec \ref{Sec:Basics}A; and $\{\bm{\alpha}_t,\beta_t,\gamma_t,\alpha_\theta,\beta_\theta,\gamma_\theta\}$ are the tunable hyperparameters. Consequently, gait synthesis is formulated as a \underline{B}inary \underline{I}nteger \underline{L}inear \underline{P}rogramming (BILP) optimization problem with linear constraints that can be solved using optimization solvers, e.g., MATLAB\textregistered \texttt{ intlinprog}. The optimization problem for the synthesis of a translation-dominant gait $\mathbf{z}_t$ is formulated as 

\begin{equation}
\begin{aligned}
    \min_{\mathbf{z}} \quad &J_t(\mathbf{z}) = \left(\bm{\alpha}_t^\mathsf{T} P + \beta_t S_{p} + \gamma_t 1_{1\times m} \right)\mathbf{z} \\
    \mathrm{s.t.} \quad& |\Theta \mathbf{z}|\leq \varepsilon_\theta\\
    &\mathbf{z}=0, B^i\mathbf{z}\leq 1, z_i\in\{0,1\} \forall i,\\
    &\nexists \mathbf{z}_1, \mathbf{z}_2 \mathrm{~s.t.~} \mathbf{z} = \mathbf{z}_1+\mathbf{z}_2,~B\mathbf{z}_1=B\mathbf{z}_2=0
\end{aligned}
\label{Eq:tconstraints}
\end{equation}
where $\varepsilon_\theta$ is a user-defined scalar that limits the total permitted rotation for a simple cycle to ensure translation-dominance. Similarly, a rotation-dominant gait $\mathbf{z}_\theta$ can be synthesized via
\begin{equation}
\begin{aligned}
    \min_{\mathbf{z}} \quad &J_\theta(\mathbf{z}) =\left({\alpha}_\theta \Theta + \beta_\theta S_{\theta} + \gamma_\theta 1_{1\times m} \right)\mathbf{z} \\
    \mathrm{~s.t.~} \quad& |P\mathbf{z}| \leq \varepsilon_t\\
    &\mathbf{z}=0, B^i\mathbf{z}\leq 1, z_i\in\{0,1\} \forall i,\\
    &\nexists \mathbf{z}_1, \mathbf{z}_2 \mathrm{~s.t.~} \mathbf{z} = \mathbf{z}_1+\mathbf{z}_2,~B\mathbf{z}_1=B\mathbf{z}_2=0
\end{aligned}
\label{Eq:rconstraints}
\end{equation}
where $\varepsilon_t$ is the user-defined maximum permitted translation.
%\hl{This BILP has two computational challenges, one pertaining to satisfaction of the }
The exploration of gaits is performed by varying the hyperparameters ($\alpha, \beta, \gamma$), thereby varying the direction and magnitude of the translation/rotation, the variance (predictability) of the gait behavior, and the length (speed) of the gaits. Therefore, an advantage of this method is the explainability of these hyperparameters, as they correspond to understandable values.
%\hl{(1) Non-linear nature of our problem as the $\bm{z}$ are binary integers. Hence, the need for sampling methods. Here, we vary the hyper-parameters, i.e., we're trying to find the optimal hyper-parameters. Note: optimal solution for given hyper-parameters exists using traditional methods.(2) Can we get a plot for exploration. (3) Another loop for choosing $\gamma_\theta$, i.e., if the lengths are too large or small - gait lengths are bounded (0 and 6). All parameters are robot-dependent.}
% M- the number of uniform widths for each dimension - for stratification of the search space
% N - is the number of samples per stratification
% L : parameter for DFS to check if it is simple cycle - z = z1 + z2 constraint.

Randomized variation of the first hyperparameter $\{\bm{\alpha_t}, \alpha_\theta\}$ is achieved with \underline{L}atin \underline{H}ypercube \underline{S}ampling (LHS), which produces a grid $H\in \mathbb{R}^{N\times M}$ of sample points constrained to the interval $[0,1]$. $N$ is the user-defined number of variations, while $M$ corresponds to the length of $\alpha$ ($M=2$ and $M=1$ for translation and rotation goals, respectively). The LHS grid can then be adapted to sweep $\bm{\alpha}_t$ and $\alpha_\theta$ in the search spaces $\bigl[ \begin{smallmatrix}-1 & 1\\ -1 & 1\end{smallmatrix}\bigr]$ and $[-1, 1]$, respectively. This permits the exploration of omnidirectional translation and both clockwise and counter-clockwise rotation. \Fig \ref{Fig:flowchart} details the proposed gait synthesis algorithm. Initially, the solution to the \begin{figure}[h]
\begin{center}
\includegraphics[width=0.94\columnwidth,trim=.5cm 0cm 14cm 0cm, clip=true]{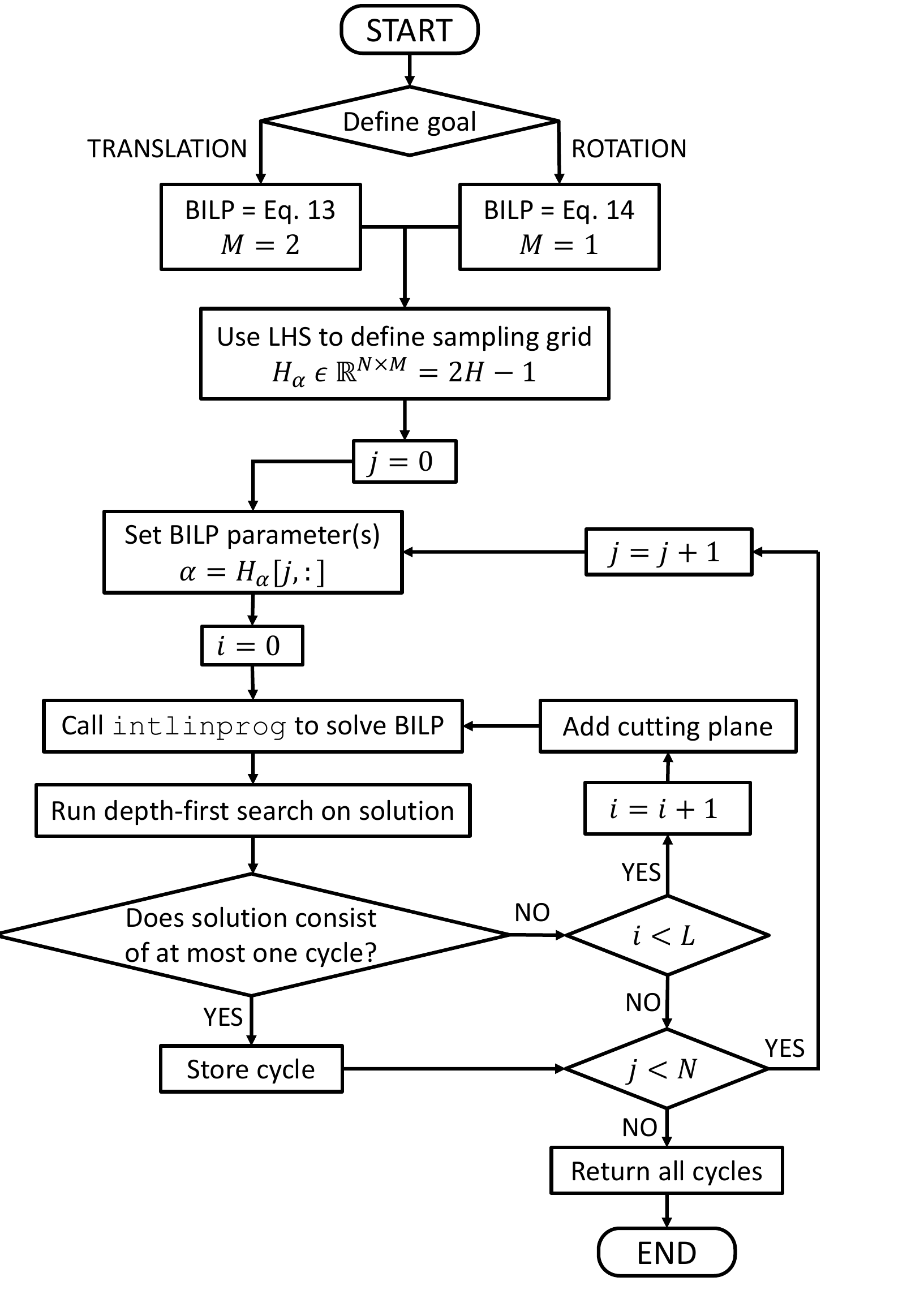}
\caption{Gait synthesis algorithm. Sampling grid $H$ is of size $N$ (number of variations) by $M$ ($M=1,2$ for rotation and translation gaits, respectively). $L$ is the user-defined maximum number of iterations to achieve simple cycle output.}
\label{Fig:flowchart}
\end{center}
\end{figure}
linearly constrained BILP, \eqref{Eq:tconstraints} or \eqref{Eq:rconstraints}, permits disjoint cycles. A detection step using \underline{D}epth-\underline{F}irst \underline{S}earch (DFS) is included to check if the solution is a single simple cycle, and, if not, a cutting plane is added to the BILP that removes this solution from the feasible region.

\subsection{Complexity and Tractability} \label{subsec:tract}
The gait synthesis problem remains NP-hard even in its linearized cost formulation and is complicated by the density of the directed graph and the existence of negative-weight cycles. However, the use of integer variables and DFS creates an iterative (not linear or convex) optimization approach that is solved using linear programming relaxation sub-problems. \Matlab \verb+intlinprog+ uses a branch-and-bound technique to solve BILPs and is generally fast, but can devolve into exhaustive enumeration in its worst case. With parameters $N = 100$ and $L = 50$, gait synthesis for the datasets presented in this paper is generally solved in 1 to 2 minutes on a desktop computer with an Intel(R) Xeon(R) CPU E5-1650 v4 @ 3.60GHz processor and 64 GB of RAM. This speed is important because this process will need to be repeated multiple times per robot-environment interaction to properly tune the hyperparameters to the user's satisfaction. 

 There is an active area of research into solving these types of combinatorial optimization problems, particularly when there is a large number of graph nodes. Because branch-and-bound solvers may not scale to systems with a larger number of actuators, alternative approaches like those mentioned in Section \ref{subsec:nonlinear} can be used. However, our presented methodology (Fig.~\ref{Fig:flowchart}) with \Matlab \verb+intlinprog+ is preferred in our case due to its tractability for the robots presented in Section \ref{Sec:ExpSetup}, the tunability based of the hyperparameters, and the near optimality demonstrated in Section \ref{subsubsec:optimality}.
\section{Experimental Setup and Procedure} \label{Sec:ExpSetup}
\subsection{Robot Design and Fabrication}

To validate the efficacy and versatility of the proposed gait synthesis, we fabricate two motor-tendon actuated (MTA) soft robots with different morphologies: the three-limb \TriSoRo and the four-limb \TetraSoRo. Both are designed and fabricated to be radially symmetric modular soft robots that are capable of multi-module reconfiguration into a sphere, as detailed in \cite{freeman_topology_2022}. This design leads to unique locomotive behavior of the individual robot modules as the curling of the limbs induces coupled translation and rotation. All robot configurations are statically stable, which simplifies calculations in the pMFC framework. Furthermore, the motion primitives are quasi-static, provided that the binary actuation is performed at the appropriate power level for the appropriate time duration such that the limb curling and uncurling reaches steady state. Hence, this robot can be used to highlight and validate this gait synthesis procedure as: 1) simple actuation produces complex movement with coupled translation and rotation; 2) the stick-slip nature of the limb movement complicates traditional gait synthesis and increases the dependence of the robot behavior on the environment; and 3) the exploration of a diverse set of gaits can be used for open-loop path planning. %

Both robots consist of a 3D-printed semi-flexible robot hub which houses the motors and tether connector and silicone limbs that are cast in a 3D-printed mold. Fishing line attached to the motor spool is then threaded through Teflon-reinforced channels in the limb and secured to the tips of the limbs with fishing hooks. The tendon routing and hub design for the \TetraSoRo~is shown in Fig.~\ref{Fig:fourlimbSoRo} and the two robots are shown in Fig.~\ref{Fig:robots}. The resulting MTAs are powered via a tether that connects to an off-board power supply and motor driver circuit, controlled via an Arduino Nano microcontroller. Both robots are supplied with 12 V DC and utilize binary actuation (step functions) with consistent time constants for all motion primitives that are experimentally determined (550 ms for \TriSoRo~and 450 ms for \TetraSoRo). To facilitate passive uncurling of the limbs, a brief motor reversal is incorporated when transitioning from a curled to uncurled limb state. The tether for the \TetraSoRo~is equipped with a slip ring connector at the hub to prevent motion bias due to accumulated torsional effects when the robot performs rotations for an extended period of time.

\begin{figure}[htb]
\begin{center}
\subfloat[][]{\includegraphics[page=1,width=.33\columnwidth,trim= 0.7cm 4.6cm 18cm 2.65cm, clip=true]{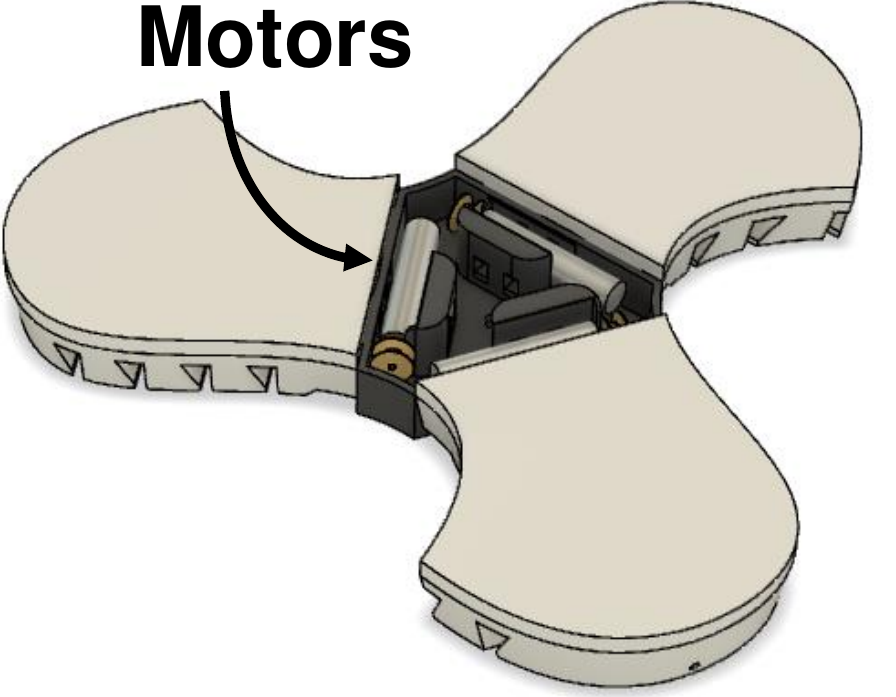}}\hfill
\subfloat[][]{\includegraphics[page=1,width=.30\columnwidth,trim= 16cm 4.9cm 2.6cm .9cm, clip=true]{figures/MTA3_fig.pdf}}\hfill
\subfloat[][]{\includegraphics[width=.31\columnwidth,trim= 22cm 8cm 15.4cm 7.8cm, clip=true]{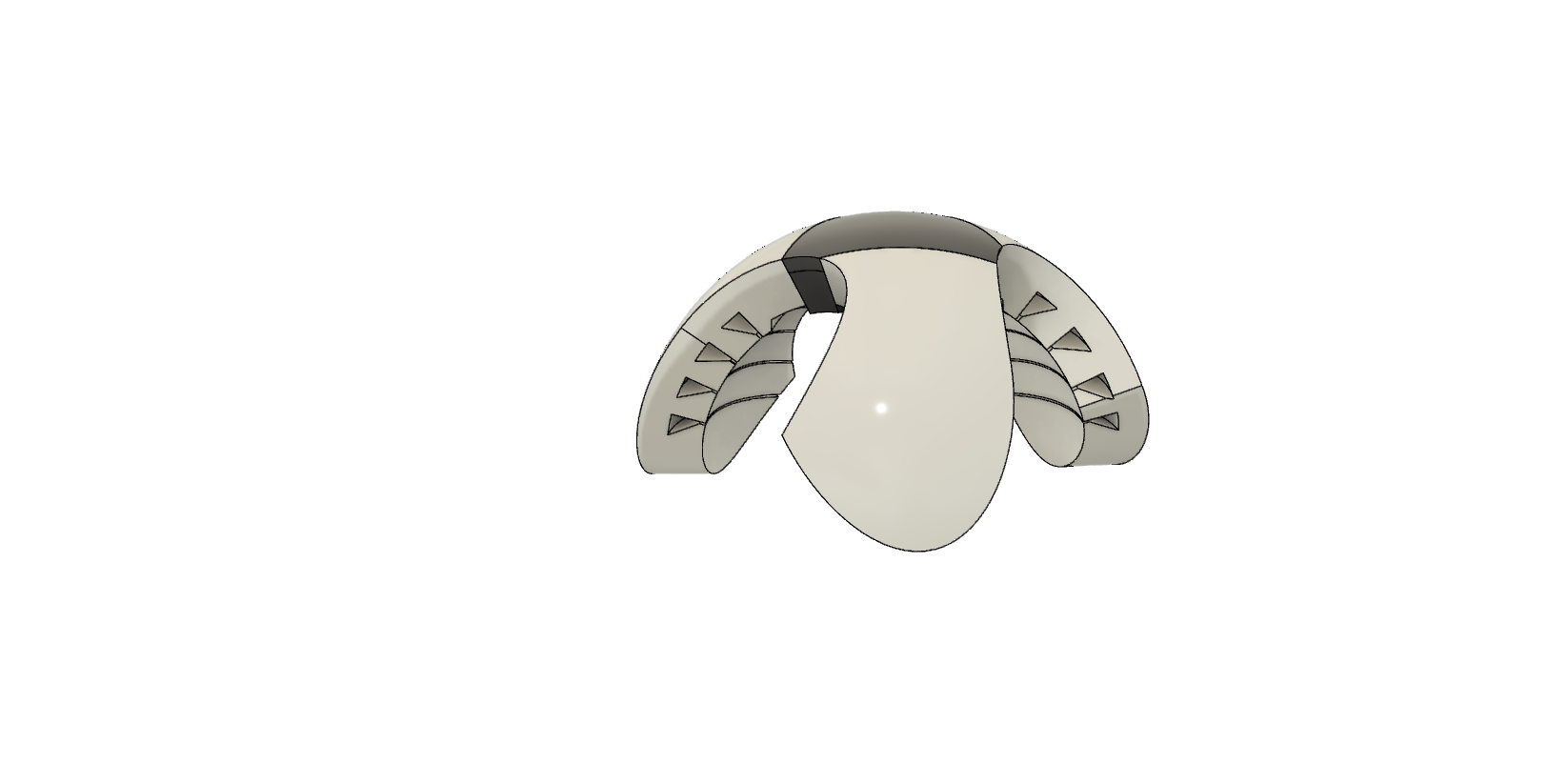}}

\caption{\TriSoRo~CAD drawings with (a) hub cap removed to show motors unactuated state, (b) bottom view with tendon routing (red), and (c) spherical curling in fully actuated state.}
\label{Fig:fourlimbSoRo}
\end{center}
\begin{center}
\subfloat[][]{\includegraphics[width=.5\columnwidth]{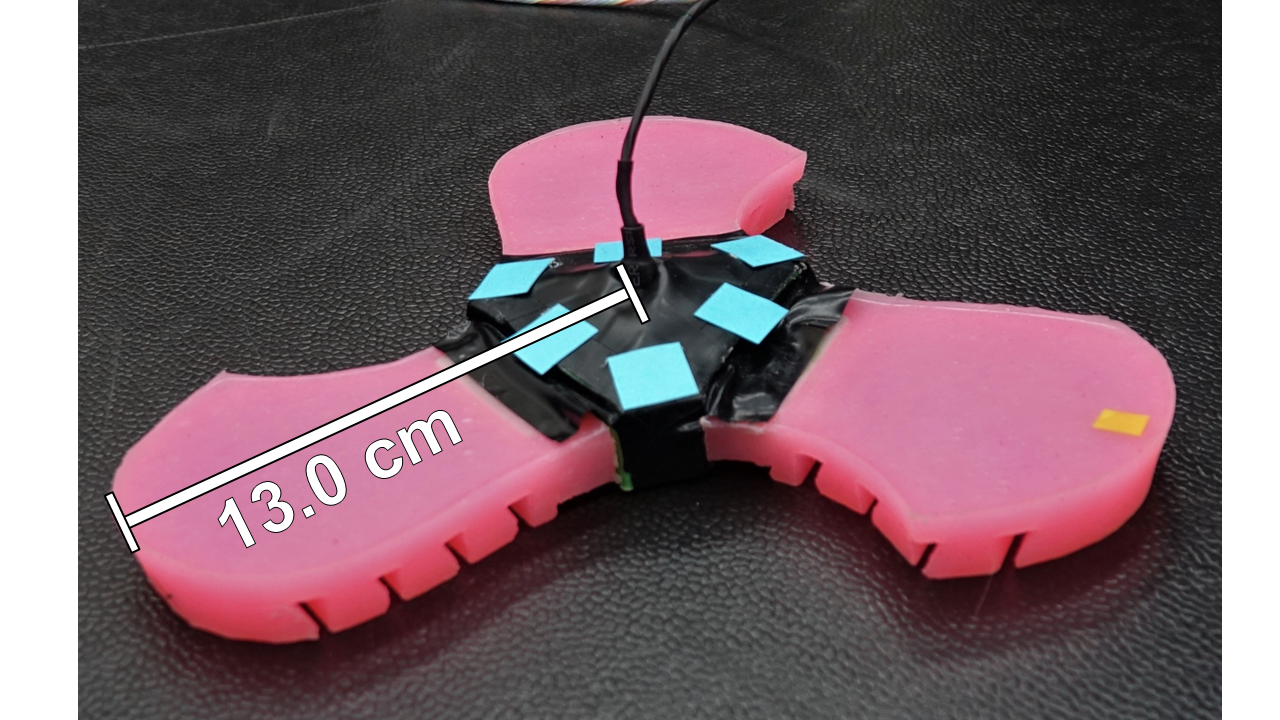}} \hfill
\subfloat[][]{\includegraphics[width=.5\columnwidth]{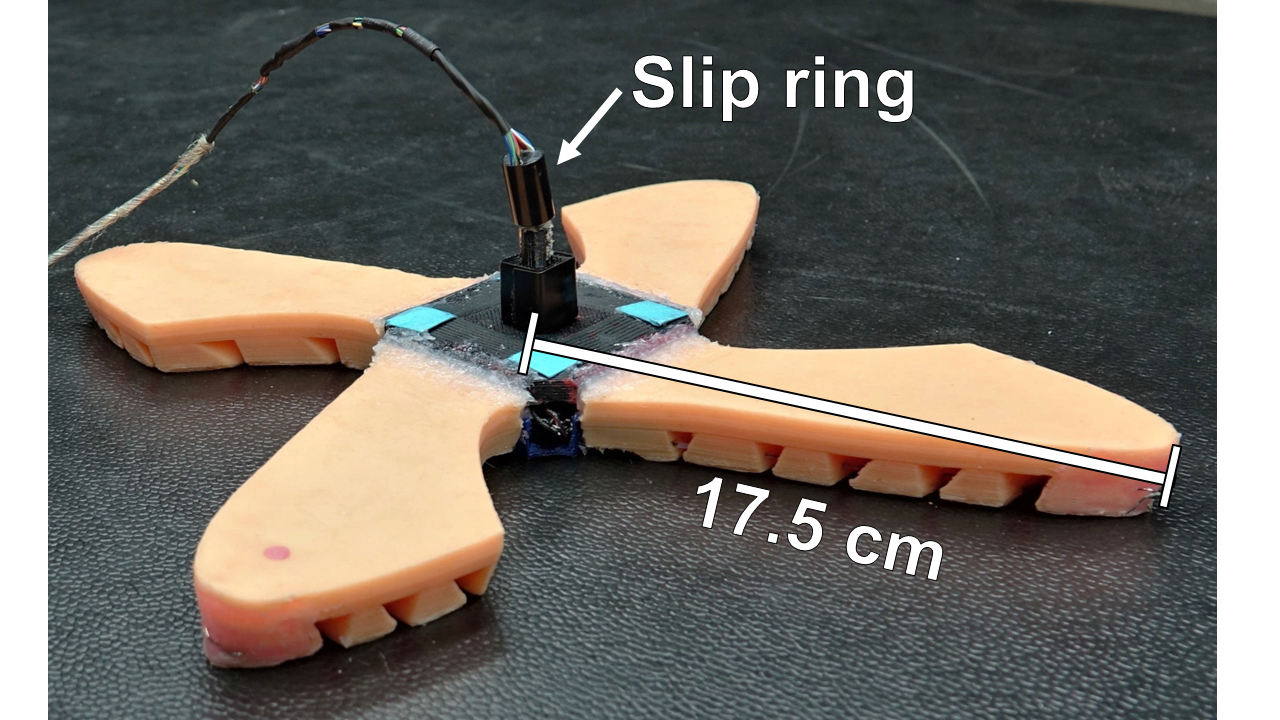}}
\caption{Soft MTA robots with labelled distance from robot center to limb tip: (a) three-limb \TriSoRo~and (b) four-limb \TetraSoRo~with slip ring tether connector.}
\label{Fig:robots}
\end{center}
\end{figure}

\subsection{Experimental Setup}

Experiments are performed on flat substrates placed on top of a laboratory floor and secured to prevent slippage. Three substrates with varying frictional and roughness properties are used in this experiment: a rubber garage mat, a whiteboard, and a carpet. As shown in Fig.~\ref{Fig:expsetup}, the robot is placed in the center of the substrate and fluorescent markers (fiducials) are placed on the hub. Two overhead webcams (approx. 30 fps) are placed above to capture the substrate with a field of view of approximately 4 feet by 7 feet. One webcam is reserved for capturing video of the experiments while the other adjusts the camera properties (e.g., contrast, brightness, etc.) to facilitate tracking by improving image segmentation of the fluorescent markers. Both cameras are calibrated to account for distortion using the MATLAB Camera Calibrator App and a printed checkerboard pattern. 
\begin{figure}
    \centering
    \subfloat[][]{\includegraphics[page=2,width=.53\columnwidth,trim= 7.6cm 0.5cm 15cm .3cm, clip=true]{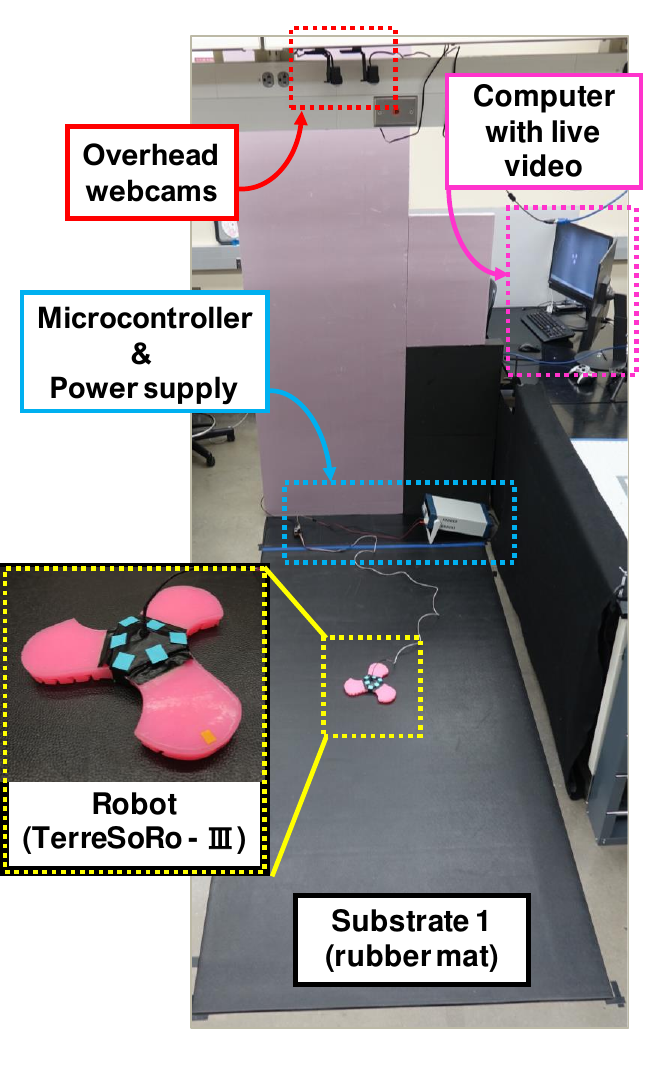}} \hfill
    \subfloat[][]{\includegraphics[page=2,width=.43\columnwidth,trim= 19cm 0.5cm 5.8cm .3cm, clip=true]{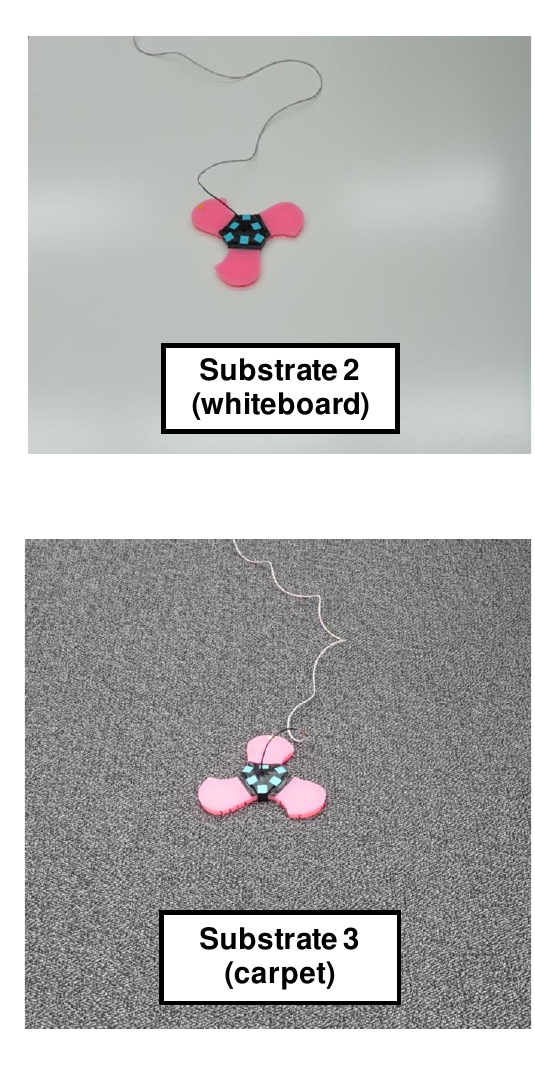}}
    \caption{(a) Experimental setup with the soft robot placed on top of a removable flat substrate with (b) alternate substrates. A tether cable connects the robot to the microcontroller, which connects to the power supply. }
\label{Fig:expsetup}
\end{figure}

Tracking of the robot is performed frame by frame in MATLAB by using the Color Thresholder App to segment the markers as regions of interest in a binary mask. The marker centroids are calculated and recorded for each frame. The robot pose $[x,y,\theta]^T$ can then be estimated by finding the least-squares solution for the rotation and translation between the first frame and the current frame \cite{arun_least-squares_1987}. In cases where the tether occludes one or more markers, the marker positions(s) can be reconstructed from the estimated pose by finding the nearest neighbor from the previous frame.

\subsection{Experimental Procedure}

Experiments for each robot-environment pair are performed using the following procedure.

\subsubsection{Digraph Construction}
Before beginning experiments, two generic unweighted digraphs are constructed as described in Section \ref{Sec:Basics}~that are specific to the \TriSoRo~and \TetraSoRo, respectively. These digraphs are environment-agnostic and are used to abstractly model all robot states and motion primitives. The \TriSoRo~has 8 robot states with 56 motion primitives while the \TetraSoRo~has 16 robot states with 240 motion primitives. The objective of the experiments is for the robots to learn environment-specific digraph weights and validate the synthesized environment-specific gaits.

\subsubsection{Edge Weight Learning Experiments}
 The robot executes five randomized Eulerian cycles (Algorithm  \ref{Algo:ModifiedHeriholzers}) and the resulting motion is recorded, as detailed in Section \ref{Subsec:Hierholzer}. The total experiment time for the five trials of a given robot-environment pair is determined by $t_{\text{total}} = 5m\tau$ for $m$ motion primitives with $\tau$ time constant and is $5\times(56)\times(550 \text{ ms}) = 2.6$ minutes for the \TriSoRo~and $5\times(240)\times(450 \text{ ms}) = 9.0$ minutes for the \TetraSoRo. These data are then processed in MATLAB to extract and sort the digraph weights as measured in the initial vertex coordinate frame. 

\subsubsection{Gait Synthesis}
Once the edge weights of the digraph have been calculated, the translation and rotation gaits are synthesized separately using the gait synthesis algorithm, \Fig \ref{Fig:flowchart}. The intervals for the swept hyperparameters $\alpha_t$ and $\alpha_\theta$ are set to $[-1,1]$ to permit translation in all directions and both clockwise and counterclockwise rotation, respectively. Expected gait motion is then calculated for each gait and the remaining hyperparameters ($\beta$ and $\gamma$) are manually adjusted as needed to achieve useful gait motion. Tolerances ($\varepsilon_t$ and $\varepsilon_\theta$) are also chosen depending on the motion and size of the robot. Third, each gait is tested for a total of 120 motion primitives (66 seconds for the \TriSoRo~and 54 seconds for the \TetraSoRo) to give sufficient time to investigate the motion and variance of the gaits. 

\subsubsection{Gait Validation Experiments}
Finally, the gait experiment data (robot trajectories and rotation over time) are processed and aligned to start at the same initial global orientation of the robot at the origin of the coordinate grid. The optimality of the optimization is explored in Section \ref{subsubsec:optimality}. 

% \subsection{Experiment Matrix}
% The objectives of the experiments are to investigate and analyze (1) the extent to which robot-environment interactions affect locomotion, and (2) the applicability and effectiveness of the proposed gait synthesis strategy. As such, \hl{Table x} presents the matrix of planned experiments. Gait synthesis is performed for Robot 1 (\hl{TerraSoro}) on three different substrates: (1) a black rubber garage mat, (2) a whiteboard, and (3) a carpet. Three translation gaits and three rotation gaits are synthesized for each substrate and the results are recorded. For Robot 2 (\hl{TetraSoro}), gait synthesis is only performed on one substrate (rubber mat), but is also performed in the case of simulated actuator failure (loss-of-limb). Finally, both robots also execute intuitive/ad-hoc gaits to compare to synthesized gaits.

% \begin{center}
%     \begin{tabular}{|c|c?c|c|c|c|c|c|}
%     \hline
%            \multicolumn{2}{|c|}{} & \multicolumn{3}{c|}{Substrate} & \multicolumn{3}{c|}{Gaits}\\
%            \cline{3-8}
%             \multicolumn{2}{|c|}{} &  &  &  &  \multicolumn{2}{c|}{Synthesized} & \\
%             \cline{6-7}
%             \multicolumn{2}{|c|}{} & 1 & 2 & 3 &  Normal & Loss-of-Limb & Intuitive\\
%             \Xhline{1.5pt}
%             &&&&&&&\\[-.5em]
%             \parbox[t]{2mm}{\multirow{2}{*}{\rotatebox[origin=c]{90}{ Robot }}} & 1 & x & x & x & x & & x\\
%             \cline{2-8} &&&&&&&\\[-.5em]
%             & 2 & x & &  & x & x& x\\
%             \hline
%     \end{tabular}
% \end{center}

\section{Results and Discussion} \label{Sec:Results}
\Tab\ref{Tab:ExpMat} presents the matrix of planned experiments aimed at investigating and analyzing the extent to which robot-environment interactions affect locomotion, and the applicability and effectiveness of the proposed gait synthesis strategy. First, both robots execute intuitive gaits that are chosen manually. Both robots lack direct robotic analogues (due to the specific combined inching and rotating motion of the limbs) and direct biological analogues (due to being rotationally symmetric with three or four limbs). Therefore, three intuitive gaits for each robot are chosen based loosely on gaits that have successfully been performed on rotationally symmetric robots and starfish. Then, gait synthesis is performed for Robot 1 (\TriSoRo) on three different substrates: 1) a black rubber garage mat; 2) a whiteboard; and 3) a carpet. Three translation gaits and three rotation gaits are synthesized for each substrate, and the results are recorded. For Robot 2 (\TetraSoRo), gait synthesis is only performed on one substrate (rubber mat), but is also performed in the case of simulated actuator failure (loss of limb scenario).

\begin{table}[ht]
\captionsetup{justification=centering, labelsep=newline,textfont=sc}
    \caption{Experiment Matrix.}
        \label{Tab:ExpMat}
\centering
\renewcommand{\arraystretch}{1.3}
    \noindent\begin{tabular*}{\columnwidth}{@{\extracolsep{\fill}}*{7}{c}@{}}
    \hline
            \multirow{2}{*}{\hfil Robot} & \multicolumn{3}{c}{Substrate} & Intuitive & \multicolumn{2}{c}{Synthesized Gaits}\\
            \cline{2-4} \cline{6-7}
           & 1 & 2 & 3 & Gaits & Standard & Loss-of-Limb\\
            \hline \hline
            1 & $\bm{\times}$& $\bm{\times}$ & $\bm{\times}$ & $\bm{\times}$& $\bm{\times}$&\\
             2 & $\bm{\times}$ & &  & $\bm{\times}$ & $\bm{\times}$ & $\bm{\times}$\\
            \hline
    \end{tabular*}
\end{table}

 Analysis of preliminary experiments indicated a significant influence of the tether placement and direction on robot behavior. This was attributed to the weight and relative stiffness of the tether, as well as swinging of the tether when connected to an overhead component or interference with the robot when allowed to drag on the floor. To address this, the tether was redesigned to be skinnier, lighter, and more flexible. The tether connector was changed to include a vertical offset to avoid interference with the robot, and a slip ring was added to the four-limb robot. Repeatability experiments were then conducted by executing the same locomotion gait with the tether placed in different orientations relative to the robot; the results showed a significant reduction in locomotion variation due to the redesigned tether, resulting in an acceptable level of repeatability.

Readers may refer to the attached multimedia for the experimental videos of selected synthesized gaits and a visual outline of the pMFC framework.

\subsection{Robot 1: \TriSoRo}
This section details the results of both intuitive/ ad hoc gaits and synthesized gaits for the three-limb robot on three different substrates. %
The optimality of the BILP formulation approach is analyzed by comparing the results with those of an \textbf{exhaustive gait search} where the nonlinear cost functions \eqref{Eqn:nlcostfuncs} are evaluated for all possible locomotion gaits. 
%Furthermore, \textbf{exhaustive gait search} (nonlinear cost functions calculated for all possible simple cycle gaits) is performed to compare to the BILP formulation. 
This analysis is only performed for the three-limb robot as increasing the number of limbs (even to four limbs) leads to combinatorial explosion and the intractability of exhaustive gait search.
\subsubsection{Intuitive Gaits}
First, manually chosen gaits are executed on all three substrates. These gaits, detailed in \Tab\ref{Tab:Int3}, are loosely inspired by the actuation sequences of pronk and rotary gallop gaits that are commonly executed in rotationally symmetric robots\cite{Masuda_2017jrm,7989652}, and push-pull gaits observed in biological starfish and starfish-inspired robots\cite{MAO2014400}, respectively. The sequence of discrete robot states (vertices), $V(L)$, for each gait is shown. Shaded (maroon) limbs represent limb actuation. As the actuation of each limb is designed to have a rotational component and is rotationally symmetric about the robot's center hub, the intuitive belief is that the symmetric actuation sequences of both pronk and rotary gallop gaits would induce only rotation.

\begin{table}[hb]
\renewcommand\arraystretch{2}
\begin{center}
\caption{\TriSoRo~Intuitive Gaits (Not Synthesized).}
\label{Tab:Int3}
\begin{tabular*}{\columnwidth}{@{\extracolsep{\fill}}cccc@{}}
\hline           
Gait & $V(L)$ & Robot States & Inspiration\\           \hline \hline         
$L_1$ &  $[1,8]$ & \raisebox{-.3\totalheight}{\includegraphics[width=.5cm]{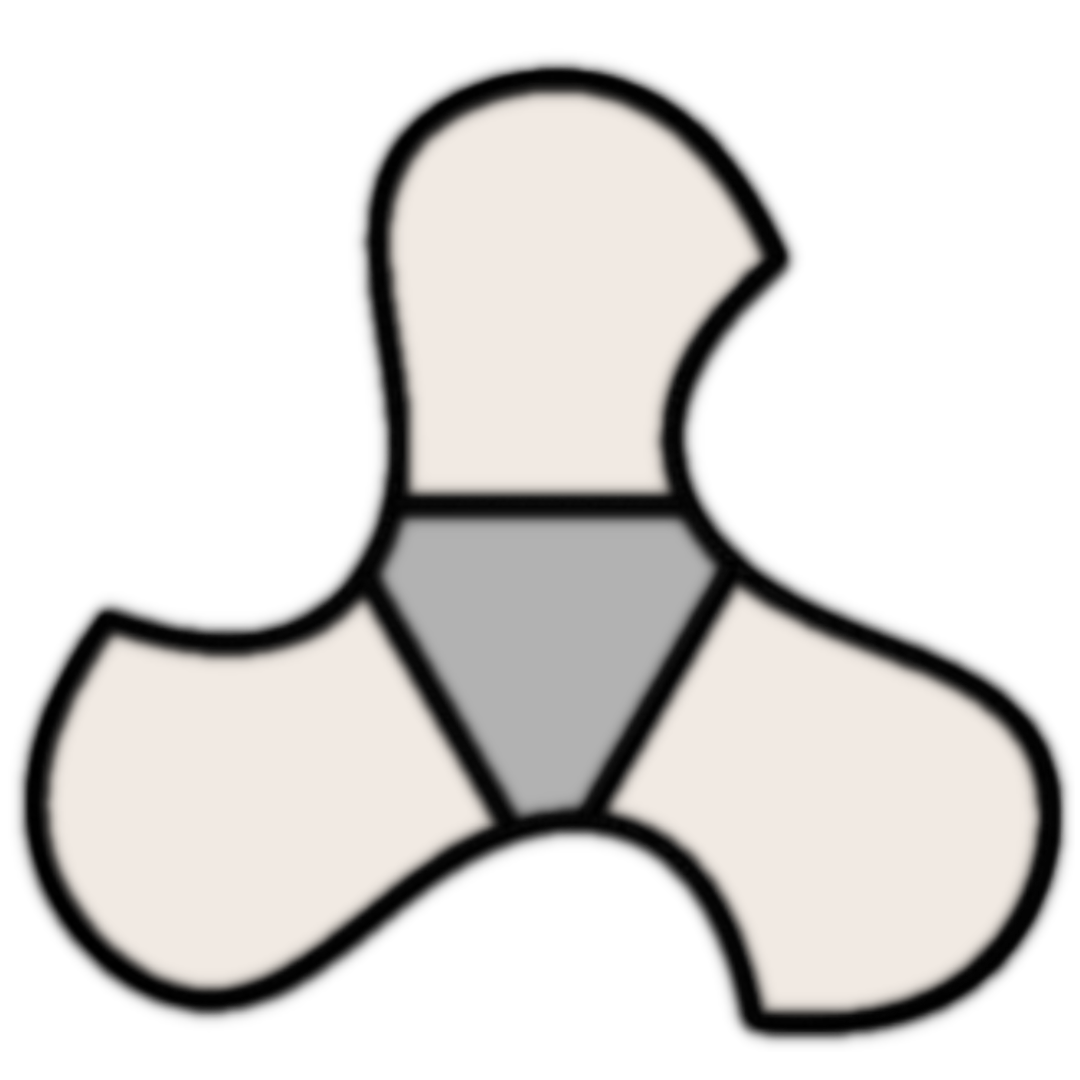}} \kern-.3em $\pmb{\rightarrow}$  \kern-.3em \raisebox{-.3\totalheight}{\includegraphics[width=.5cm]{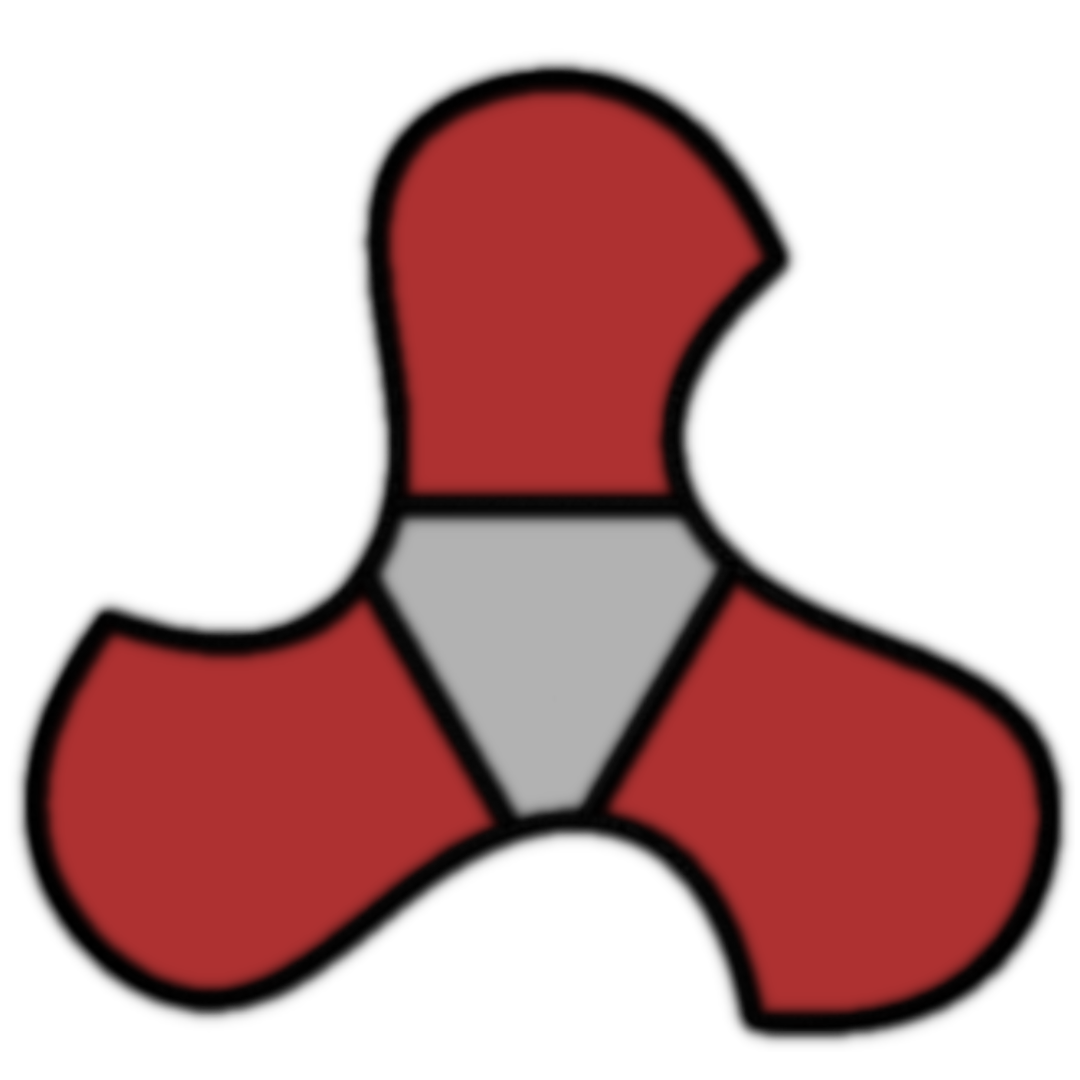}}& pronk\\            
        
$L_2$ &  $[2,3,5]$ & \raisebox{-.3\totalheight}{\includegraphics[width=.5cm]{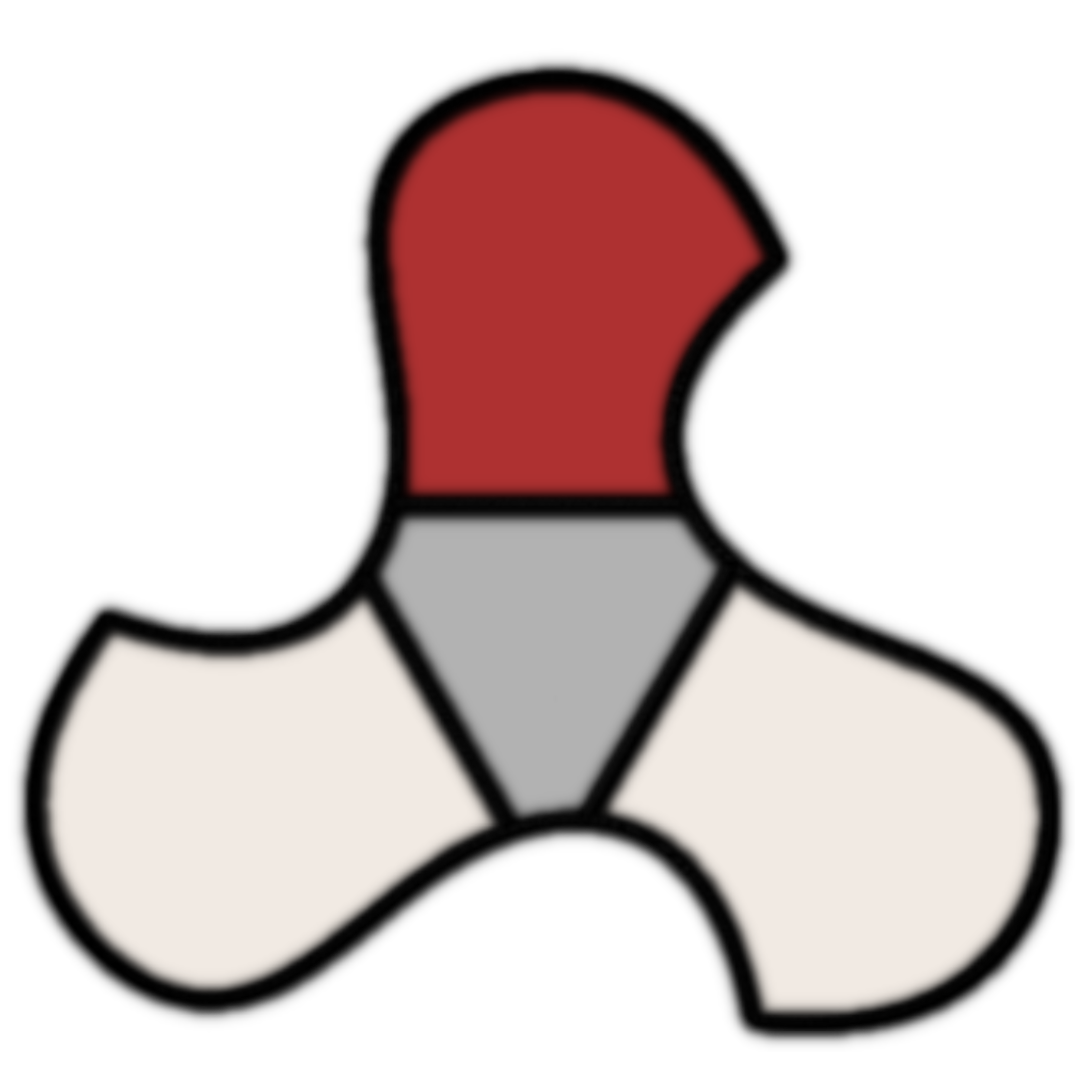}} \kern-.3em $\pmb{\rightarrow}$  \kern-.3em \raisebox{-.3\totalheight}{\includegraphics[width=.5cm]{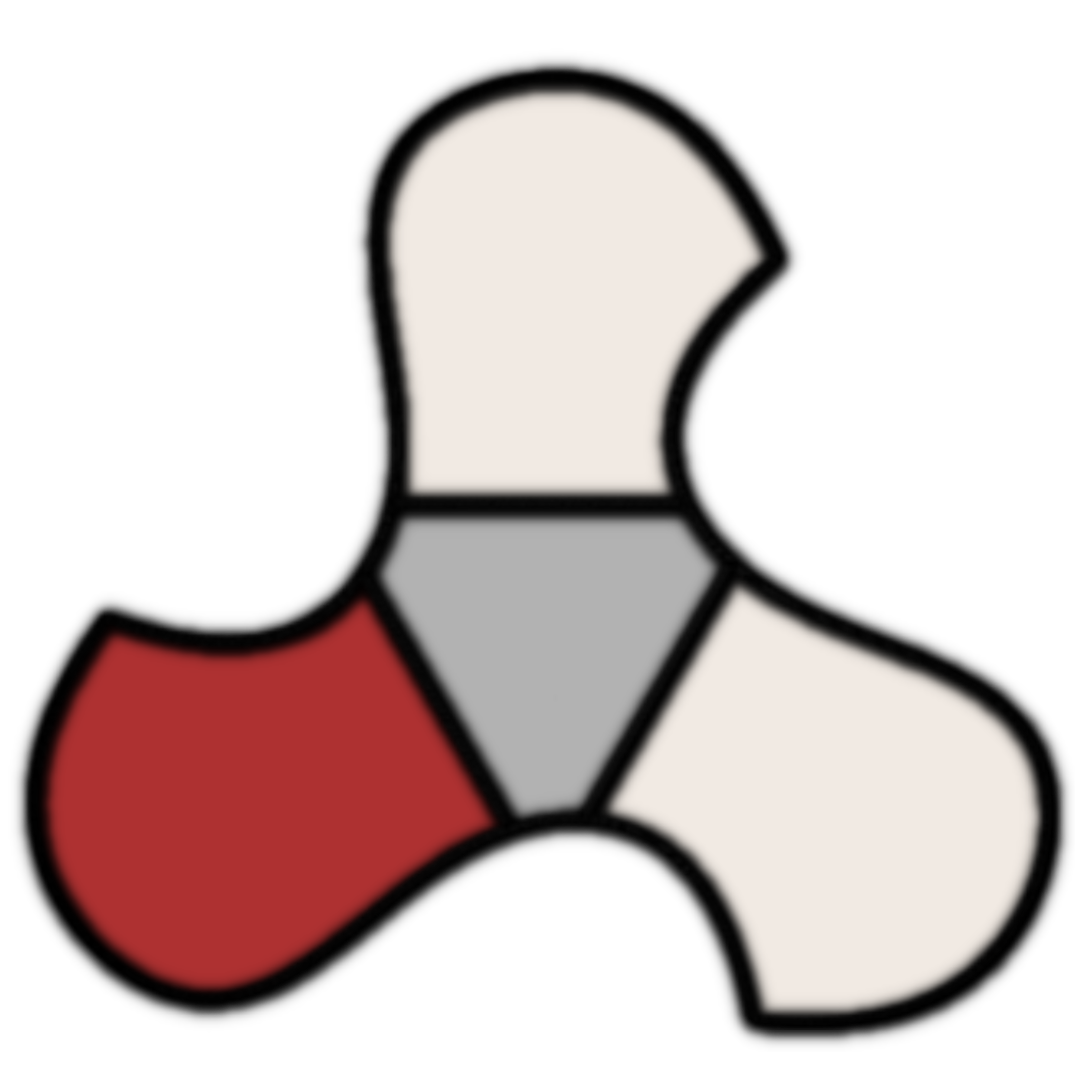}}$\pmb{\rightarrow}$  \kern-.3em \raisebox{-.3\totalheight}{\includegraphics[width=.5cm]{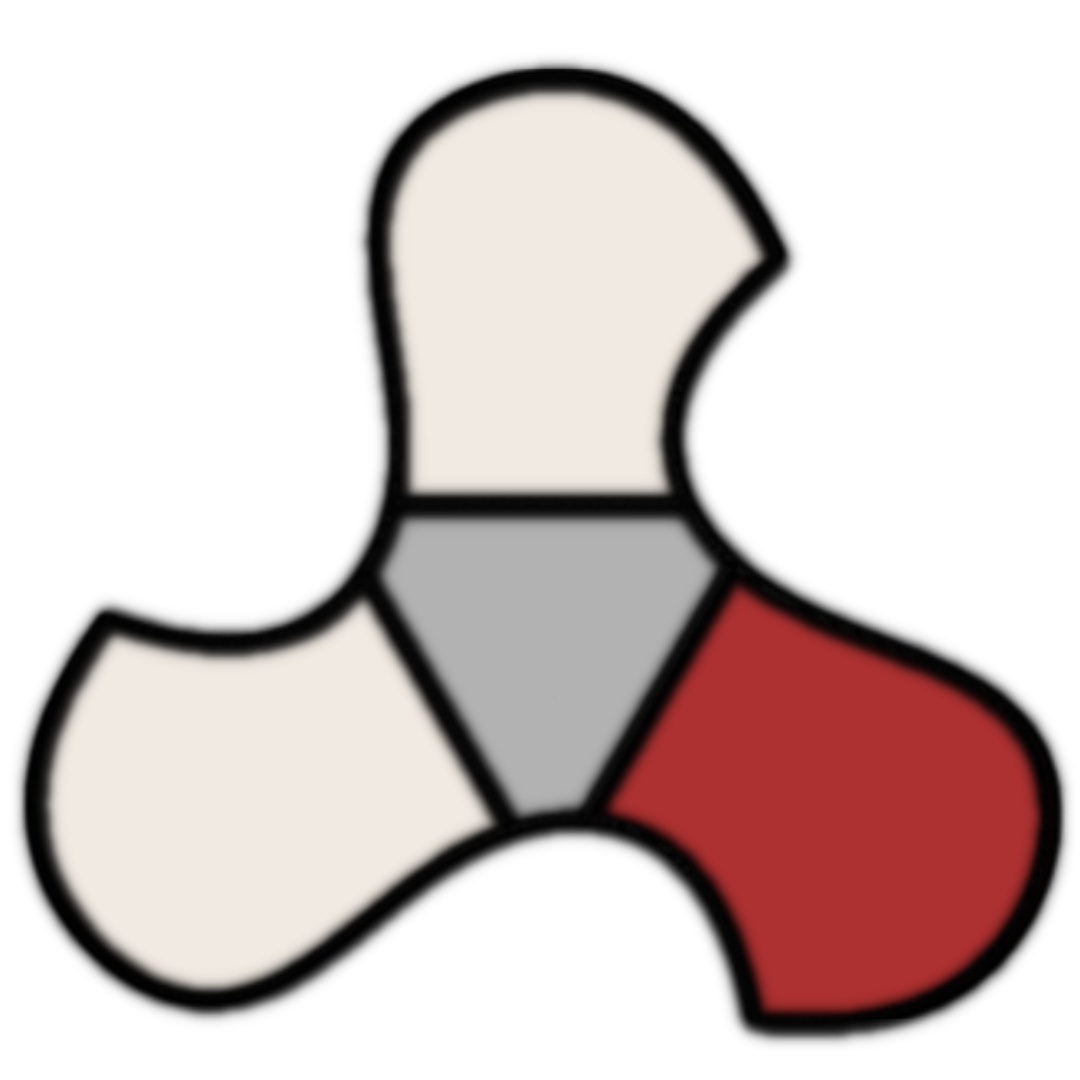}} & rotary gallop\\            
            
$L_3$ &  $[5,8,4,1]$ &\raisebox{-.3\totalheight}{\includegraphics[width=.5cm]{figures/MTA3_State_5}} \kern-.3em $\pmb{\rightarrow}$  \kern-.3em \raisebox{-.3\totalheight}{\includegraphics[width=.5cm]{figures/MTA3_State_8}}$\pmb{\rightarrow}$  \kern-.3em \raisebox{-.3\totalheight}{\includegraphics[width=.5cm]{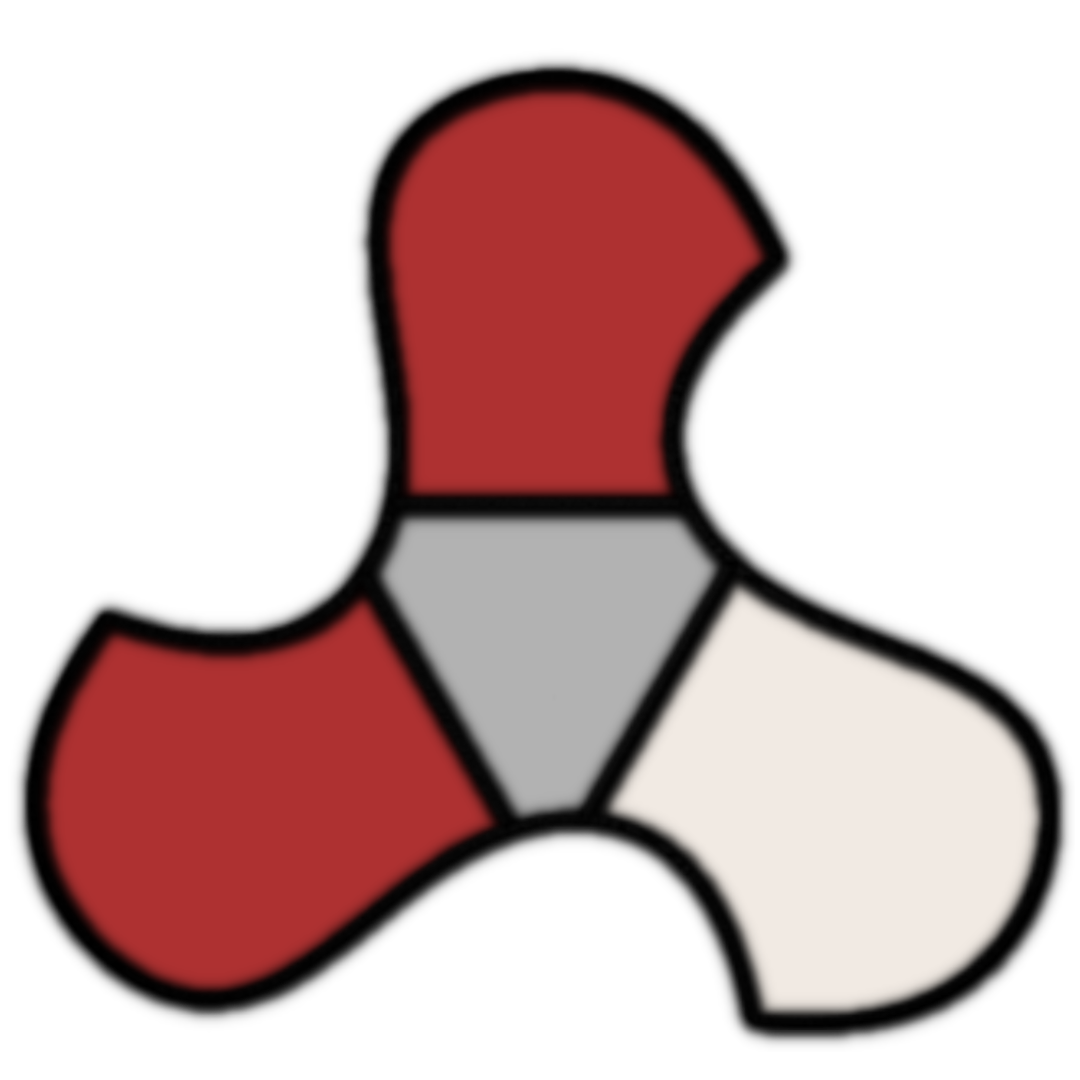}} \kern-.3em $\pmb{\rightarrow}$  \kern-.3em \raisebox{-.3\totalheight}{\includegraphics[width=.5cm]{figures/MTA3_State_1}}& push-pull\\            
\hline    
\end{tabular*}
\end{center}
\end{table}
\begin{figure}[ht]
\centering
    \includegraphics[width=.9\columnwidth,trim=1.5cm 23.5cm 3cm 3.4cm, clip=true]{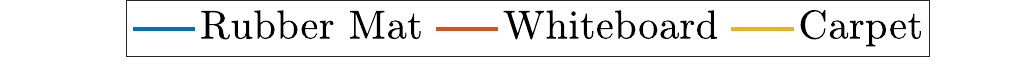} \\[-2.2ex] 

    \subfloat[][]{\includegraphics[width = .43\columnwidth,trim=1.5cm 0.2cm 2.5cm .3cm, clip=true]{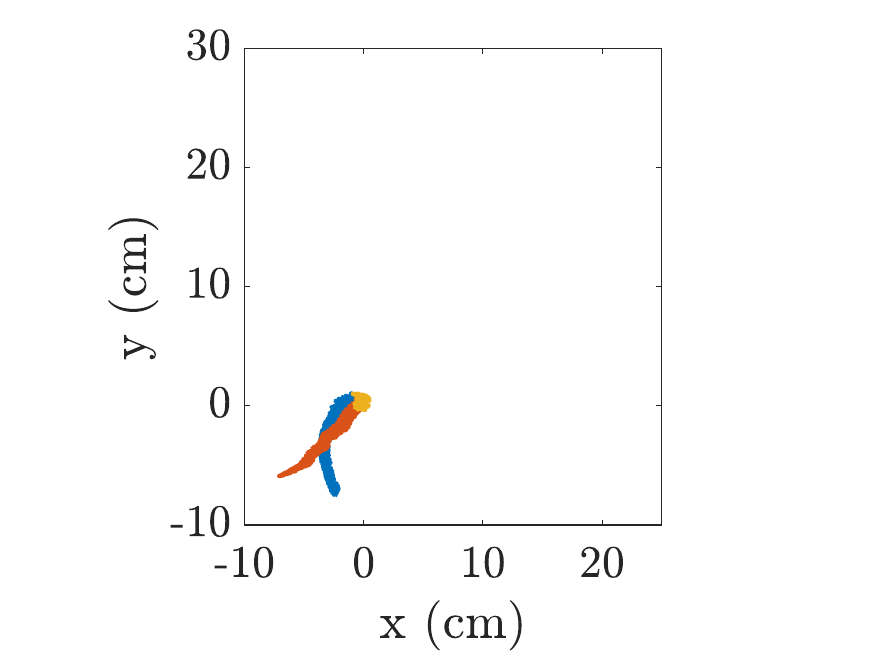}} \label{7a}\hfill
    \subfloat[][]{\includegraphics[width = .54\columnwidth,trim= 0cm 0.1cm 1cm .3cm, clip=true]{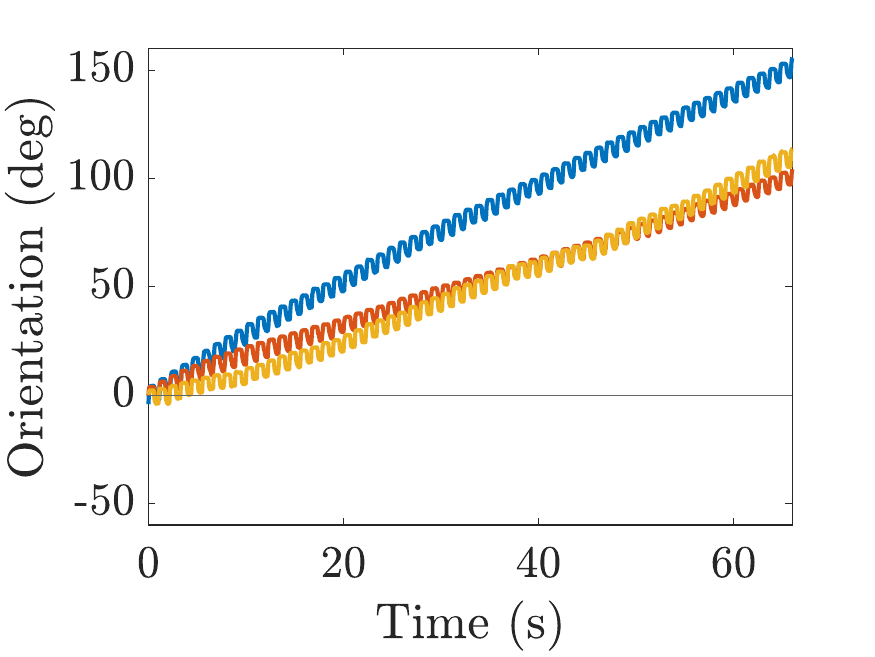}}\\[-2ex] 
    \subfloat[][]{\includegraphics[width = .43\columnwidth,trim=1.5cm 0cm 2.5cm 0cm, clip=true]{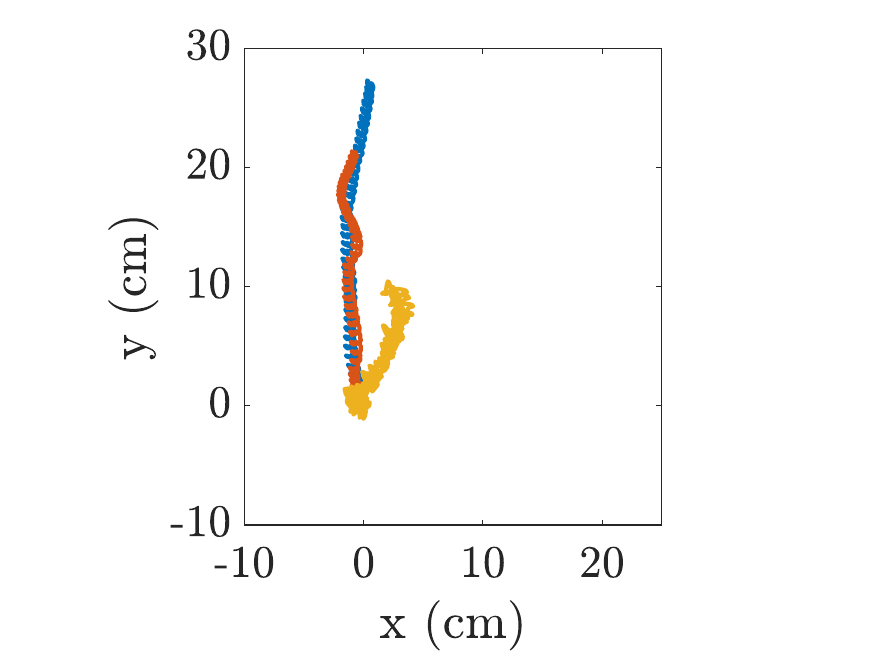}}\hfill
    \subfloat[][]{\includegraphics[width = .54\columnwidth,trim= 0cm 0cm 1cm 0cm, clip=true]{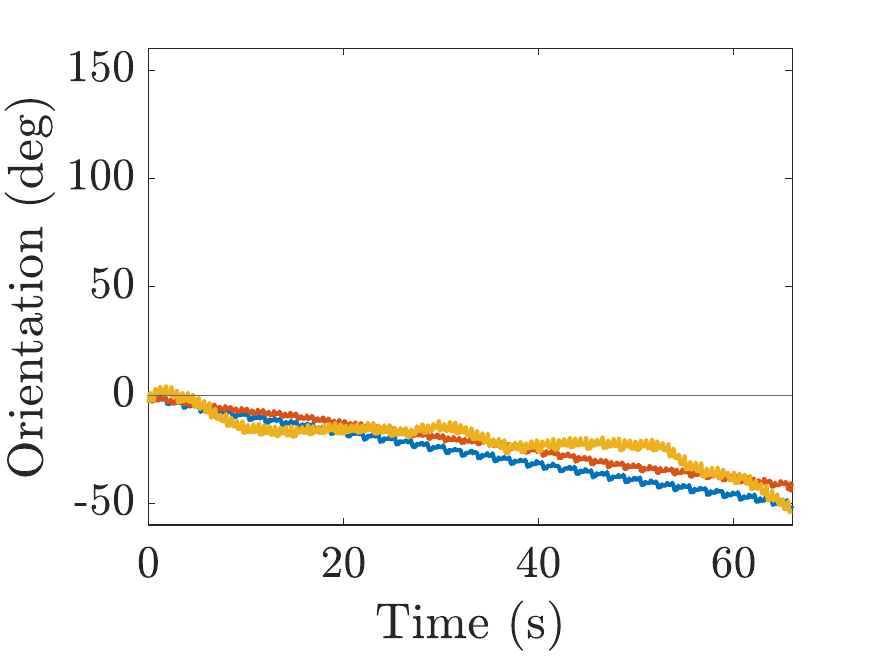}}\\[-2ex] 
    \subfloat[][]{\includegraphics[width = .43\columnwidth,trim=1.5cm 0cm 2.5cm 0cm, clip=true]{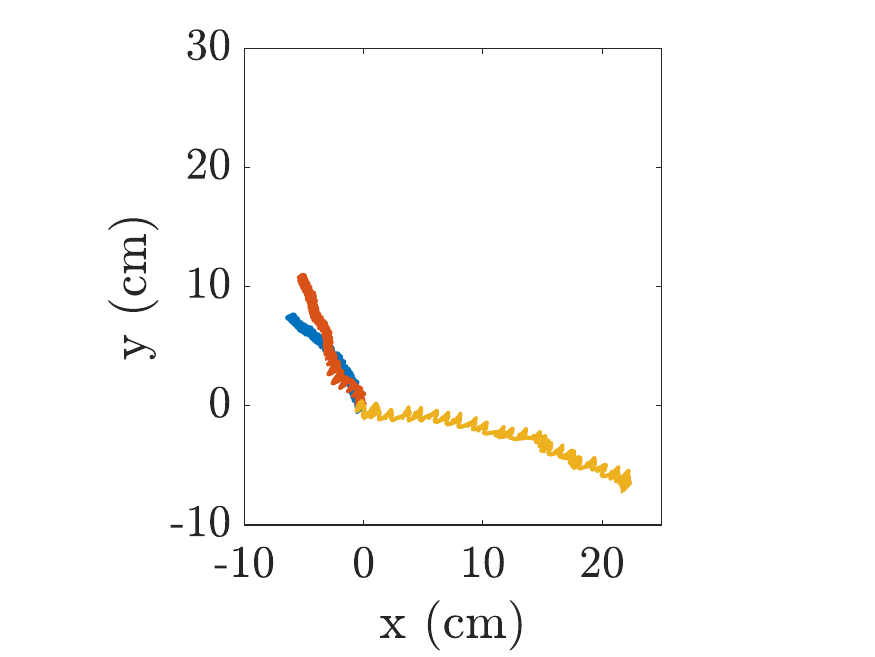}}\hfill
    \subfloat[][]{\includegraphics[width = .54\columnwidth,trim= 0cm 0cm 1cm 0cm, clip=true]{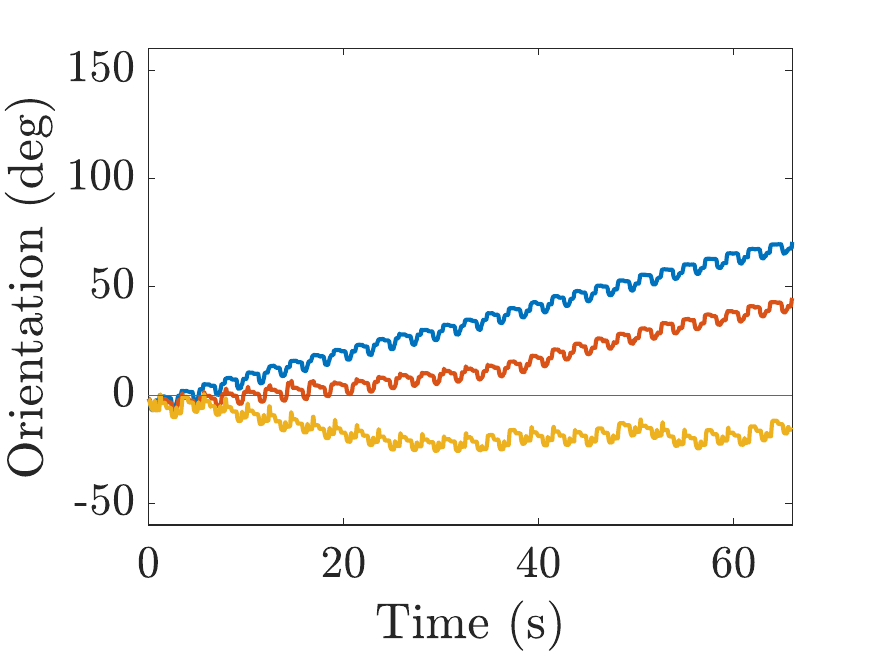}}\par
    \caption{Experimental results for \TriSoRo~performance of the intuitive gaits (Table \ref{Tab:Int3}) on three different substrates. (a,c,e) Robot trajectories and (b,d,f) global robot orientation over time are plotted for gaits $L_1$, $L_2$, and $L_3$, respectively. Each plot contains experimental data for three different substrates: rubber mat (blue), whiteboard (orange), and carpet (yellow). }
    \label{Fig:Int3}
\end{figure}
However, Fig.~\ref{Fig:Int3}(a-b) shows that the pronk-like gait results in a small amount of translation on two of the substrates, indicating some unexpected asymmetrical distribution of forces. Similarly, the rotary gallop-like gait, Fig.~\ref{Fig:Int3}(c-d), results in large net translation on two of the substrates. Finally, the push-pull gait, Fig.~\ref{Fig:Int3}(e-f), which is expected to be translation-dominant, only achieves useful translation on carpet. These results suggest that unexpected asymmetries (which can result from several sources including manufacturing inaccuracies and surface defects of the substrate) can have substantial influence on gait behavior, hence, motivating a more data-driven approach. Further, these gaits have different performance on different substrates; the push-pull gait translates in almost the opposite direction on carpet as it does on the rubber mat and the whiteboard. This further justifies an environment-centric gait synthesis strategy.

\begin{table}
\renewcommand\arraystretch{2}
\begin{center}
\caption{\TriSoRo~Synthesized Gaits on Substrate 1 (Rubber Mat).}
\label{Tab:sub1}
\noindent\begin{tabular*}{\columnwidth}{@{\extracolsep{\fill}}ccc@{}}
\hline        
Gait & $V(L)$ & Robot States \\           \hline \hline  
$^1L_{t1}$ &  $[6,3]$ & \raisebox{-.3\totalheight}{\includegraphics[width=.5cm]{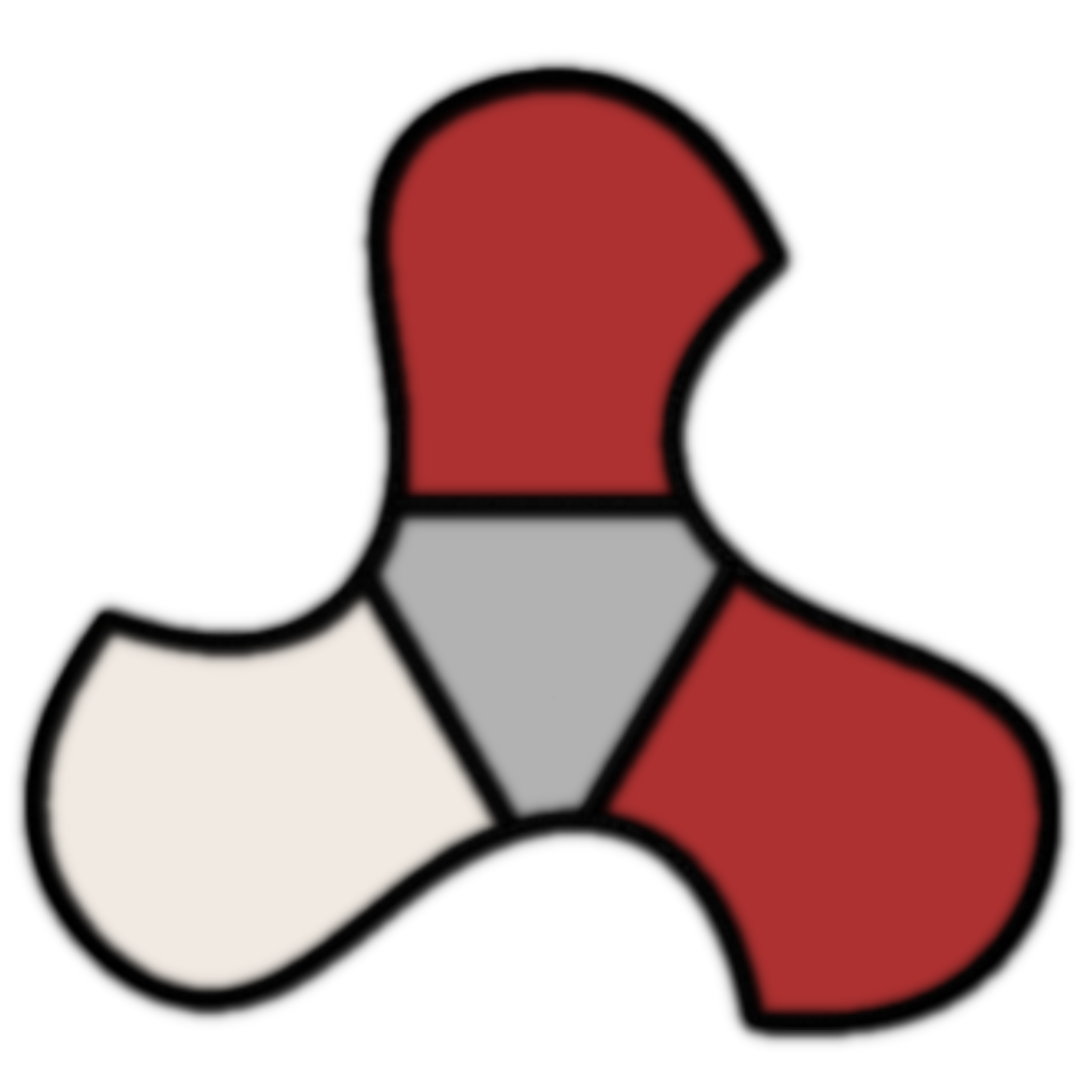}} \kern-.3em $\pmb{\rightarrow}$  \kern-.3em \raisebox{-.3\totalheight}{\includegraphics[width=.5cm]{figures/MTA3_State_3}}\\            
        
$^1L_{t2}$ &  $[8,2,1]$ & \raisebox{-.3\totalheight}{\includegraphics[width=.5cm]{figures/MTA3_State_8}} \kern-.3em $\pmb{\rightarrow}$  \kern-.3em \raisebox{-.3\totalheight}{\includegraphics[width=.5cm]{figures/MTA3_State_2}}$\pmb{\rightarrow}$  \kern-.3em \raisebox{-.3\totalheight}{\includegraphics[width=.5cm]{figures/MTA3_State_1}} \\            
            
$^1L_{t3}$ &  $[6,4,5,7,1]$ &\raisebox{-.3\totalheight}{\includegraphics[width=.5cm]{figures/MTA3_State_6}} \kern-.3em $\pmb{\rightarrow}$  \kern-.3em \raisebox{-.3\totalheight}{\includegraphics[width=.5cm]{figures/MTA3_State_4}}$\pmb{\rightarrow}$  \kern-.3em \raisebox{-.3\totalheight}{\includegraphics[width=.5cm]{figures/MTA3_State_5}} \kern-.3em $\pmb{\rightarrow}$  \kern-.3em \raisebox{-.3\totalheight}{\includegraphics[width=.5cm]{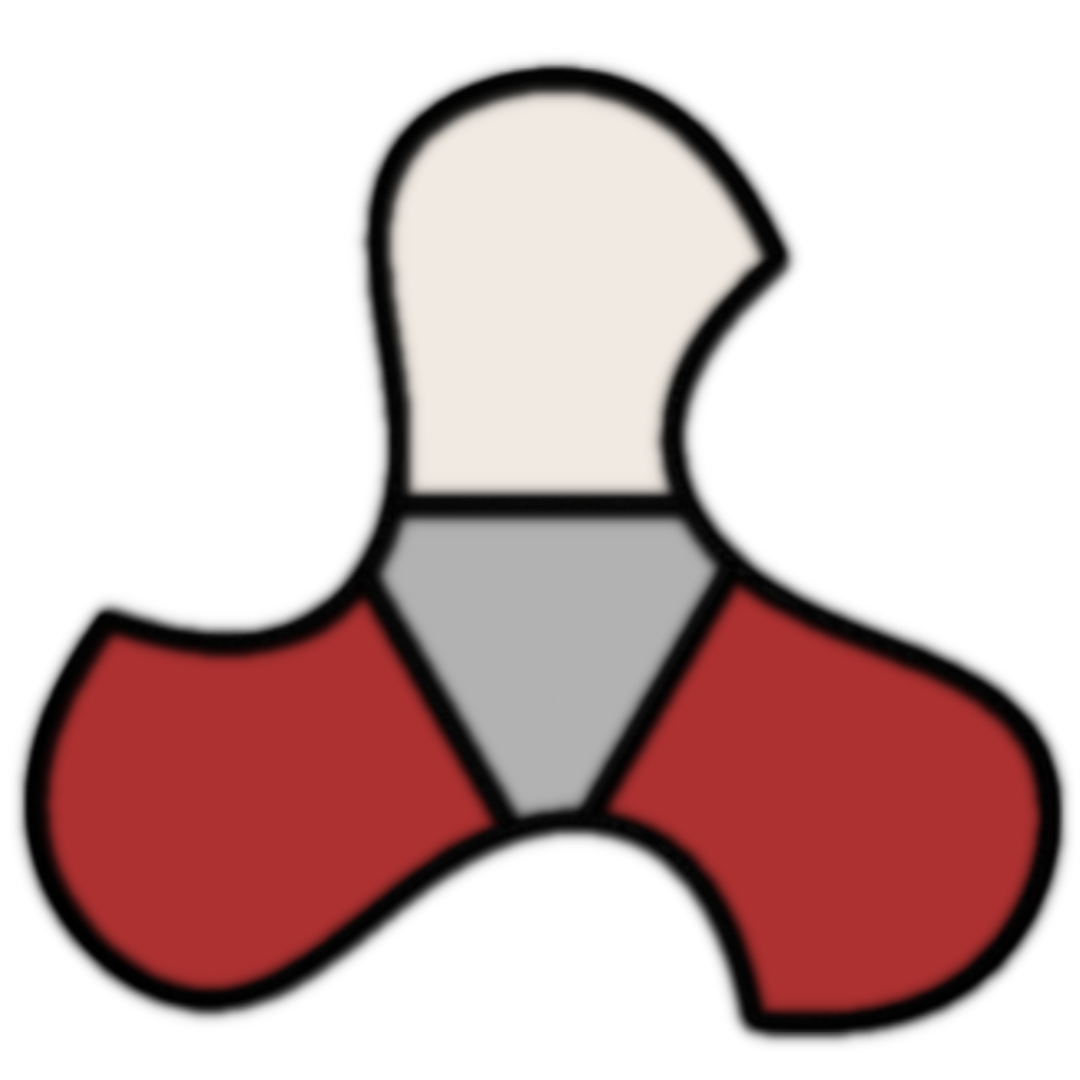}} \kern-.3em $\pmb{\rightarrow}$  \kern-.3em \raisebox{-.3\totalheight}{\includegraphics[width=.5cm]{figures/MTA3_State_1}}\\            
\hline    \hline
$^1L_{\theta 1}$ &  $[4,2,1]$ & \raisebox{-.3\totalheight}{\includegraphics[width=.5cm]{figures/MTA3_State_4}} \kern-.3em $\pmb{\rightarrow}$  \kern-.3em \raisebox{-.3\totalheight}{\includegraphics[width=.5cm]{figures/MTA3_State_2}} \kern-.3em $\pmb{\rightarrow}$  \kern-.3em \raisebox{-.3\totalheight}{\includegraphics[width=.5cm]{figures/MTA3_State_1}}\\            
        
$^1L_{\theta 2}$ &  $[3, 4, 2, 6, 7, 1]$ & \raisebox{-.3\totalheight}{\includegraphics[width=.5cm]{figures/MTA3_State_3}} \kern-.3em $\pmb{\rightarrow}$  \kern-.3em \raisebox{-.3\totalheight}{\includegraphics[width=.5cm]{figures/MTA3_State_4}}$\pmb{\rightarrow}$  \kern-.3em \raisebox{-.3\totalheight}{\includegraphics[width=.5cm]{figures/MTA3_State_2}} \kern-.3em $\pmb{\rightarrow}$  \kern-.3em \raisebox{-.3\totalheight}{\includegraphics[width=.5cm]{figures/MTA3_State_6}}  \kern-.3em $\pmb{\rightarrow}$ \kern-.3em \raisebox{-.3\totalheight}{\includegraphics[width=.5cm]{figures/MTA3_State_7}} \kern-.3em $\pmb{\rightarrow}$  \kern-.3em \raisebox{-.3\totalheight}{\includegraphics[width=.5cm]{figures/MTA3_State_1}} \\            
            
$^1L_{\theta 3}$ &  $[7, 5, 6, 2,4,1]$ &\raisebox{-.3\totalheight}{\includegraphics[width=.5cm]{figures/MTA3_State_7}} \kern-.3em $\pmb{\rightarrow}$  \kern-.3em \raisebox{-.3\totalheight}{\includegraphics[width=.5cm]{figures/MTA3_State_5}}$\pmb{\rightarrow}$  \kern-.3em \raisebox{-.3\totalheight}{\includegraphics[width=.5cm]{figures/MTA3_State_6}} \kern-.3em $\pmb{\rightarrow}$  \kern-.3em \raisebox{-.3\totalheight}{\includegraphics[width=.5cm]{figures/MTA3_State_2}} \kern-.3em $\pmb{\rightarrow}$  \kern-.3em \raisebox{-.3\totalheight}{\includegraphics[width=.5cm]{figures/MTA3_State_4}} \kern-.3em $\pmb{\rightarrow}$ \kern-.3em \raisebox{-.3\totalheight}{\includegraphics[width=.5cm]{figures/MTA3_State_1}}\\            
\hline    
\end{tabular*}
\end{center}
\end{table}

\subsubsection{Gait Synthesis for Substrate 1 (Rubber Mat)}

\Tab\ref{Tab:sub1} shows the synthesized gaits $^1L$ for the rubber mat. The results are divided into translation-dominant  gaits $^1L_{ti}$ and rotation-dominant gaits $^1L_{\theta i}$ where $i=1,2,3$. Experimental results suggest that the method is effective at finding both translation-dominant, Fig.~\ref{Fig:sub1}(a-b), and rotation-dominant gaits, Fig.~\ref{Fig:sub1}(c-d) (see multimedia video attachment). It should be noted that gait optimality depends on the application. For example, the translation-dominant gait $^1L_{t2}$ has a large translation amplitude but suffers from a rotational component that may be undesirable. On the other hand, the translation-dominant gait $^1L_{t3}$ has very stable straight-line motion but is not as fast as the other gaits. The rotation-dominant gait $^1L_{\theta 1}$ has very minimal translation but has less precise rotational motion compared to the other rotation gaits. The proposed methodology allows the researcher to vary hyperparameters and constraints accordingly to tune these very aspects (e.g., speed, precision, and translation/rotation separation). Furthermore, it permits omnidirectional translation and is able to find both clockwise and counter-clockwise rotational gaits. 

\begin{figure}[ht]
\centering
    \includegraphics[width=.7\columnwidth,trim=4cm 17.8cm 5cm 9.1cm, clip=true]{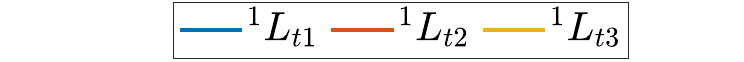} \\[-1.5ex] 
    \subfloat[][]{\includegraphics[width = .43\columnwidth,trim=2cm 0.2cm 2.2cm .8cm, clip=true]{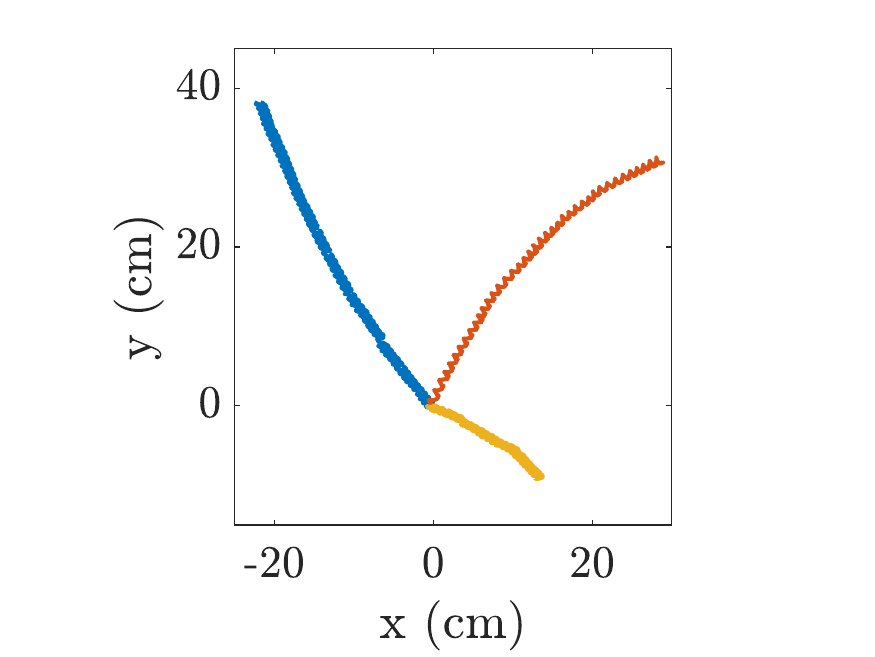}}\hfill
    \subfloat[][]{\includegraphics[width = .54\columnwidth,trim= 0cm 0.1cm 1.3cm .8cm, clip=true]{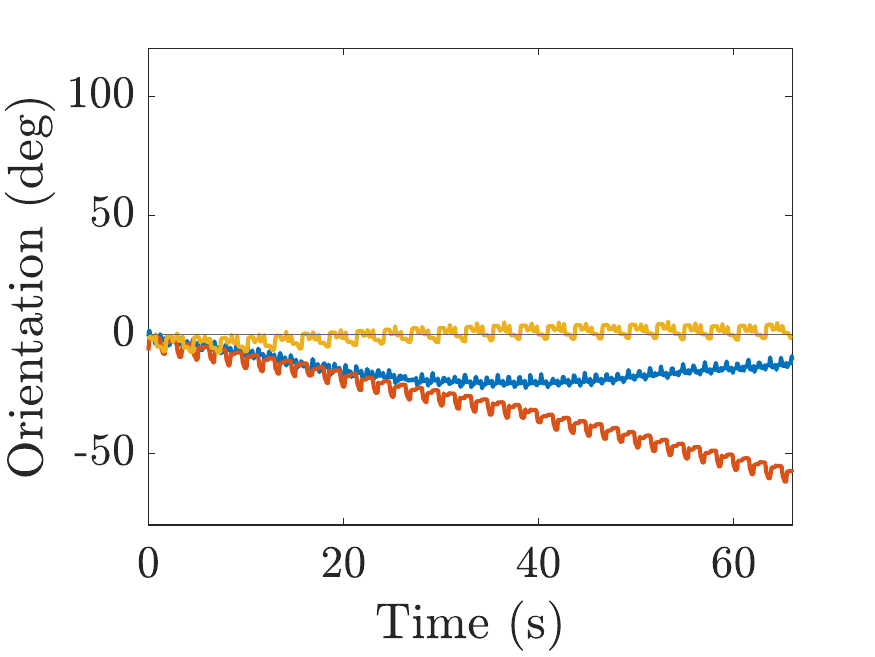}}\par
    
    \includegraphics[width=.6\columnwidth,trim=6cm 19.7cm 5cm 7cm, clip=true]{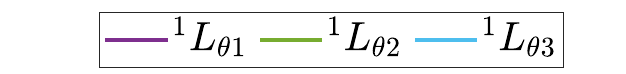} \\[-1.8ex] 
    \subfloat[][]{\includegraphics[width = .43\columnwidth,trim=2cm 0.2cm 2.2cm .8cm, clip=true]{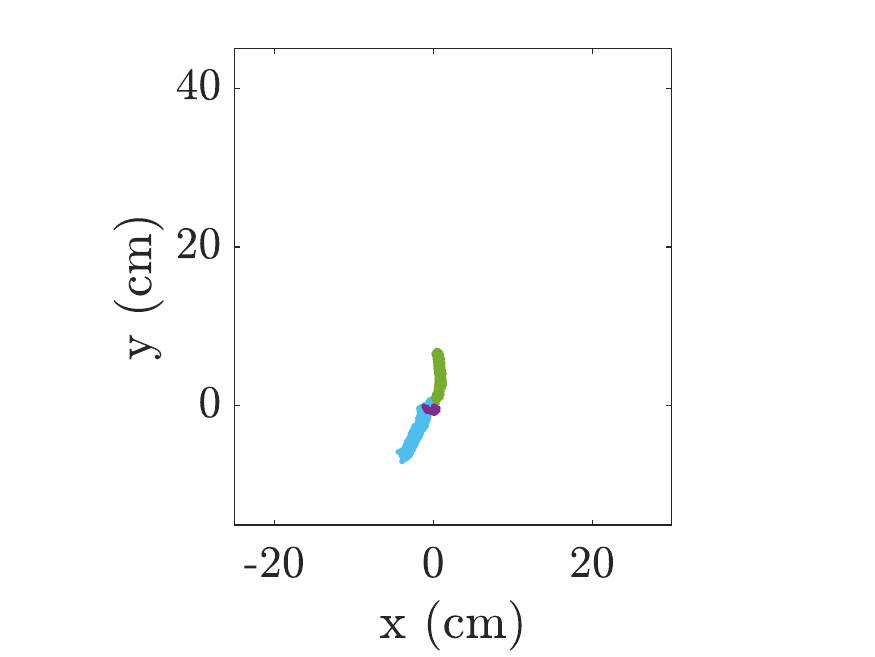}}\hfill
    \subfloat[][]{\includegraphics[width = .54\columnwidth,trim= 0cm 0.1cm 1.3cm .8cm, clip=true]{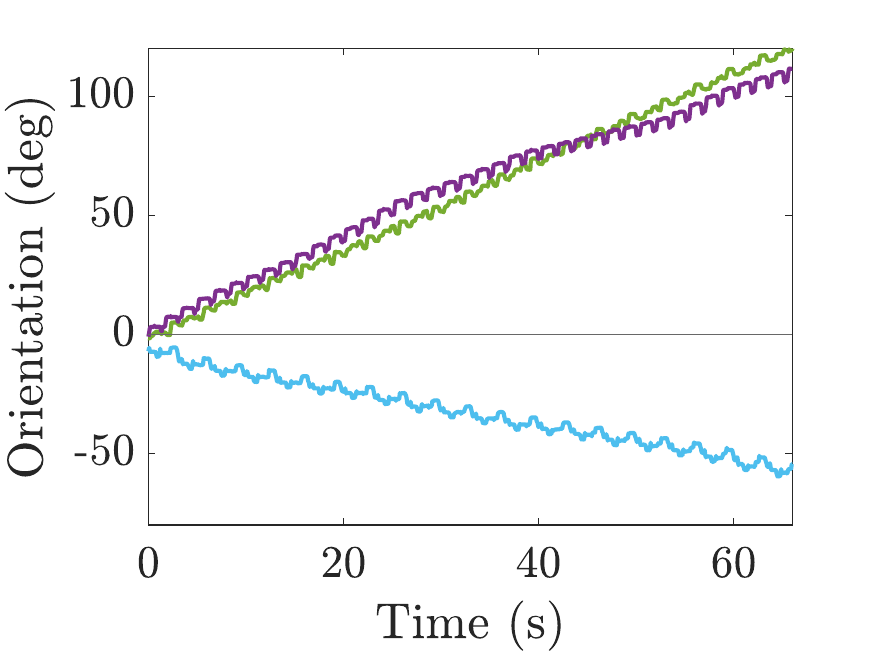}}\par

    \caption{Experimental results of \TriSoRo~gaits synthesized for the rubber mat (substrate 1). The top figures plot the translation-dominant gaits ($^1L_{t1}, ^1L_{t2}, ^1L_{t3}$) with (a) the robot trajectory and (b) the global robot orientation over time, respectively. The bottom figures plot the rotation-dominant gaits ($^1L_{\theta1}, ^1L_{\theta2}, ^1L_{\theta3}$) with (c) the robot trajectory and (d) the global robot orientation over time, respectively.}
    \label{Fig:sub1}
\end{figure}

One may wonder whether a gait that is synthesized to be optimal for one substrate will be optimal and/or effective on others. To investigate this, Fig.~\ref{Fig:1L}
visualizes the results of the translation-dominant gait $^1L_{t1}$ and rotation-dominant gait $^1L_{\theta 1}$ executed on all three substrates (see multimedia video attachment). Interestingly,  $^1L_{t1}$ performs comparably on the rubber mat and the whiteboard, but has a very large rotational component on the carpet. On the other hand, $^1L_{\theta 1}$ performs similarly for the rubber mat and carpet, but has a larger rotational component on the whiteboard. These results re-affirm the need for environment-centric and data-driven gait synthesis.
\begin{figure}[ht]
\centering
    \includegraphics[width=.9\columnwidth,trim=1.5cm 23.5cm 3cm 3.4cm, clip=true]{figures/substrate_legend.pdf} \\[-1.8ex]
    \subfloat[][]{\includegraphics[width = .43\columnwidth,trim=1.7cm 0.2cm 2cm 0.5cm, clip=true]{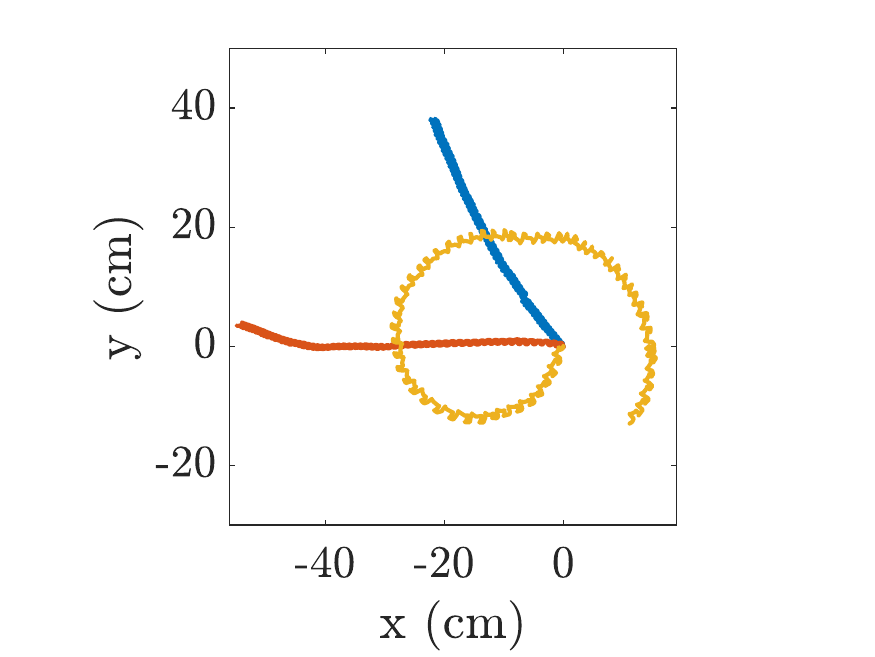}}\hfill
    \subfloat[][]{\includegraphics[width = .53\columnwidth,trim= 0cm 0.1cm 1.3cm .8cm, clip=true]{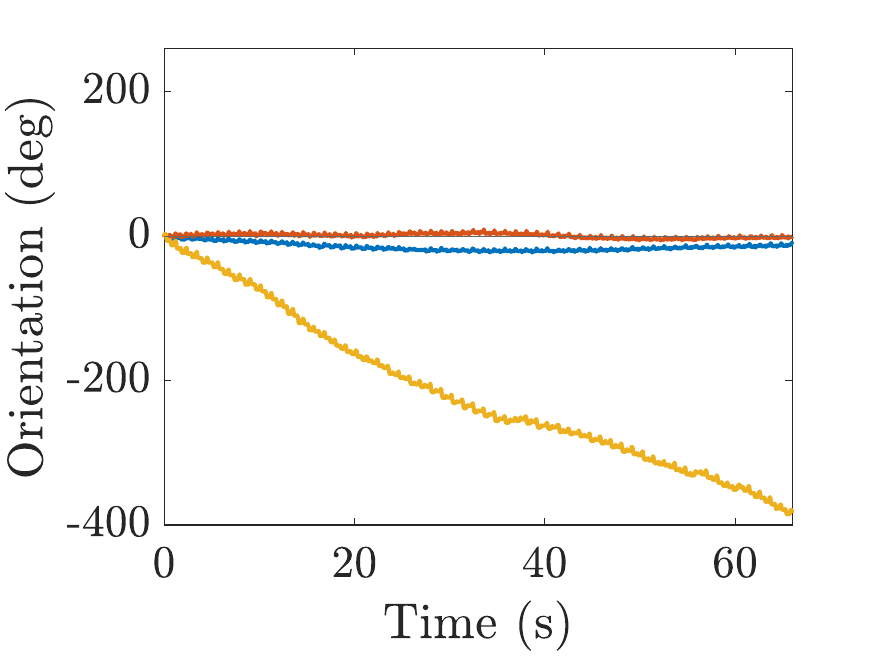}}\\[-1.5ex]

    \subfloat[][]{\includegraphics[width = .43\columnwidth,trim=1.7cm 0.2cm 2cm 0.5cm, clip=true]{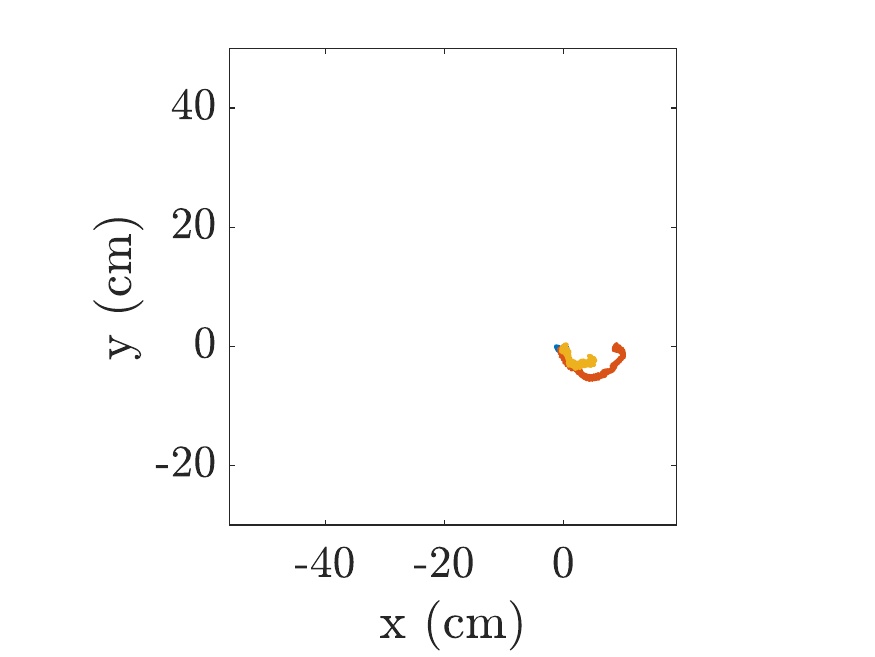}}\hfill
    \subfloat[][]{\includegraphics[width = .53\columnwidth,trim= 0cm 0.1cm 1.3cm .8cm, clip=true]{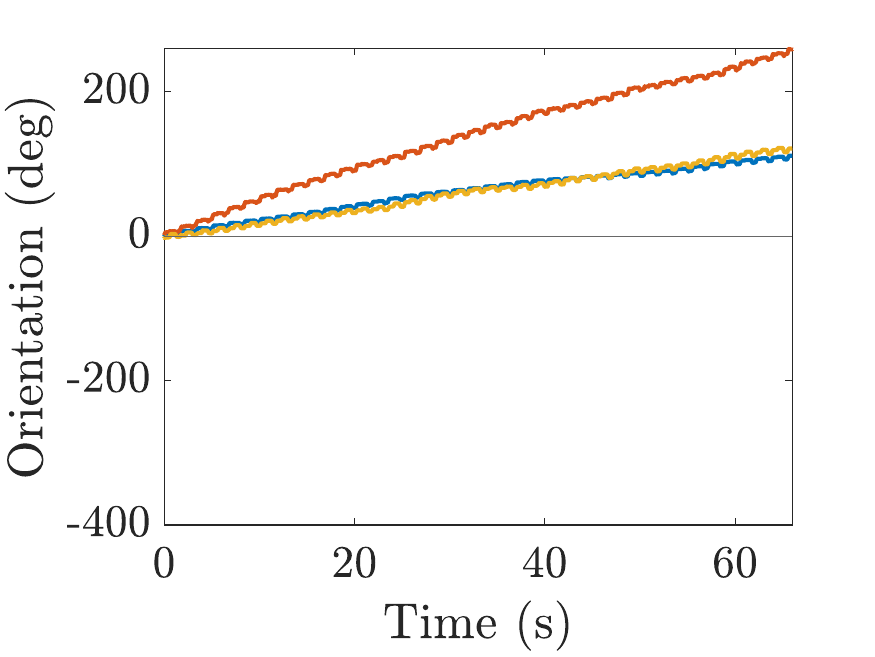}}\par
    \caption{\TriSoRo~gaits synthesized for the rubber mat executed on all three substrates. Translation-dominant gait $^1L_{t1}$ trajectory and orientation results are plotted in (a) and (b), respectively. Rotation-dominant gait $^1L_{\theta 1}$ trajectory and orientation results are plotted in (c) and (d), respectively.}
    \label{Fig:1L}
\end{figure}

\subsubsection{Gait Synthesis for Substrate 2 (Whiteboard)}
Synthesized gaits on the whiteboard are detailed in \Tab\ref{Tab:sub2} and experimental plots are presented in \Fig\ref{Fig:sub2} (see multimedia video attachment for $^2L_{t3}$). As can be seen, the gaits synthesized for this surface are different from the rubber mat ones, suggesting that gait motion and optimality are dependent on the environment (i.e., substrate). Nevertheless, the gait synthesis appears to be fairly successful, although one of the translation gaits $^2L_{t2}$ has a large rotational component. While the hyperparameters can be tuned to avoid this, the variance in the robot motion means that expected (predicted) gait behavior may still vary from its experimental trajectory. The translation-dominant gaits are of similar speed and magnitude to those tested on the rubber mat, whereas the rotation-dominant gaits achieve much faster and precise rotation. This may suggest that it is easier for the robot to achieve rotational motion on this substrate as compared to the rubber mat. 

\subsubsection{Gait Synthesis for Substrate 3 (Carpet)}
For the last substrate, the synthesized gaits for carpet are detailed in \Tab\ref{Tab:sub3} with the experimental results shown in \Fig\ref{Fig:sub3} (see multimedia video attachment for $^3L_{t3}$). Carpet is assumed to be the most difficult surface tested in this experiment as (a) the loop pile texture leads to a rougher and less flat substrate that results in a much more complex robot-environment interaction, especially as silicone material can deform and interact with small hills and valleys in the material, (b) the material is much less homogeneous than the other substrates as it can wear over time, and (c) experiments show that unlike the other substrates, the robot is not able to fully uncurl the limbs at times, leading to higher variance in the motion. All three of these factors contribute to violations of the assumptions of the methodology, which dictate repeatable quasi-static motion (e.g., consistent curling and uncurling of the limbs) and flat and homogeneous substrates. It is therefore anticipated that the results will be poorer for this substrate. Despite this, the synthesized gaits are still fairly successful, resulting in multi-directional translation-dominant gaits and rotation-dominant gaits in both clockwise and counter-clockwise directions. One of the translation-dominant gaits $^3L_{t1}$ has extremely limited translation, but very precise zero rotation. However, these gaits tend to be less precise (i.e., contain more variance) than those executed on the other substrates, particularly for the rotation-dominant gaits. For example, the rotation-dominant gait $^3L_{\theta2}$  has a large translation component and does not have a precise instantaneous center of rotation. It is also interesting that the synthesized rotation-dominant gaits have much longer sequences when compared to the other substrates. 
\begin{table}[ht]
\renewcommand\arraystretch{2}
\begin{center}
\caption{\TriSoRo~Synthesized Gaits on Substrate 2 (Whiteboard).}
\label{Tab:sub2}
\noindent\begin{tabular*}{\columnwidth}{@{\extracolsep{\fill}}ccc@{}}
\hline        
Gait & $V(L)$ & Robot States \\           \hline \hline  
$^2L_{t1}$ &  $[6,4,5,2]$ & \raisebox{-.3\totalheight}{\includegraphics[width=.5cm]{figures/MTA3_State_6}} \kern-.3em $\pmb{\rightarrow}$  \kern-.3em \raisebox{-.3\totalheight}{\includegraphics[width=.5cm]{figures/MTA3_State_4}} \kern-.3em $\pmb{\rightarrow}$  \kern-.3em \raisebox{-.3\totalheight}{\includegraphics[width=.5cm]{figures/MTA3_State_5}} \kern-.3em $\pmb{\rightarrow}$  \kern-.3em \raisebox{-.3\totalheight}{\includegraphics[width=.5cm]{figures/MTA3_State_2}}\\            
        
$^2L_{t2}$ &  $[4,7,5,1]$ & \raisebox{-.3\totalheight}{\includegraphics[width=.5cm]{figures/MTA3_State_4}} \kern-.3em $\pmb{\rightarrow}$  \kern-.3em \raisebox{-.3\totalheight}{\includegraphics[width=.5cm]{figures/MTA3_State_7}}$\pmb{\rightarrow}$  \kern-.3em \raisebox{-.3\totalheight}{\includegraphics[width=.5cm]{figures/MTA3_State_5}} \kern-.3em $\pmb{\rightarrow}$  \kern-.3em \raisebox{-.3\totalheight}{\includegraphics[width=.5cm]{figures/MTA3_State_1}} \\            
            
$^2L_{t3}$ &  $[5,3]$ &\raisebox{-.3\totalheight}{\includegraphics[width=.5cm]{figures/MTA3_State_5}} \kern-.3em $\pmb{\rightarrow}$  \kern-.3em \raisebox{-.3\totalheight}{\includegraphics[width=.5cm]{figures/MTA3_State_3}}\\            
\hline    \hline
$^2L_{\theta 1}$ &  $[5,1]$ & \raisebox{-.3\totalheight}{\includegraphics[width=.5cm]{figures/MTA3_State_5}} \kern-.3em $\pmb{\rightarrow}$  \kern-.3em \raisebox{-.3\totalheight}{\includegraphics[width=.5cm]{figures/MTA3_State_1}}\\            
        
$^2L_{\theta 2}$ &  $[1,4,2,6]$ & \raisebox{-.3\totalheight}{\includegraphics[width=.5cm]{figures/MTA3_State_1}} \kern-.3em $\pmb{\rightarrow}$  \kern-.3em \raisebox{-.3\totalheight}{\includegraphics[width=.5cm]{figures/MTA3_State_4}}$\pmb{\rightarrow}$  \kern-.3em \raisebox{-.3\totalheight}{\includegraphics[width=.5cm]{figures/MTA3_State_2}} \kern-.3em $\pmb{\rightarrow}$  \kern-.3em \raisebox{-.3\totalheight}{\includegraphics[width=.5cm]{figures/MTA3_State_6}} \\            
            
$^2L_{\theta 3}$ &  $[5,7,6,2]$ &\raisebox{-.3\totalheight}{\includegraphics[width=.5cm]{figures/MTA3_State_5}} \kern-.3em $\pmb{\rightarrow}$  \kern-.3em \raisebox{-.3\totalheight}{\includegraphics[width=.5cm]{figures/MTA3_State_7}}$\pmb{\rightarrow}$  \kern-.3em \raisebox{-.3\totalheight}{\includegraphics[width=.5cm]{figures/MTA3_State_6}} \kern-.3em $\pmb{\rightarrow}$  \kern-.3em \raisebox{-.3\totalheight}{\includegraphics[width=.5cm]{figures/MTA3_State_2}}\\            
\hline    
\end{tabular*}
\end{center}
\end{table}

\begin{figure}[ht]
\centering
    \includegraphics[width=.7\columnwidth,trim=4cm 17.8cm 5cm 9cm, clip=true]{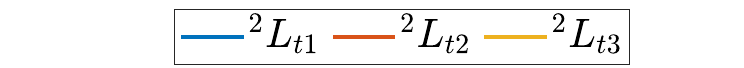} \\[-1.9ex]
    \subfloat[][]{\includegraphics[width = .46\columnwidth,trim=1.3cm 0.2cm 1.5cm 0.8cm, clip=true]{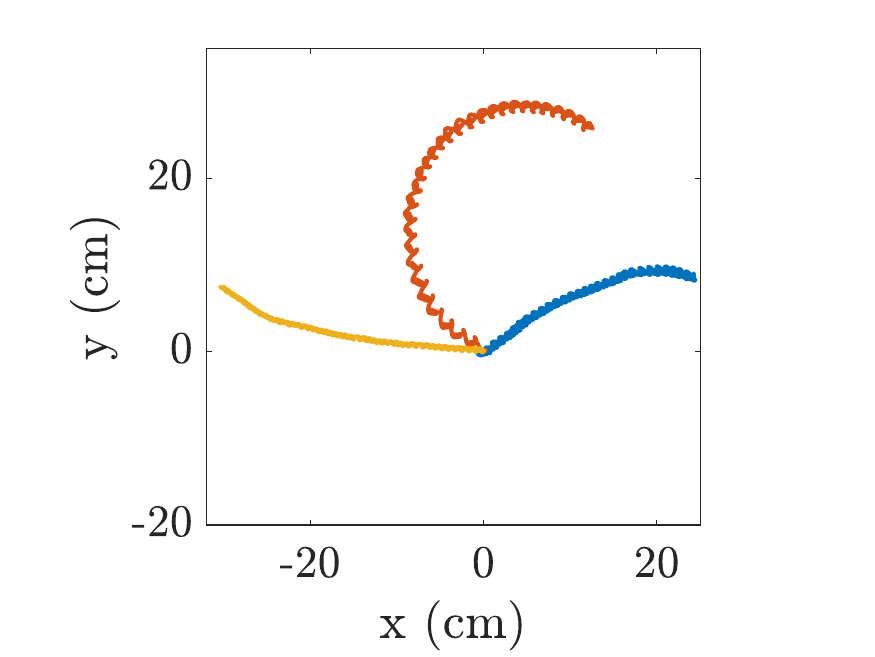}}\hfill
    \subfloat[][]{\includegraphics[width = .515\columnwidth,trim= 0cm 0.2cm 1.4cm 0.8cm, clip=true]{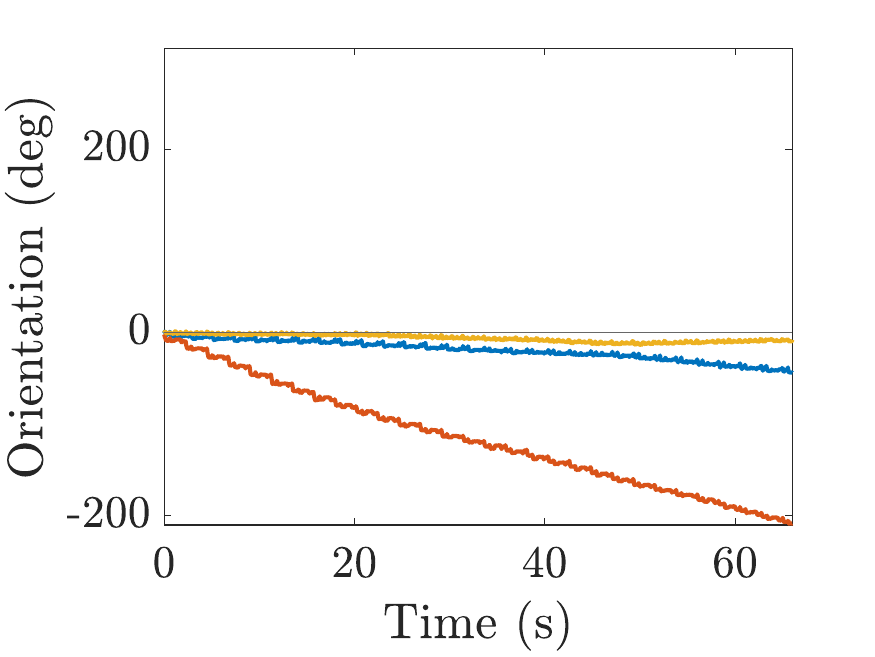}}\\
    
    \includegraphics[width=.5\columnwidth,trim=6cm 17.8cm 6.8cm 9cm, clip=true]{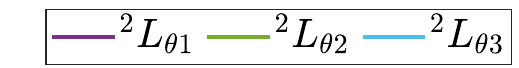} \\[-1.9ex]
    \subfloat[][]{\includegraphics[width = .46\columnwidth,trim=1.3cm 0.2cm 1.5cm 0.8cm, clip=true]{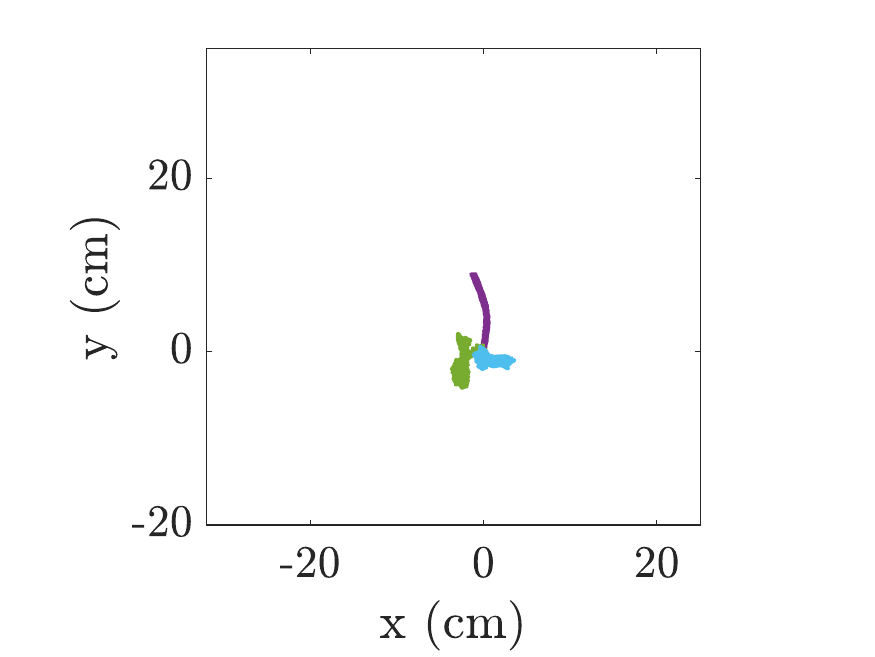}}\hfill
    \subfloat[][]{\includegraphics[width = .515\columnwidth,trim= 0cm 0.2cm 1.4cm 0.8cm, clip=true]{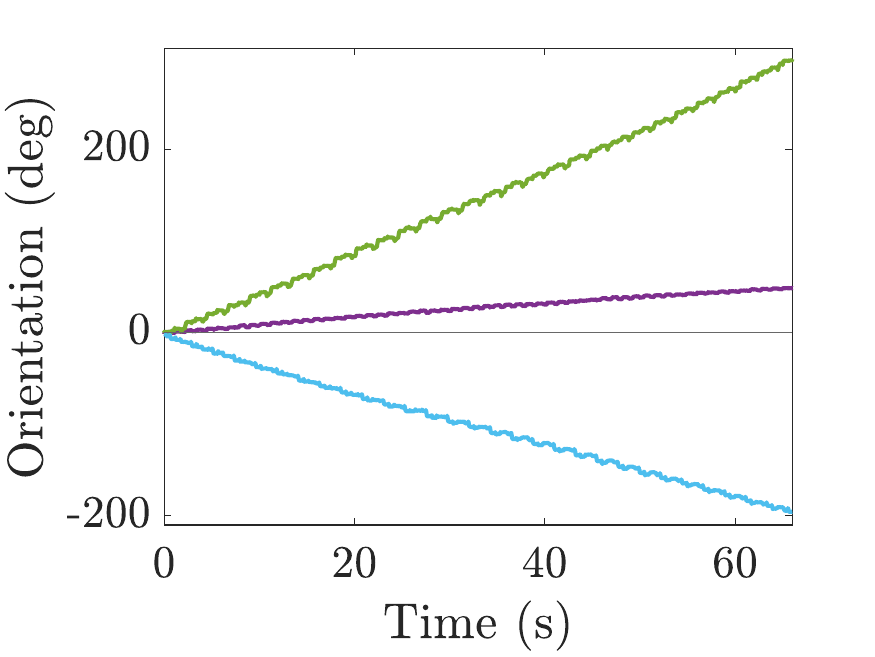}}\par

    \caption{\TriSoRo~gaits synthesized for the whiteboard (substrate 2). Experimental (a) trajectory and (b) rotation plots for the translation-dominant gaits ($^2L_{t1}, ^2L_{t2}, ^2L_{t3}$) are shown in addition to the (c) trajectory and (d) orientation plots for the rotation-dominant gaits ($^2L_{\theta 1}, ^2L_{\theta 2}, ^2L_{\theta 3}$).}
    \label{Fig:sub2}
\end{figure}

\begin{table}
\renewcommand\arraystretch{2}
\begin{center}
\caption{\TriSoRo~Synthesized Gaits on Substrate 3 (Carpet).}
\label{Tab:sub3}
\noindent\begin{tabular*}{\columnwidth}{@{\extracolsep{\fill}}ccc@{}}
\hline        
Gait & $V(L)$ & Robot States \\           \hline \hline  
$^3L_{t1}$ &  $[7,3]$ & \raisebox{-.3\totalheight}{\includegraphics[width=.5cm]{figures/MTA3_State_7}} \kern-.3em $\pmb{\rightarrow}$  \kern-.3em \raisebox{-.3\totalheight}{\includegraphics[width=.5cm]{figures/MTA3_State_3}}\\            
        
$^3L_{t2}$ &  $[7,4,3]$ & \raisebox{-.3\totalheight}{\includegraphics[width=.5cm]{figures/MTA3_State_7}} \kern-.3em $\pmb{\rightarrow}$  \kern-.3em \raisebox{-.3\totalheight}{\includegraphics[width=.5cm]{figures/MTA3_State_4}}$\pmb{\rightarrow}$  \kern-.3em \raisebox{-.3\totalheight}{\includegraphics[width=.5cm]{figures/MTA3_State_3}} \\            
            
$^3L_{t3}$ &  $[4,5,3,2]$ & \raisebox{-.3\totalheight}{\includegraphics[width=.5cm]{figures/MTA3_State_4}} \kern-.3em $\pmb{\rightarrow}$  \kern-.3em \raisebox{-.3\totalheight}{\includegraphics[width=.5cm]{figures/MTA3_State_5}}$\pmb{\rightarrow}$  \kern-.3em \raisebox{-.3\totalheight}{\includegraphics[width=.5cm]{figures/MTA3_State_3}} \kern-.3em $\pmb{\rightarrow}$  \kern-.3em \raisebox{-.3\totalheight}{\includegraphics[width=.5cm]{figures/MTA3_State_2}}\\            
\hline    \hline
$^3L_{\theta 1}$ &  $[2, 6 ,5, 3 ,8 ,4]$ & \raisebox{-.3\totalheight}{\includegraphics[width=.5cm]{figures/MTA3_State_2}} \kern-.3em $\pmb{\rightarrow}$  \kern-.3em \raisebox{-.3\totalheight}{\includegraphics[width=.5cm]{figures/MTA3_State_6}}$\pmb{\rightarrow}$  \kern-.3em \raisebox{-.3\totalheight}{\includegraphics[width=.5cm]{figures/MTA3_State_5}} \kern-.3em $\pmb{\rightarrow}$  \kern-.3em \raisebox{-.3\totalheight}{\includegraphics[width=.5cm]{figures/MTA3_State_3}} \kern-.3em $\pmb{\rightarrow}$  \kern-.3em \raisebox{-.3\totalheight}{\includegraphics[width=.5cm]{figures/MTA3_State_8}}$\pmb{\rightarrow}$  \kern-.3em \raisebox{-.3\totalheight}{\includegraphics[width=.5cm]{figures/MTA3_State_4}}\\            
        
$^3L_{\theta 2}$ &  $[3, 8, 7, 6, 4, 5, 2]$ &\raisebox{-.3\totalheight}{\includegraphics[width=.45cm]{figures/MTA3_State_3}} \kern-.3em $\pmb{\rightarrow}$  \kern-.3em \raisebox{-.3\totalheight}{\includegraphics[width=.45cm]{figures/MTA3_State_8}}$\pmb{\rightarrow}$  \kern-.3em \raisebox{-.3\totalheight}{\includegraphics[width=.45cm]{figures/MTA3_State_7}} \kern-.3em $\pmb{\rightarrow}$  \kern-.3em \raisebox{-.3\totalheight}{\includegraphics[width=.5cm]{figures/MTA3_State_6}} \kern-.3em $\pmb{\rightarrow}$  \kern-.3em \raisebox{-.3\totalheight}{\includegraphics[width=.45cm]{figures/MTA3_State_4}}$\pmb{\rightarrow}$  \kern-.3em \raisebox{-.3\totalheight}{\includegraphics[width=.45cm]{figures/MTA3_State_5}} \kern-.3em $\pmb{\rightarrow}$  \kern-.3em \raisebox{-.3\totalheight}{\includegraphics[width=.45cm]{figures/MTA3_State_2}}\\            
            
$^3L_{\theta 3}$ &  $[7, 4, 2, 6 ,5, 3, 1]$ &\raisebox{-.3\totalheight}{\includegraphics[width=.45cm]{figures/MTA3_State_7}} \kern-.3em $\pmb{\rightarrow}$  \kern-.3em \raisebox{-.3\totalheight}{\includegraphics[width=.5cm]{figures/MTA3_State_4}}$\pmb{\rightarrow}$  \kern-.3em \raisebox{-.3\totalheight}{\includegraphics[width=.45cm]{figures/MTA3_State_2}} \kern-.3em $\pmb{\rightarrow}$  \kern-.3em \raisebox{-.3\totalheight}{\includegraphics[width=.45cm]{figures/MTA3_State_6}} \kern-.3em $\pmb{\rightarrow}$  \kern-.3em \raisebox{-.3\totalheight}{\includegraphics[width=.45cm]{figures/MTA3_State_5}}$\pmb{\rightarrow}$  \kern-.3em \raisebox{-.3\totalheight}{\includegraphics[width=.45cm]{figures/MTA3_State_3}} \kern-.3em $\pmb{\rightarrow}$  \kern-.3em \raisebox{-.3\totalheight}{\includegraphics[width=.45cm]{figures/MTA3_State_1}}\\            
\hline    
\end{tabular*}
\end{center}
\end{table}

\begin{figure}[ht]
\centering
    \includegraphics[width=.55\columnwidth,trim=6cm 19.1cm 6cm 7.8cm, clip=true]{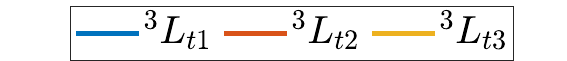} \\[-1.9ex]
    \subfloat[][]{\includegraphics[width = .44\columnwidth,trim=1.8cm 0.2cm 2cm 0.8cm, clip=true]{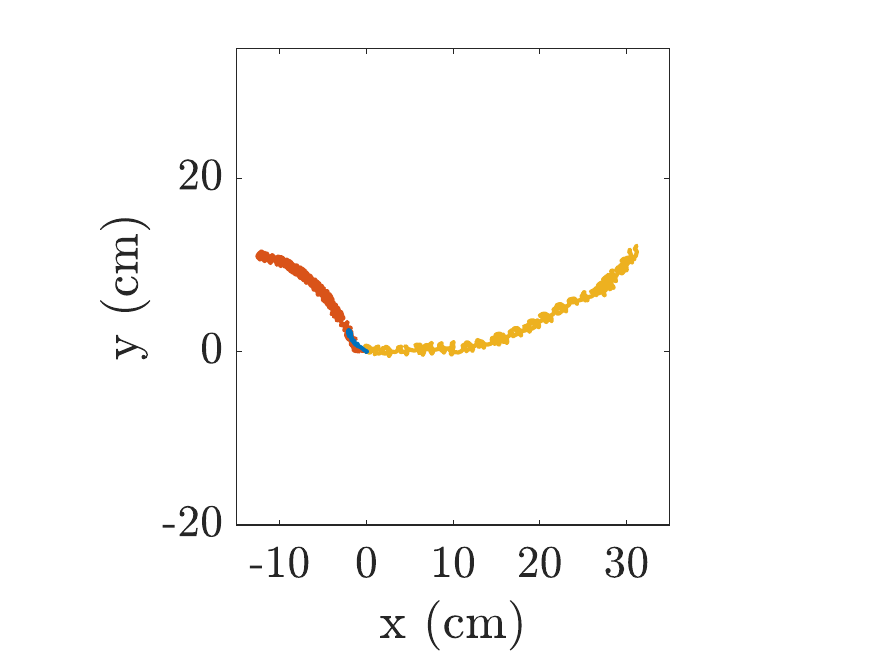}}\hfill
    \subfloat[][]{\includegraphics[width = .53\columnwidth,trim= 0cm 0.1cm 1.3cm .5cm, clip=true]{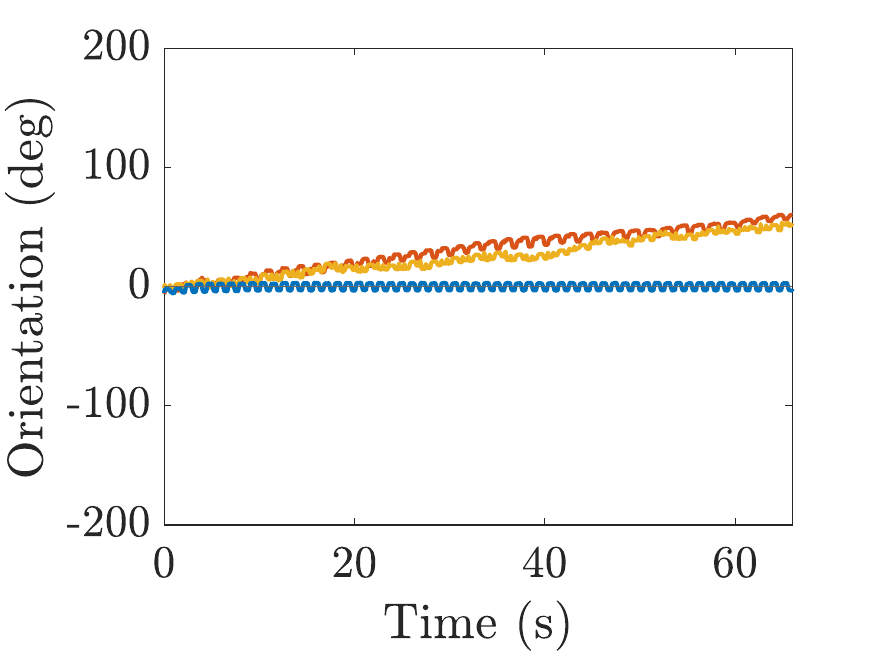}}\\
    
    \includegraphics[width=.55\columnwidth,trim=6cm 17.8cm 6cm 8.6cm, clip=true]{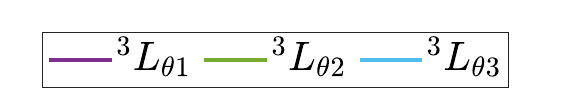} \\[-1.9ex]
    \subfloat[][]{\includegraphics[width = .44\columnwidth,trim=1.8cm 0.2cm 2cm 0.8cm, clip=true]{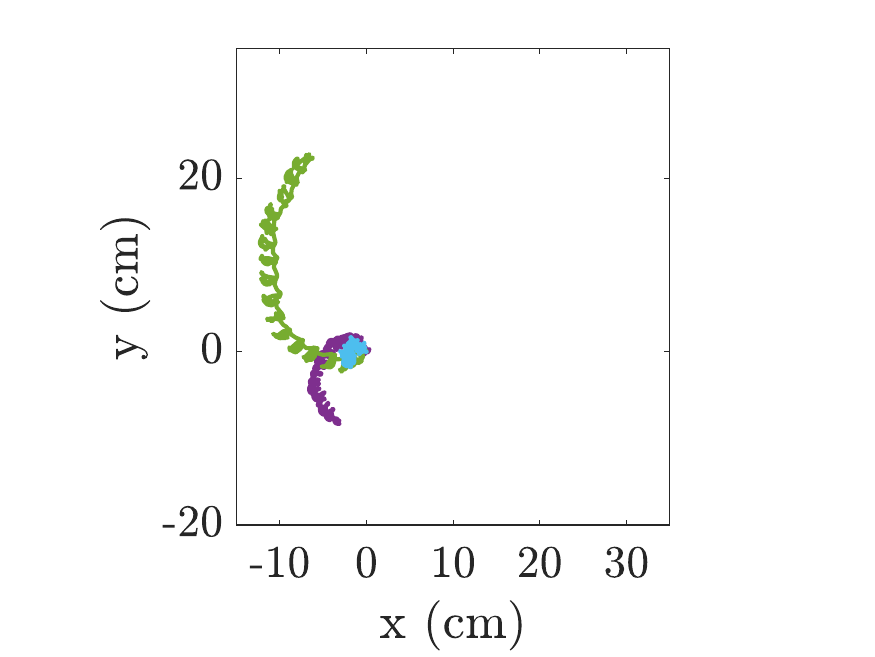}}\hfill
    \subfloat[][]{\includegraphics[width = .53\columnwidth,trim= 0cm 0.1cm 1.3cm .5cm, clip=true]{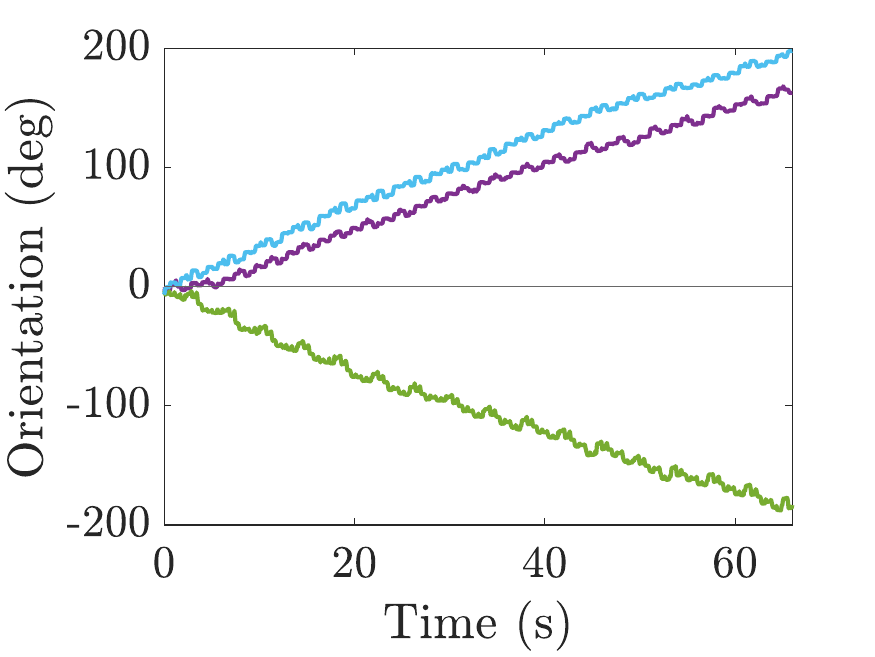}}\par
    \caption{\TriSoRo~gaits synthesized for the carpet (substrate 3). Experimental (a) trajectory and (b) orientation plots for the translation-dominant gaits ($^3L_{t1}, ^3L_{t2}, ^3L_{t3}$) are shown in addition to the (c) trajectory and (d) orientation plots for the rotation-dominant gaits ($^3L_{\theta 1}, ^3L_{\theta 2}, ^3L_{\theta 3}$).}
    \label{Fig:sub3}
\end{figure}
%In simulation, the digraph weights are used to predict the expected translation $\tilde{\bm{P}}$ and rotation $\tilde{\Theta}$ of each gait cycle. These values are compared to the experimentally obtained average translation $\bm{P}$ and rotation $\Theta$ of each gait cycle. Rotations are given as positive in the counter-clockwise direction. 
% 
% 
\subsubsection{Comparison Between BILP and Exhaustive Gait Search Results} \label{subsubsec:optimality}
The gait synthesis methodology presented here attempts to overcome the computational challenges of nonlinearity and combinatorial explosion by introducing linearizations, relaxations, and simplifications that reduce the complexity and increase the tractability of the problem. To examine the degree to which these alterations affect the optimality of the solutions, we can exhaustively calculate the nonlinear cost functions $J_{t,nl}$ and $J_{\theta,nl}$, defined in \eqref{Eqn:nlcostfuncs}, for all possible simple cycles. This is only tractable for robots of three or fewer limbs. The hyperparameter $\lambda$ is set to 1 to equally scale the motion and variance. \Fig~\ref{Fig:pareto} plots one nonlinear cost function versus the inverse of the other. The Pareto front can then be seen along the upper right-hand side of the plots as the points at which one cost function is maximized and the other is minimized. Shading is also incorporated to indicate gait length, as shorter gaits tend to be preferable due to their simplicity. While it is assumed that some level of optimality will be compromised, the data indicate that the majority of the synthesized gaits do fall on the Pareto front, while the remaining ones fall slightly within. These plots are only shown for one dataset (the rubber mat substrate) but suggest that this gait synthesis methodology maintains a reasonable degree of optimality and experimentally produces useful gaits. Again, numerical optimality for these cost functions does not necessarily equate to experimental optimality as there is expected to be some level of variation between the predicted (i.e., expected) and experimental gait behaviors.
\begin{figure}[ht]
\centering
    \includegraphics[width=\columnwidth,trim=4.5cm 17.5cm 4.5cm 8.8cm, clip=true]{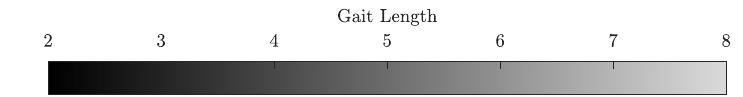} \\[-.6ex]
    \subfloat[][]{\includegraphics[width = .49\columnwidth,trim= 0cm 0.4cm 2cm .2cm, clip=true]{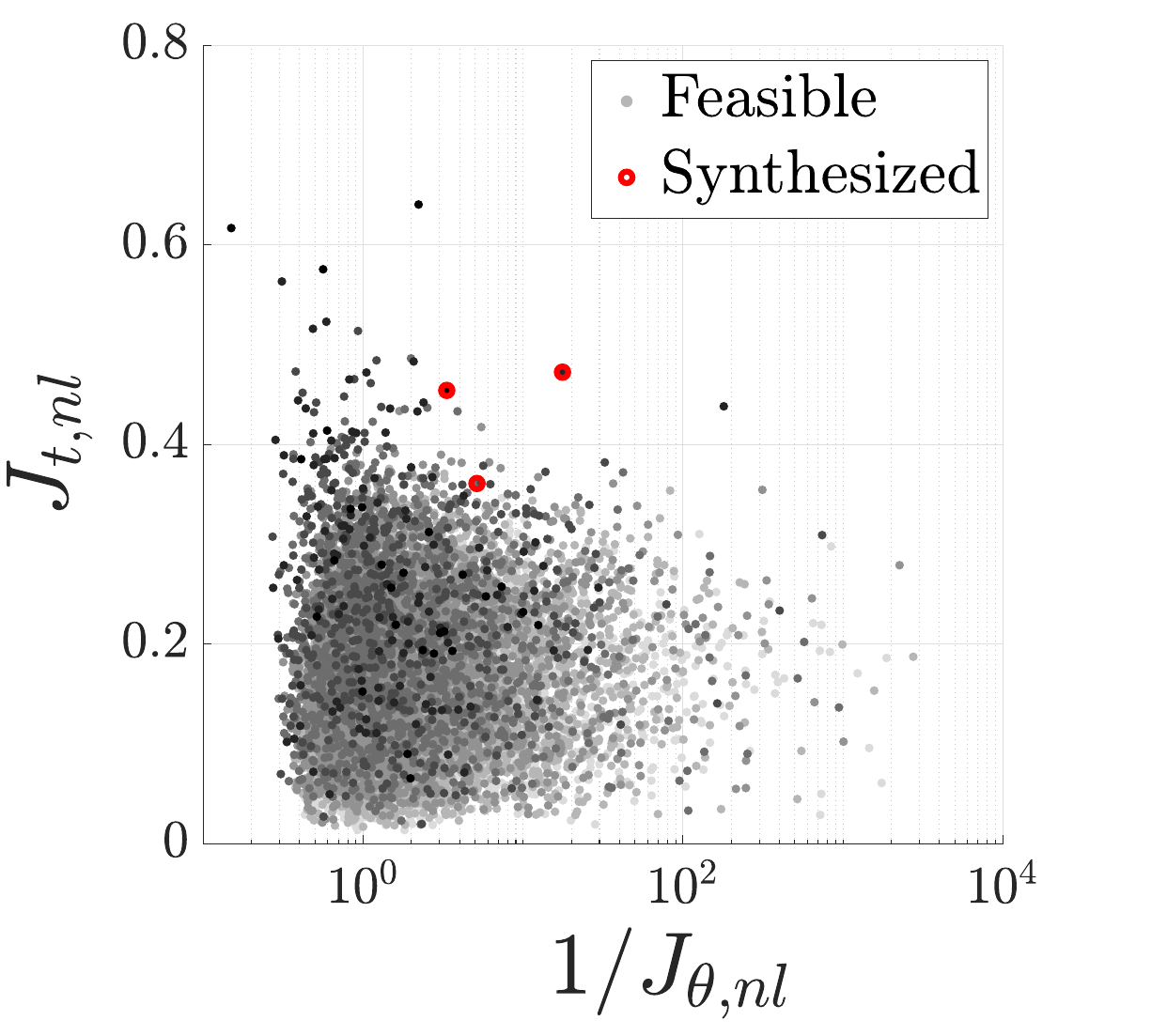}}\hfill
    \subfloat[][]{\includegraphics[width = .47\columnwidth,trim= 0cm 0.4cm 2.5cm .2cm, clip=true]{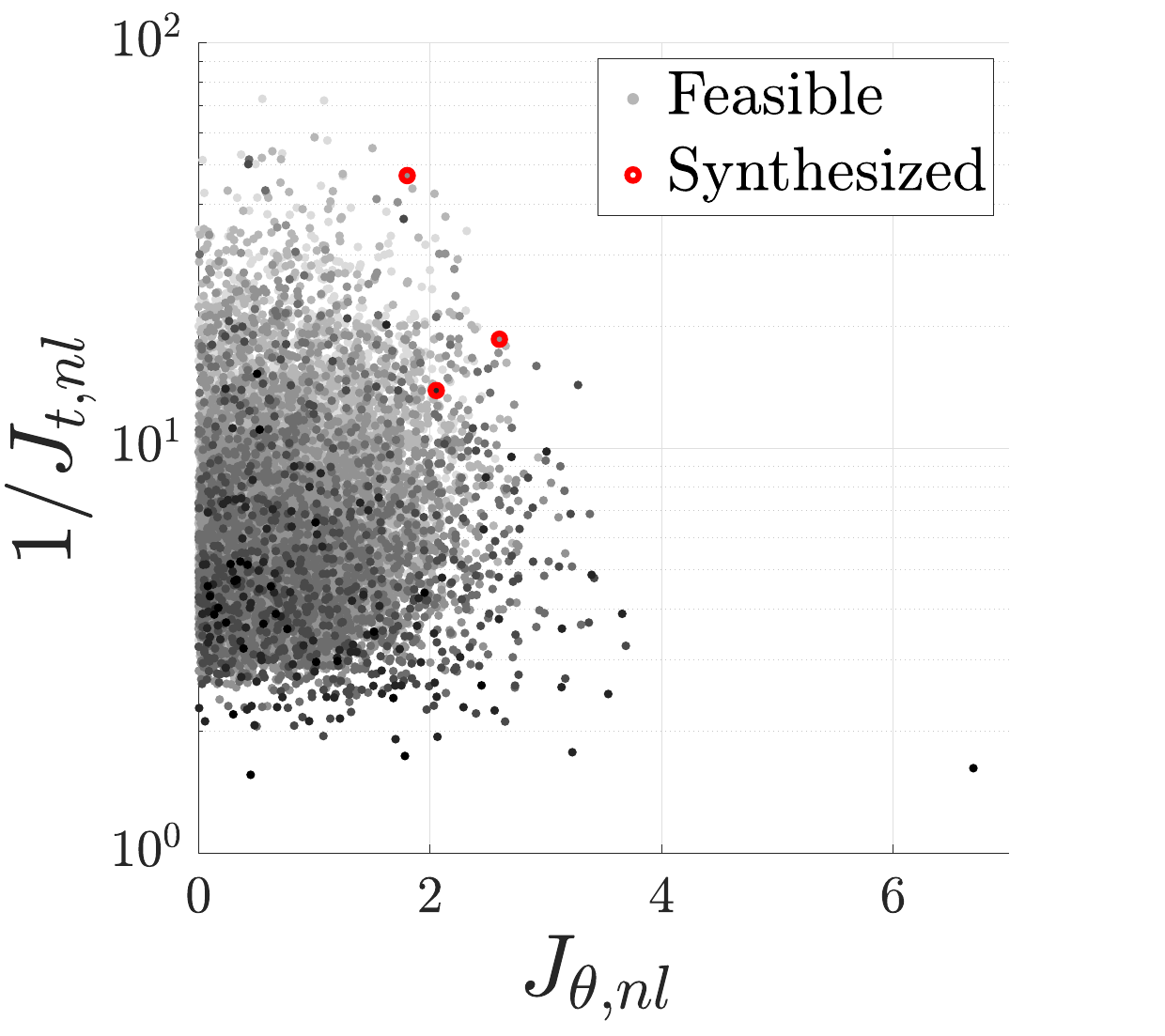}}\\

    \caption{Exhaustive nonlinear cost function evaluations for all possible \TriSoRo~gaits (simple cycles) on the rubber mat, calculated with data from Eulerian cycle experiments. Synthesized gaits (red) are plotted alongside feasible gaits of varying lengths (grey/black) to evaluate (a) translation dominance of gaits $^1L_{t1}, ^1L_{t2},$ and $ ^1L_{t3}$ and (b) rotation dominance of gaits $^1L_{\theta 1}, ^1L_{\theta 2},$ and $ ^1L_{\theta 3}$. Inverse cost functions $1/J_{t,nl}$ and $1/J_{\theta,nl}$ are shown on the logarithmic scale.}
    \label{Fig:pareto}
\end{figure}
%In simulation, the

\subsection{Robot 2: \TetraSoRo}
This section details the experimental results  of both synthesized and intuitive gaits for the four-limb robot on a single substrate (rubber mat). Gaits are then re-synthesized without re-learning the graph with the added constraint that one limb is non-functioning.
\subsubsection{Intuitive Gaits}
Intuitive gaits are first performed on the rubber mat substrate as detailed in \Tab\ref{Tab:Int4}. Despite being adapted from three-limb robot intuitive gaits, the resulting behavior is different. All three gaits contain a much stronger rotational component than those seen in the three-limb robot, suggesting that this robot may be more inclined towards counter-clockwise rotation. The pronk-like gait $L_1$, which would be expected to induce only rotation due to the rotational symmetry of the robot instead results in the largest translation with the smallest rotation of the three gaits; the opposite is true of the pronk-like gait for the three-limb robot. The rotary gallop-inspired gait $L_2$ produces relatively small translation with a rotation component, despite this pattern resulting in the largest translation for the three-limb robot. Finally, the push-pull gait, selected to produce translation, results in the largest rotation and least translation of all the gaits. This contrasts with the translation-dominant nature of the three-limb push-pull gait. These experimental results underscore the difficulty of intuitively picking soft robot gaits and predicting the resulting behavior, particularly when they lack robotic or biological analogues. 

\begin{table}[ht]
\begin{center}
\caption{Intuitive Gaits for \TetraSoRo~(Not Synthesized).}
\label{Tab:Int4}
\renewcommand\arraystretch{2}
\noindent\begin{tabular*}{\columnwidth}{@{\extracolsep{\fill}}cccc@{}}
\hline          
Gait & $V(L)$ & Robot States & Inspiration\\           \hline \hline           
$L_1$ &  $[16,1]$ &   \raisebox{-.3\totalheight}{\includegraphics[width=.6cm]{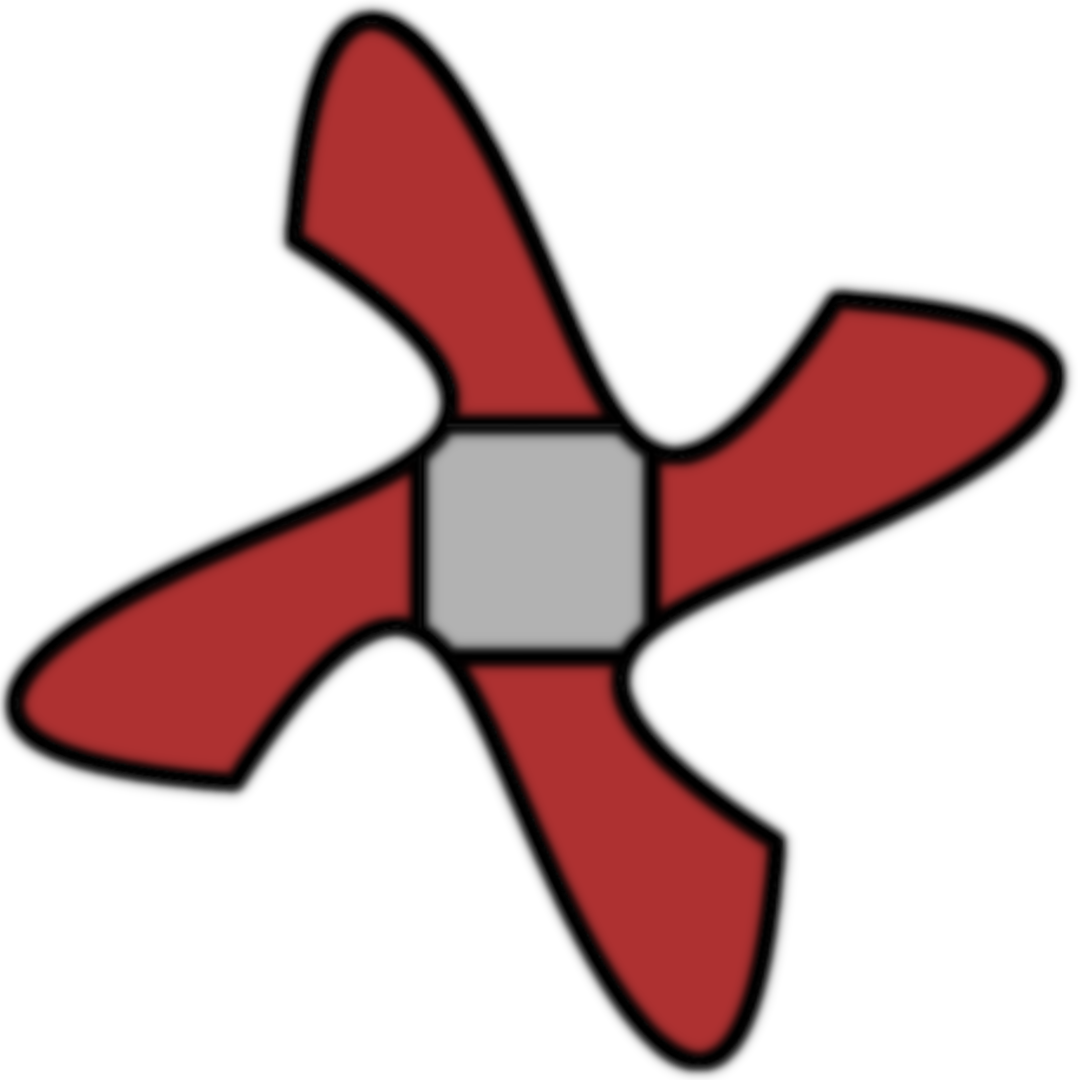}} \kern-.5em $\pmb{\rightarrow}$ \kern-.3em \raisebox{-.3\totalheight}{\includegraphics[width=.6cm]{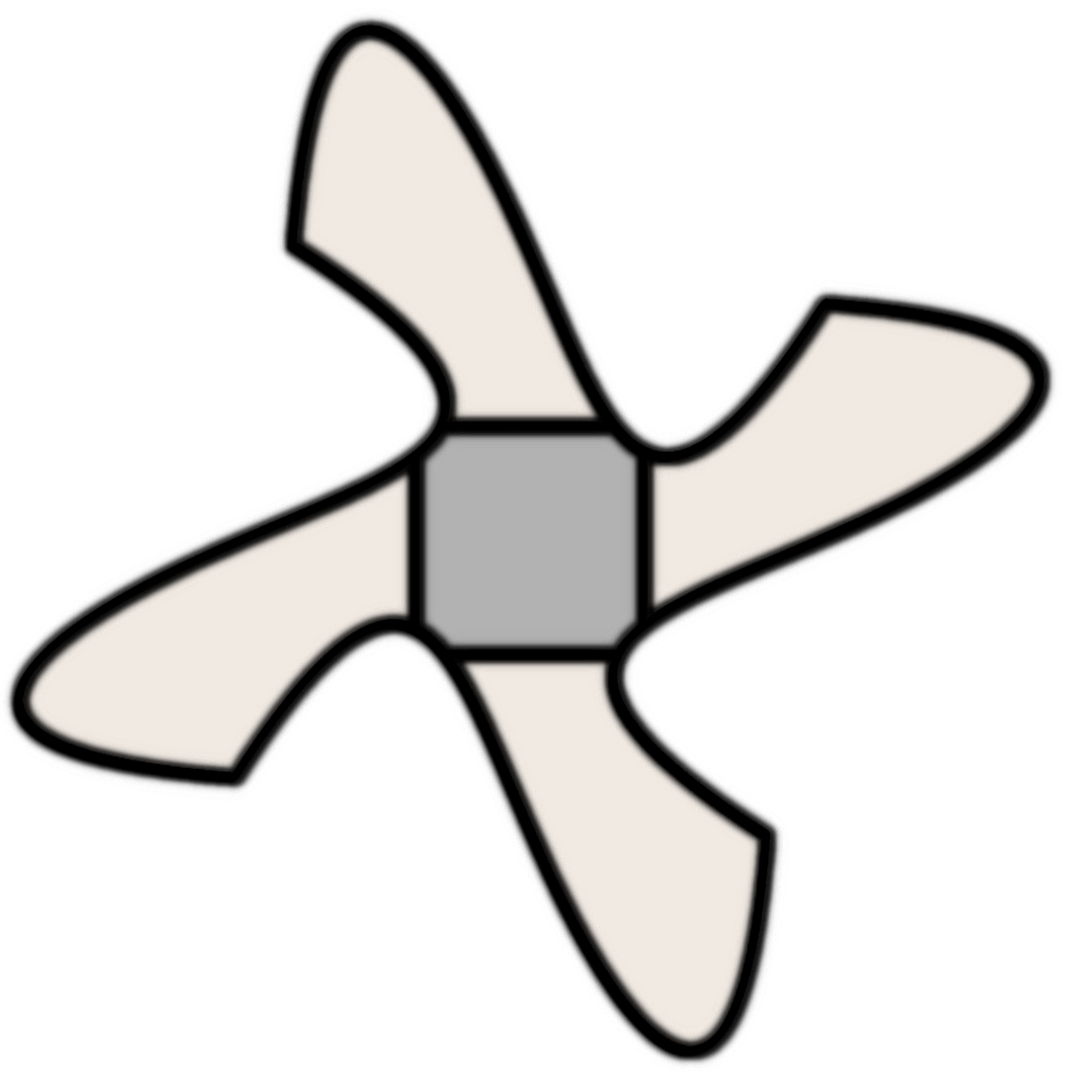}} & pronk\\            
          
$L_2$ &  $[9,5,3,2]$ &  \raisebox{-.3\totalheight}{\includegraphics[width=.6cm]{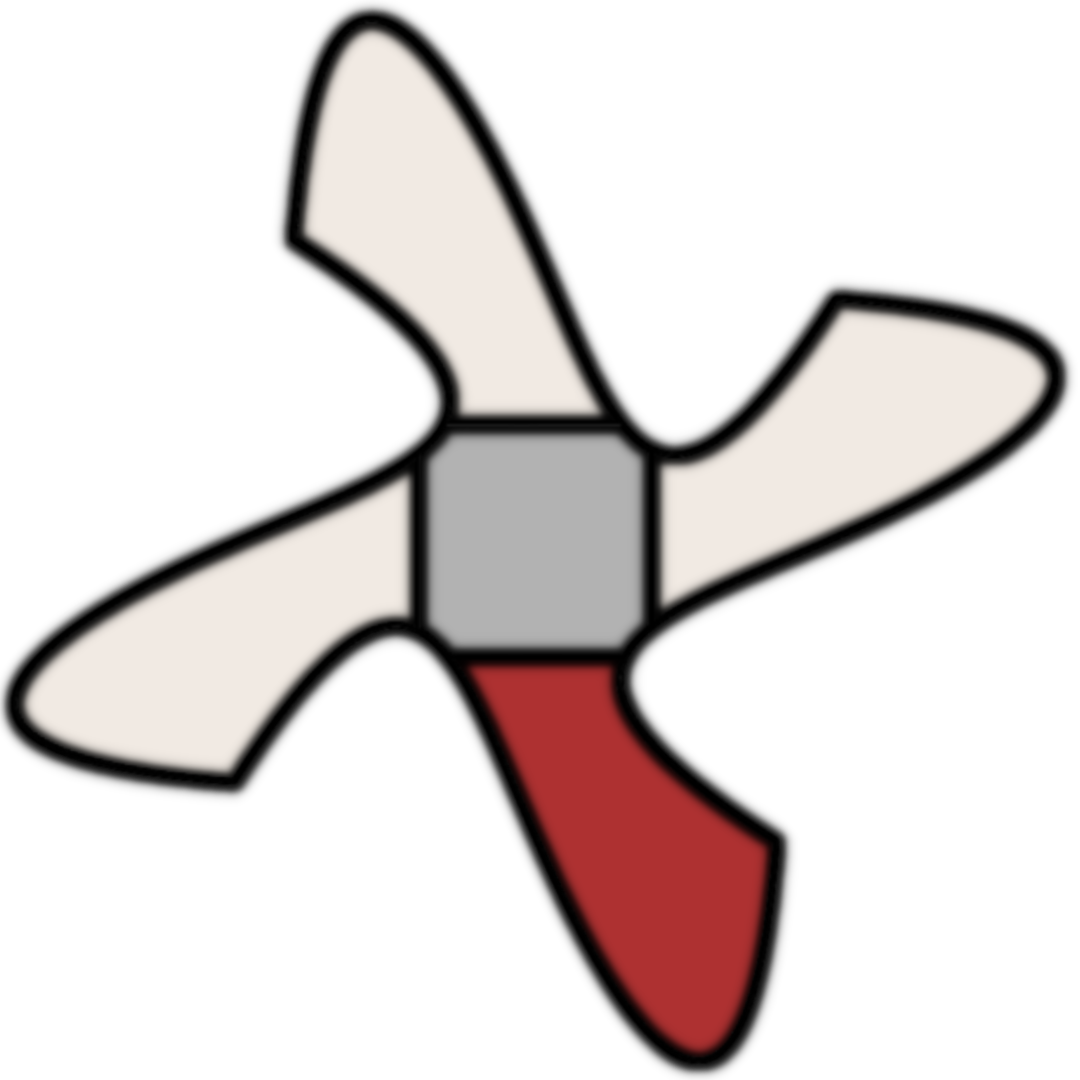}} \kern-.5em $\pmb{\rightarrow}$ \kern-.3em \raisebox{-.3\totalheight}{\includegraphics[width=.6cm]{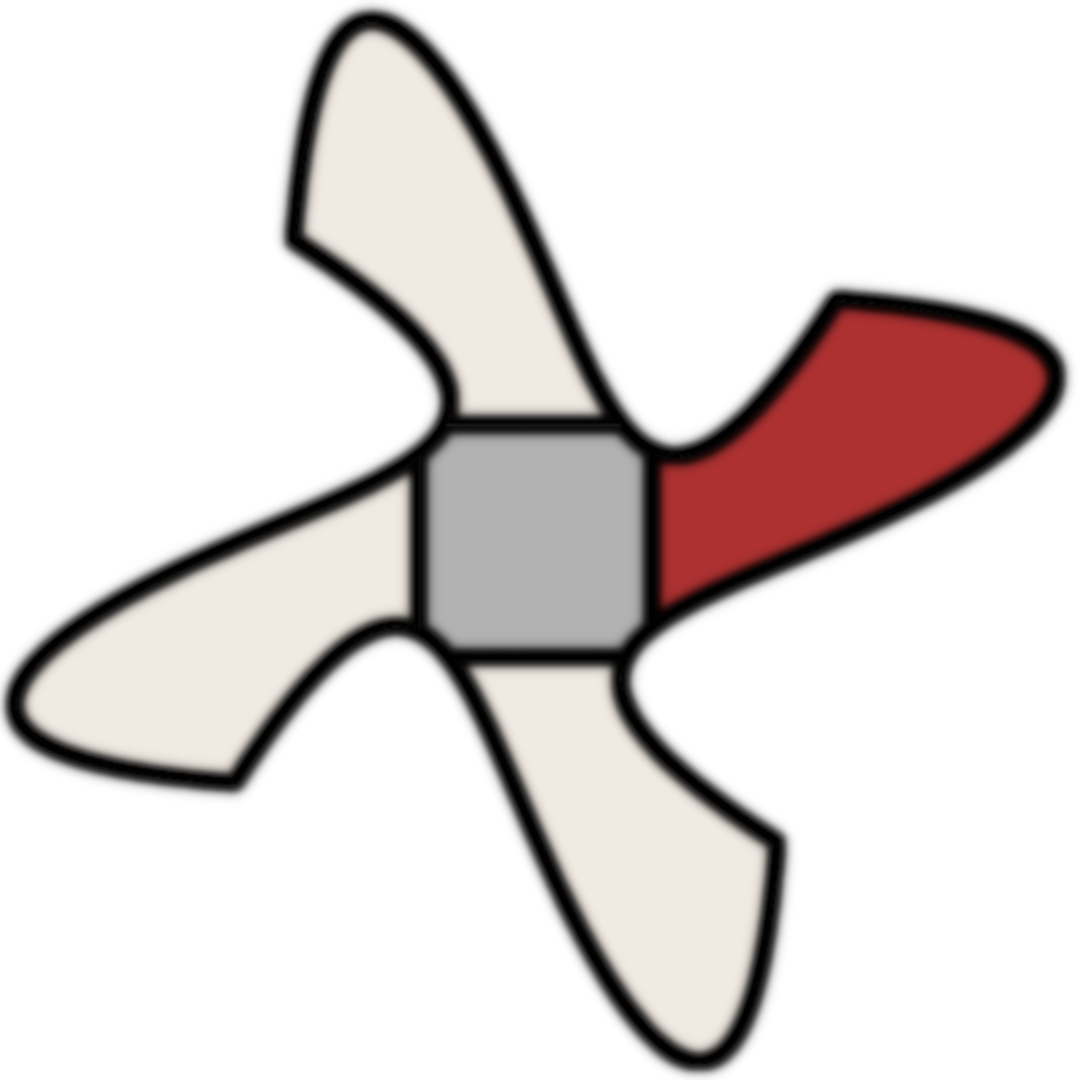}} \kern-.5em $\pmb{\rightarrow}$ \kern-.3em  \raisebox{-.3\totalheight}{\includegraphics[width=.6cm]{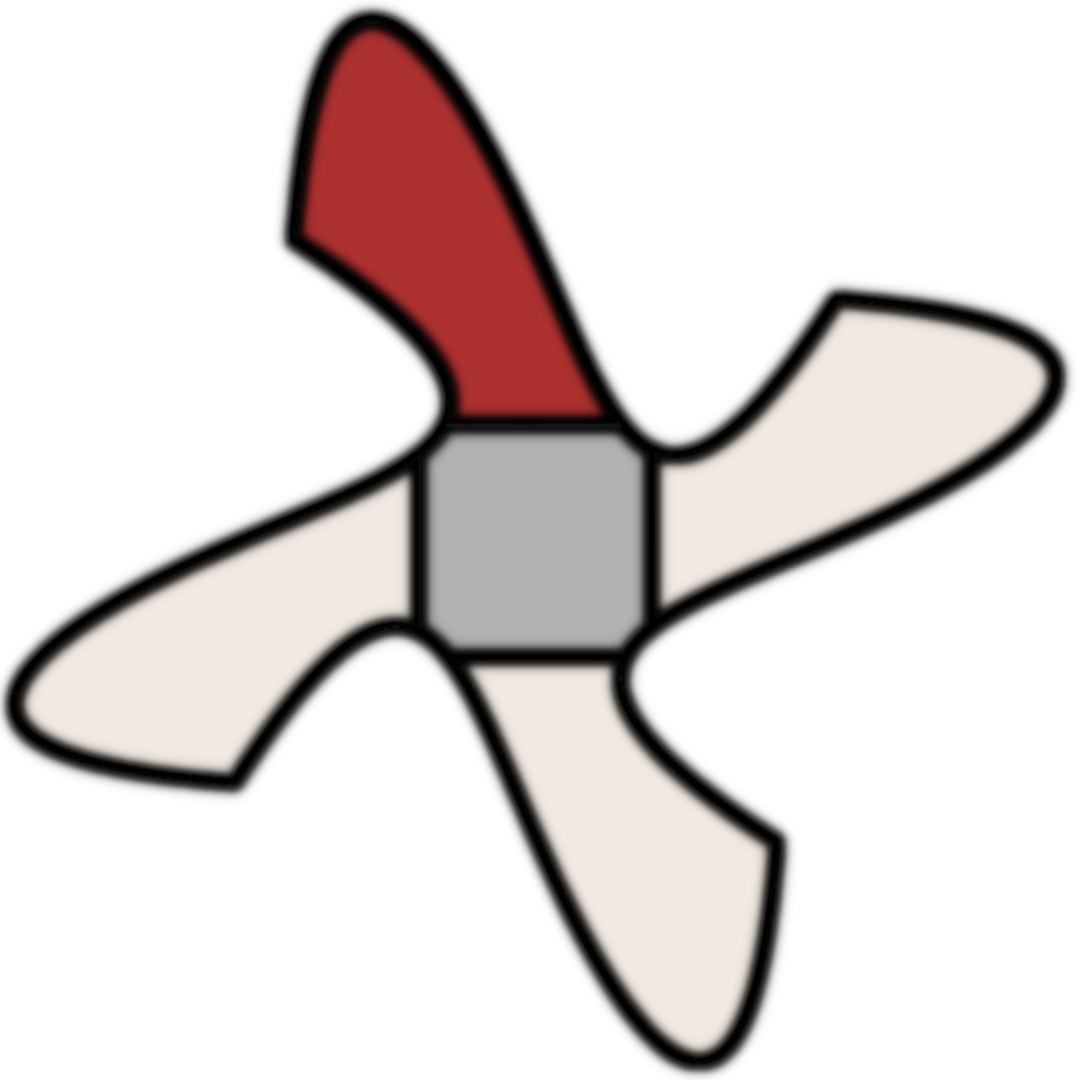}} \kern-.5em $\pmb{\rightarrow}$ \kern-.3em \raisebox{-.3\totalheight}{\includegraphics[width=.6cm]{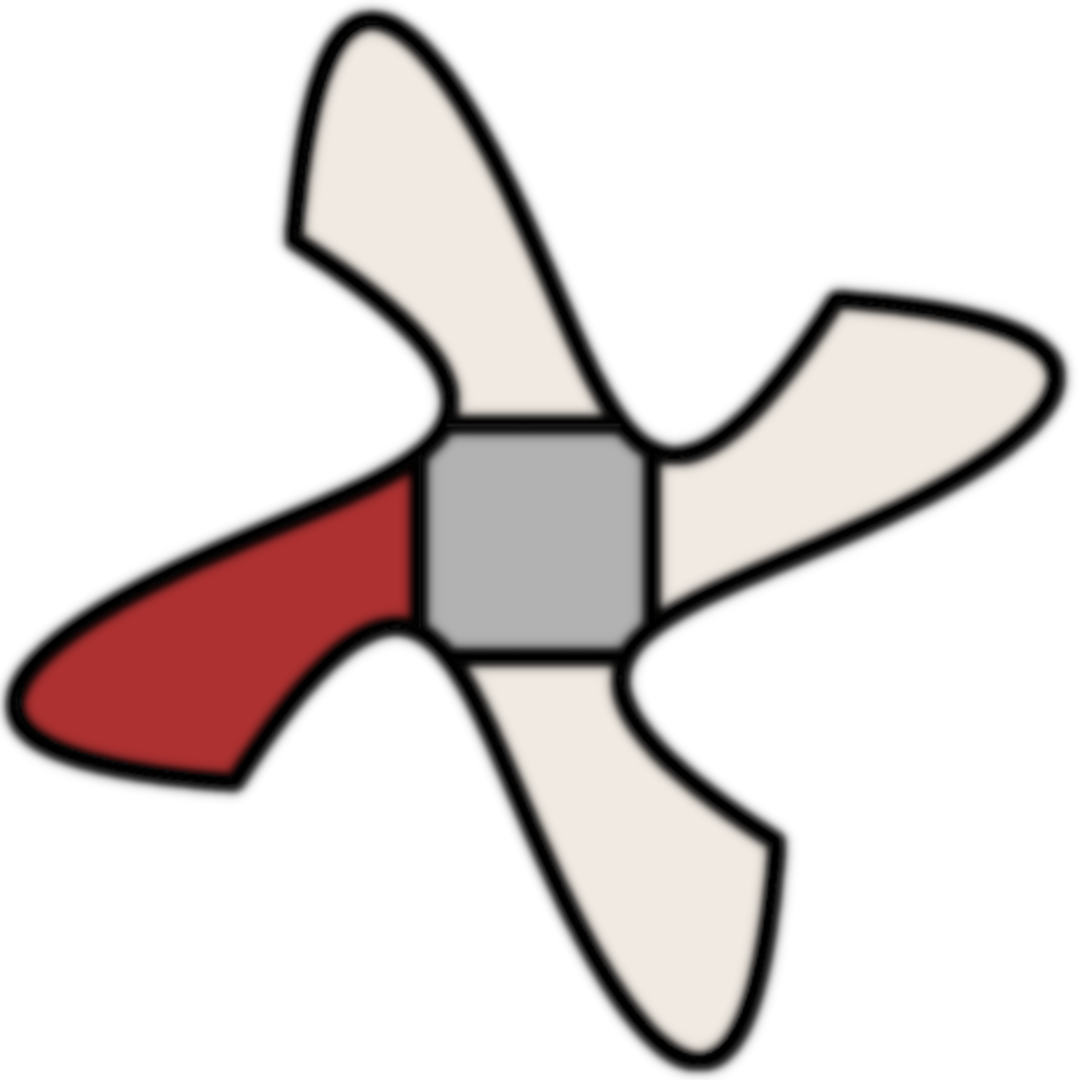}}& rotary gallop\\            
           
$L_3$ &  $[13,16,4,1]$ & \raisebox{-.3\totalheight}{\includegraphics[width=.6cm]{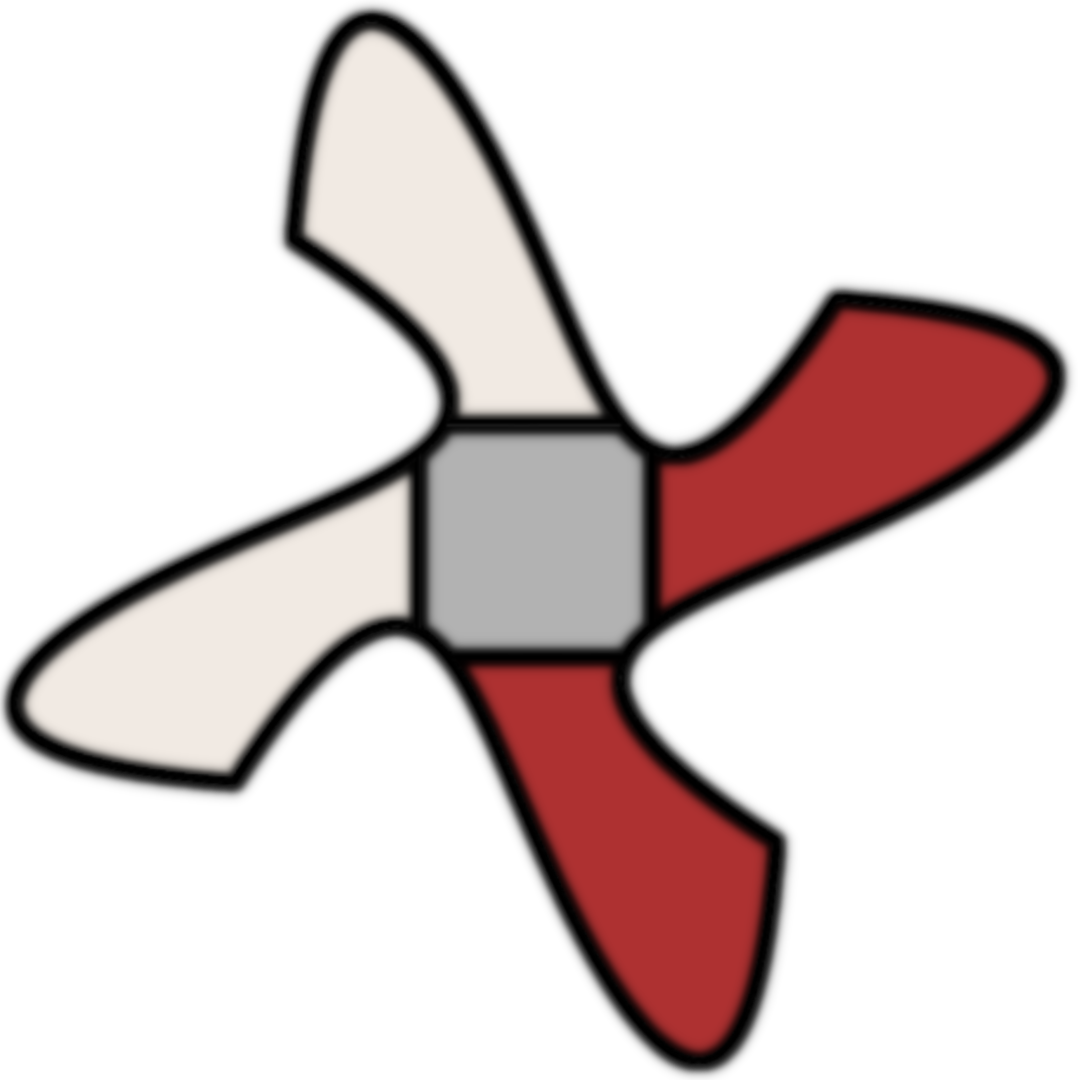}} \kern-.5em  $\pmb{\rightarrow}$    \kern-.3em \raisebox{-.3\totalheight}{\includegraphics[width=.6cm]{figures/MTA4_State_16.png}} \kern-.5em $\pmb{\rightarrow}$ \kern-.3em  \raisebox{-.3\totalheight}{\includegraphics[width=.6cm]{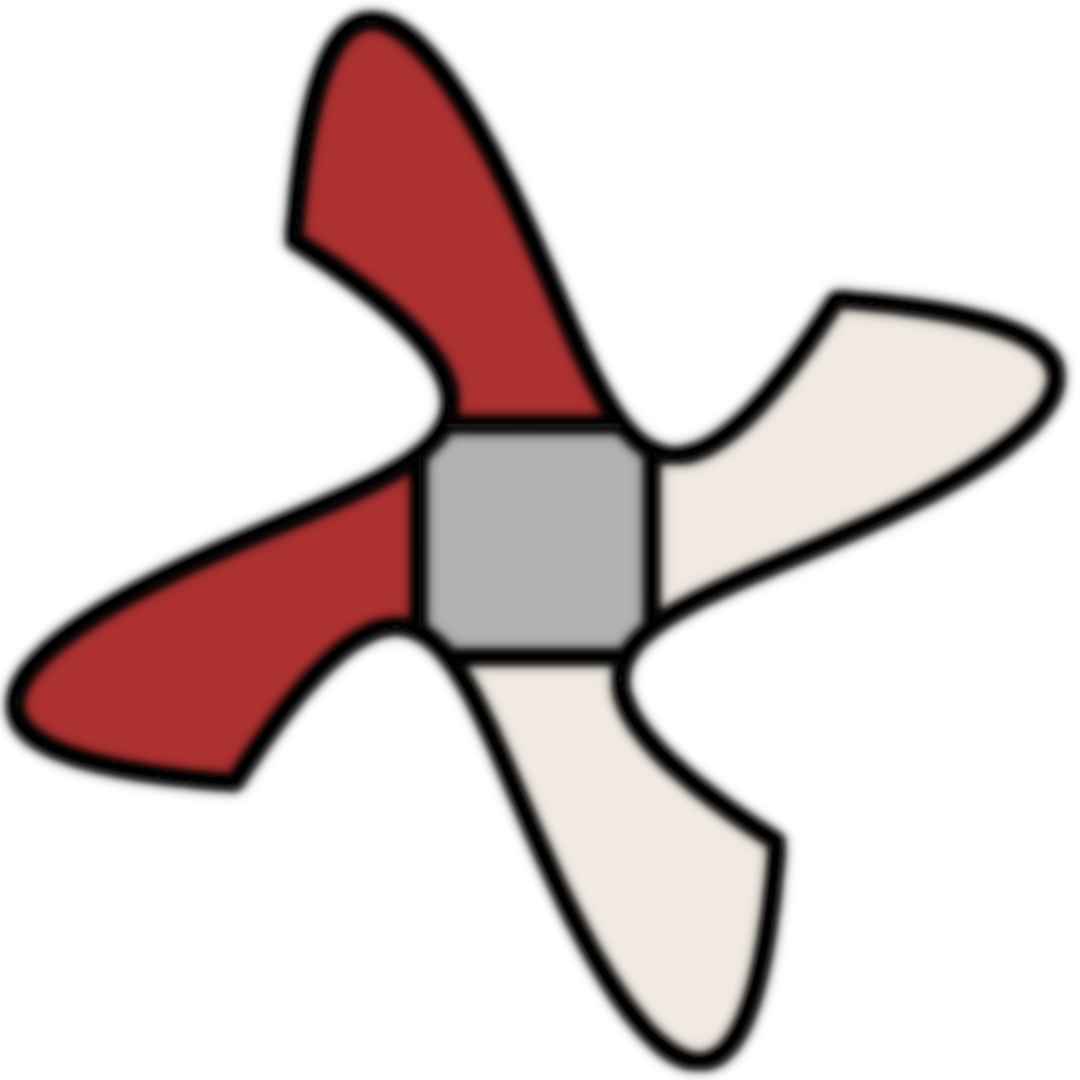}} \kern-.5em $\pmb{\rightarrow}$ \kern-.3em   \raisebox{-.3\totalheight}{\includegraphics[width=.6cm]{figures/MTA4_State_1.png}}& push-pull\\            
\hline    
\end{tabular*}
\end{center}
\end{table}

\begin{figure}[ht]
\centering
    \includegraphics[width=.7\columnwidth,trim=4cm 17.8cm 5cm 9.1cm, clip=true]{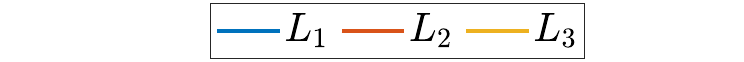} \\[-1.5ex] 
    \subfloat[][]{\includegraphics[width = .43\columnwidth,trim=2cm 0.2cm 2.2cm .8cm, clip=true]{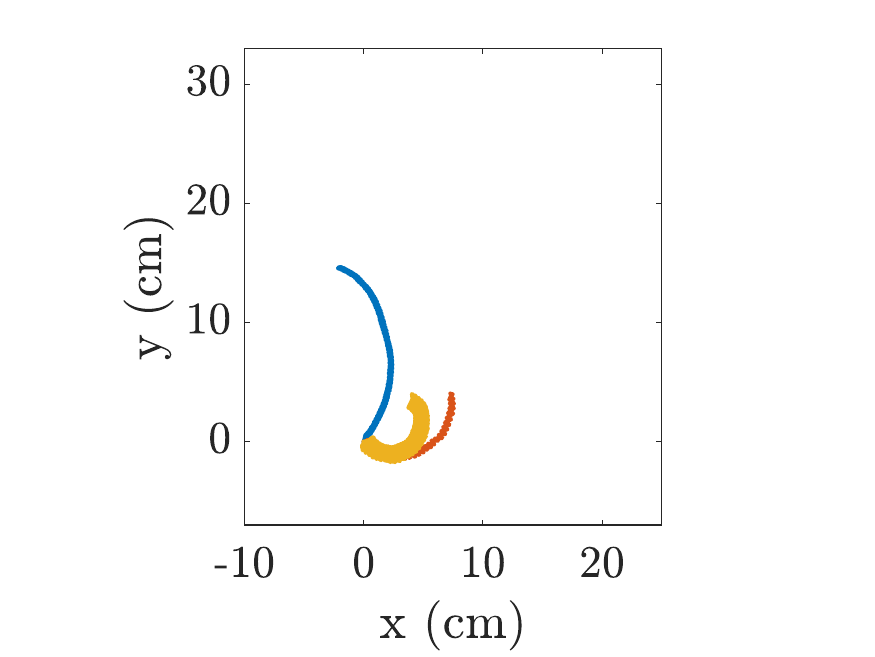}}\hfill
    \subfloat[][]{\includegraphics[width = .54\columnwidth,trim= 0cm 0.1cm 1.3cm .4cm, clip=true]{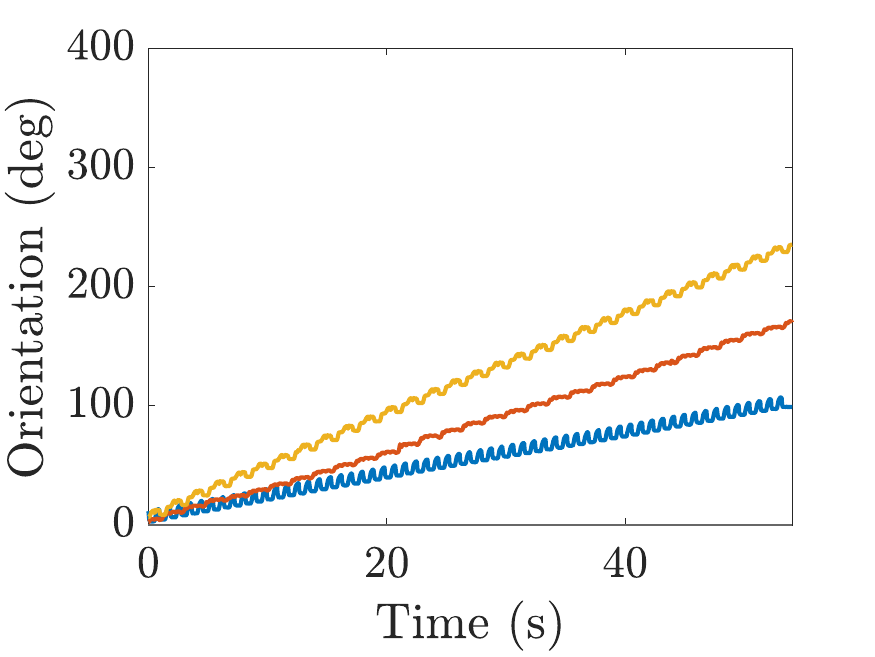}}\par

    \caption{\TetraSoRo~intuitive gaits $L_1, L_2,$ and $L_3$ tested on the rubber mat with experimental (a) trajectory and (b) orientation plots. }
    \label{Fig:4int}
\end{figure}

\subsubsection{Experimental Results on Substrate 1 (Rubber Mat)}
Gaits for the four-limb robot are synthesized for the rubber mat substrate and presented in \Tab\ref{Tab:4synth} and Fig.~\ref{Fig:4synth} (see multimedia video attachment for $^1L_{t1}$). Compared to the three-limb robot, the translation-dominant gaits have a larger rotational component and the rotation-dominant gaits achieve much larger rotation magnitudes (up to six-fold) with much smaller translation components. Additionally, no clockwise rotation-dominant gait is synthesized for the four-limb robot. These reasons suggest that this robot is more inclined towards counter-clockwise motion. Despite the strong coupling of translational and rotational motion inherent in the robot design, the synthesized gaits are successful and reasonably uncoupled. Another interesting observation is that all three translation-dominant gaits consist of the same limb being either fully curled or fully uncurled for the duration of the gait. This can essentially be interpreted as an alteration of the robot morphology to prevent rotation, similar to the way some animals drag their tails to stabilize their locomotion. The fact that this alteration occurred for every synthesized translation-dominant gait (and with the same limb) without any pre-defined knowledge of the robot highlights the versatility of this method to identify manufacturing / behavior inconsistencies and exploit them to enhance locomotion. 
\begin{table}[hb]
\renewcommand\arraystretch{2}
\begin{center}
\caption{\TetraSoRo~Synthesized Gaits on Substrate 1 (Rubber Mat).}
\label{Tab:4synth}
\noindent\begin{tabular*}{\columnwidth}{@{\extracolsep{\fill}}ccc@{}}
\hline        
Gait & $V(L)$ & Robot States \\           \hline \hline  
$^1L_{t1}$ &  $[3,15,1]$ & \raisebox{-.3\totalheight}{\includegraphics[width=.6cm]{figures/MTA4_State_3}} \kern-.5em $\pmb{\rightarrow}$ \kern-.3em \raisebox{-.3\totalheight}{\includegraphics[width=.6cm]{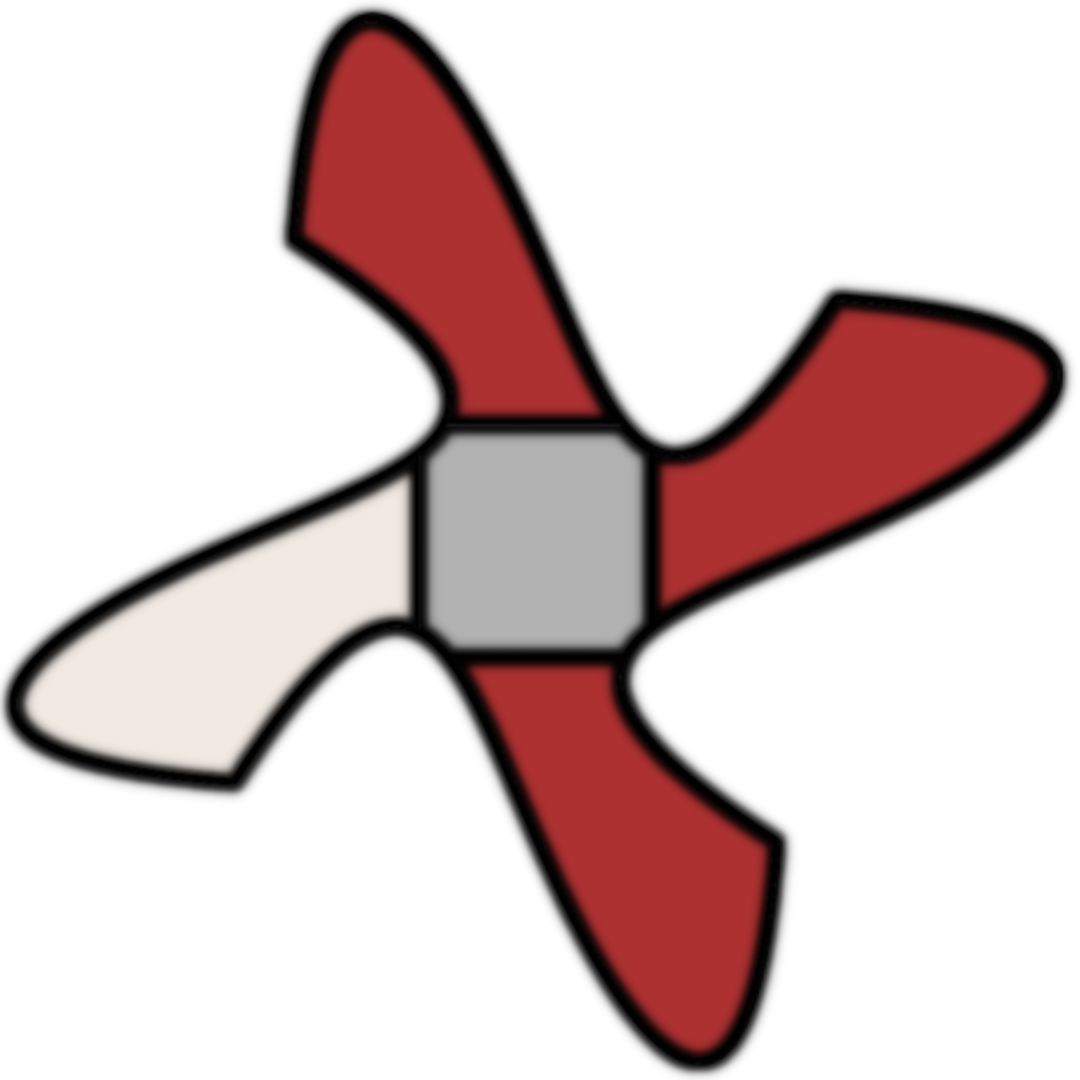}} \kern-.5em $\pmb{\rightarrow}$ \kern-.3em  \raisebox{-.3\totalheight}{\includegraphics[width=.6cm]{figures/MTA4_State_1.png}}\\            
        
$^1L_{t2}$ &  $[5,15,1]$ & \raisebox{-.3\totalheight}{\includegraphics[width=.6cm]{figures/MTA4_State_5}} \kern-.5em $\pmb{\rightarrow}$ \kern-.3em \raisebox{-.3\totalheight}{\includegraphics[width=.6cm]{figures/MTA4_State_15.png}} \kern-.5em $\pmb{\rightarrow}$ \kern-.3em  \raisebox{-.3\totalheight}{\includegraphics[width=.6cm]{figures/MTA4_State_1.png}}\\            
            
$^1L_{t3}$ &  $[6,16,2]$ & \raisebox{-.3\totalheight}{\includegraphics[width=.6cm]{figures/MTA4_State_6}} \kern-.5em $\pmb{\rightarrow}$ \kern-.3em \raisebox{-.3\totalheight}{\includegraphics[width=.6cm]{figures/MTA4_State_16.png}} \kern-.5em $\pmb{\rightarrow}$ \kern-.3em  \raisebox{-.3\totalheight}{\includegraphics[width=.6cm]{figures/MTA4_State_2.png}}\\            
\hline    \hline
$^1L_{\theta 1}$ &  $[13, 8, 10, 7, 11, 2]$ & \raisebox{-.3\totalheight}{\includegraphics[width=.6cm]{figures/MTA4_State_13}} \kern-.5em $\pmb{\rightarrow}$ \kern-.3em \raisebox{-.3\totalheight}{\includegraphics[width=.6cm]{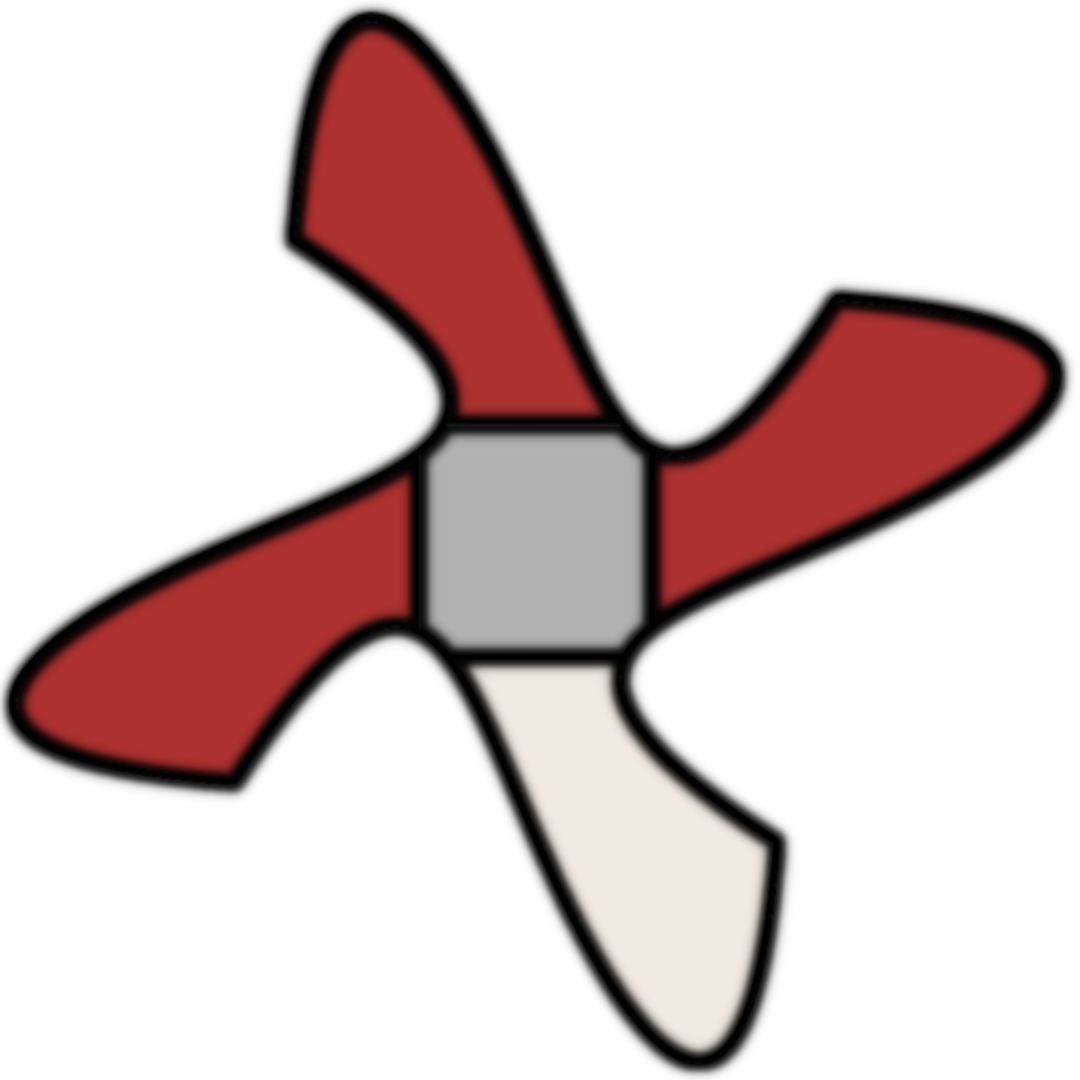}} \kern-.5em $\pmb{\rightarrow}$ \kern-.3em  \raisebox{-.3\totalheight}{\includegraphics[width=.6cm]{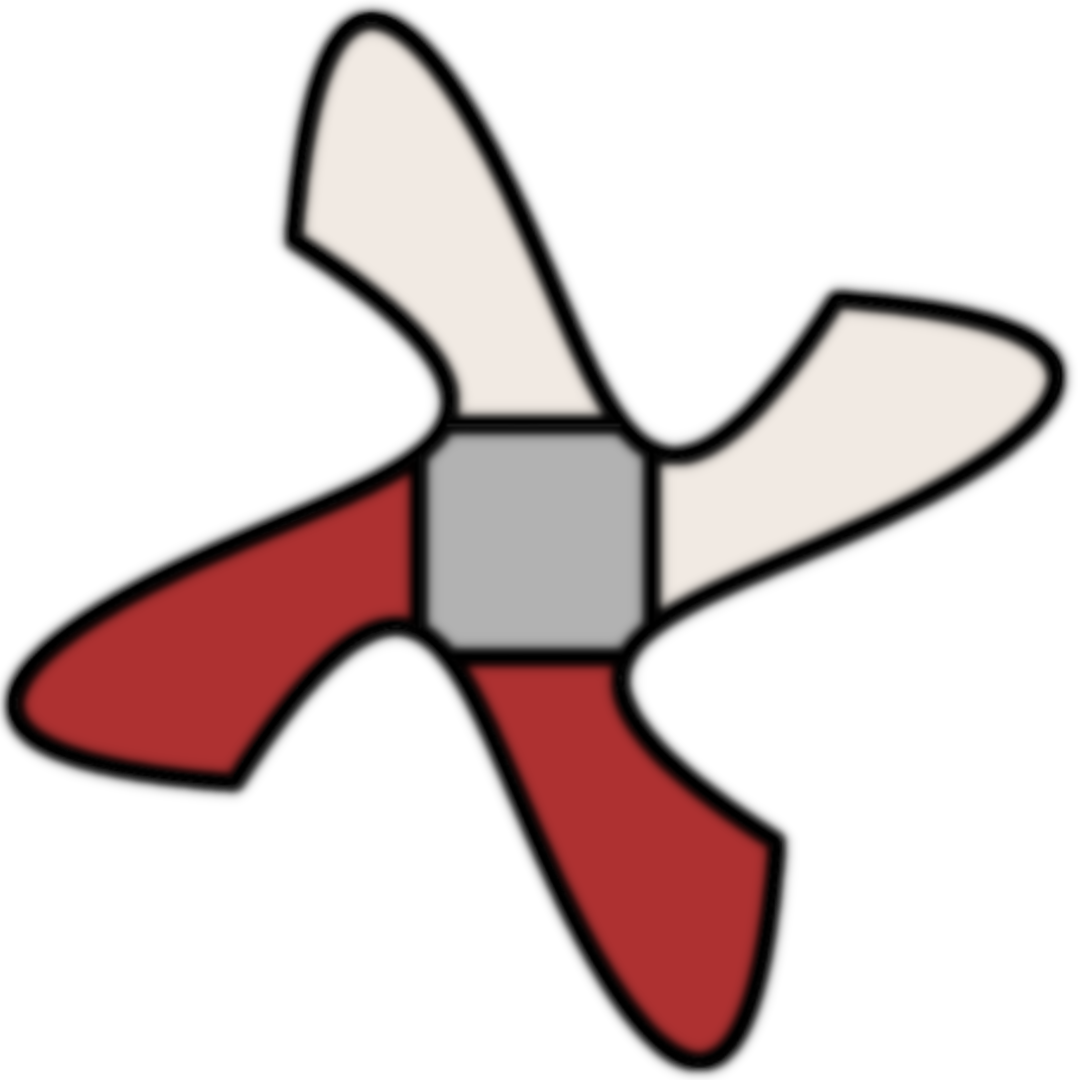}} \kern-.5em $\pmb{\rightarrow}$ \kern-.3em \raisebox{-.3\totalheight}{\includegraphics[width=.6cm]{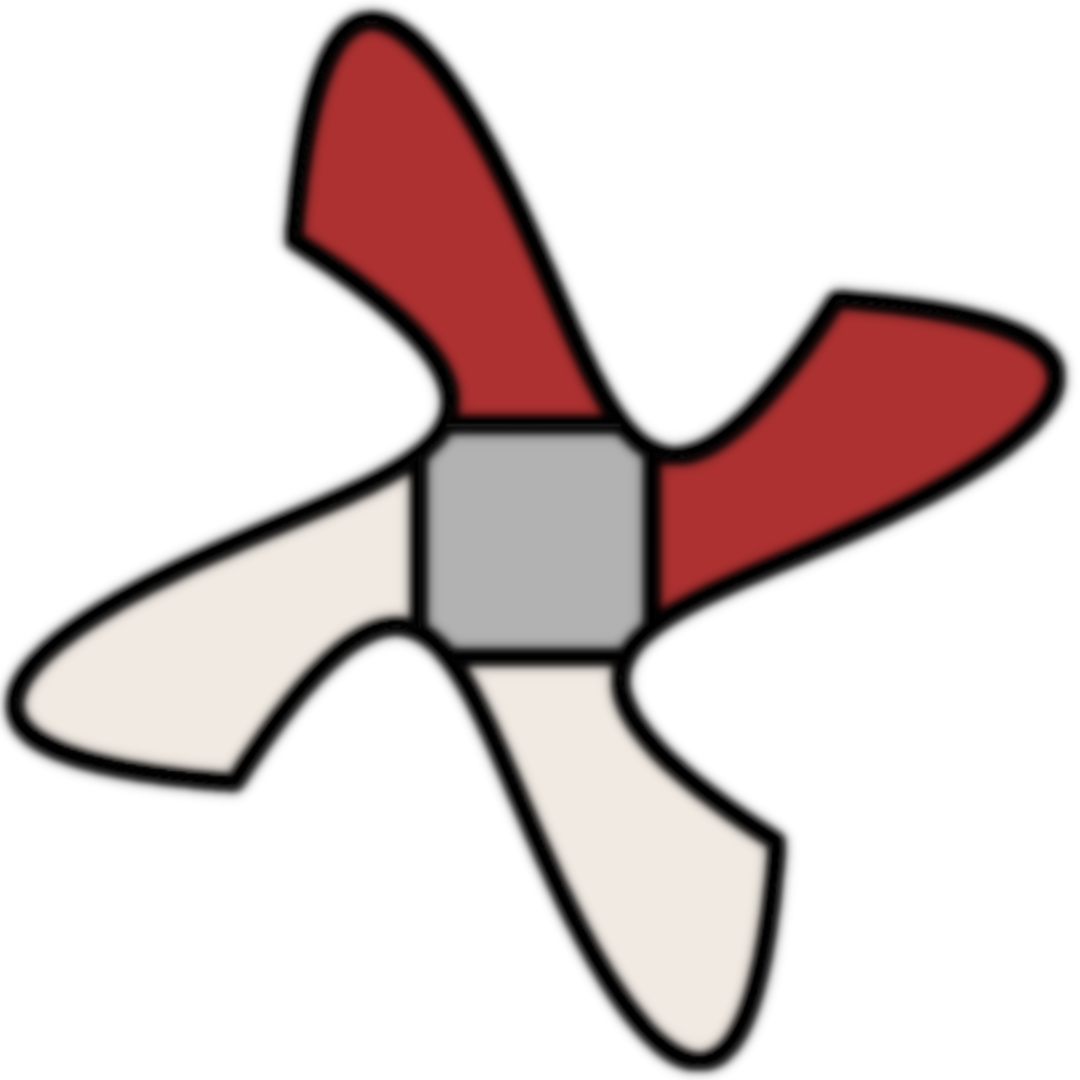}} \kern-.5em $\pmb{\rightarrow}$ \kern-.3em  \raisebox{-.3\totalheight}{\includegraphics[width=.6cm]{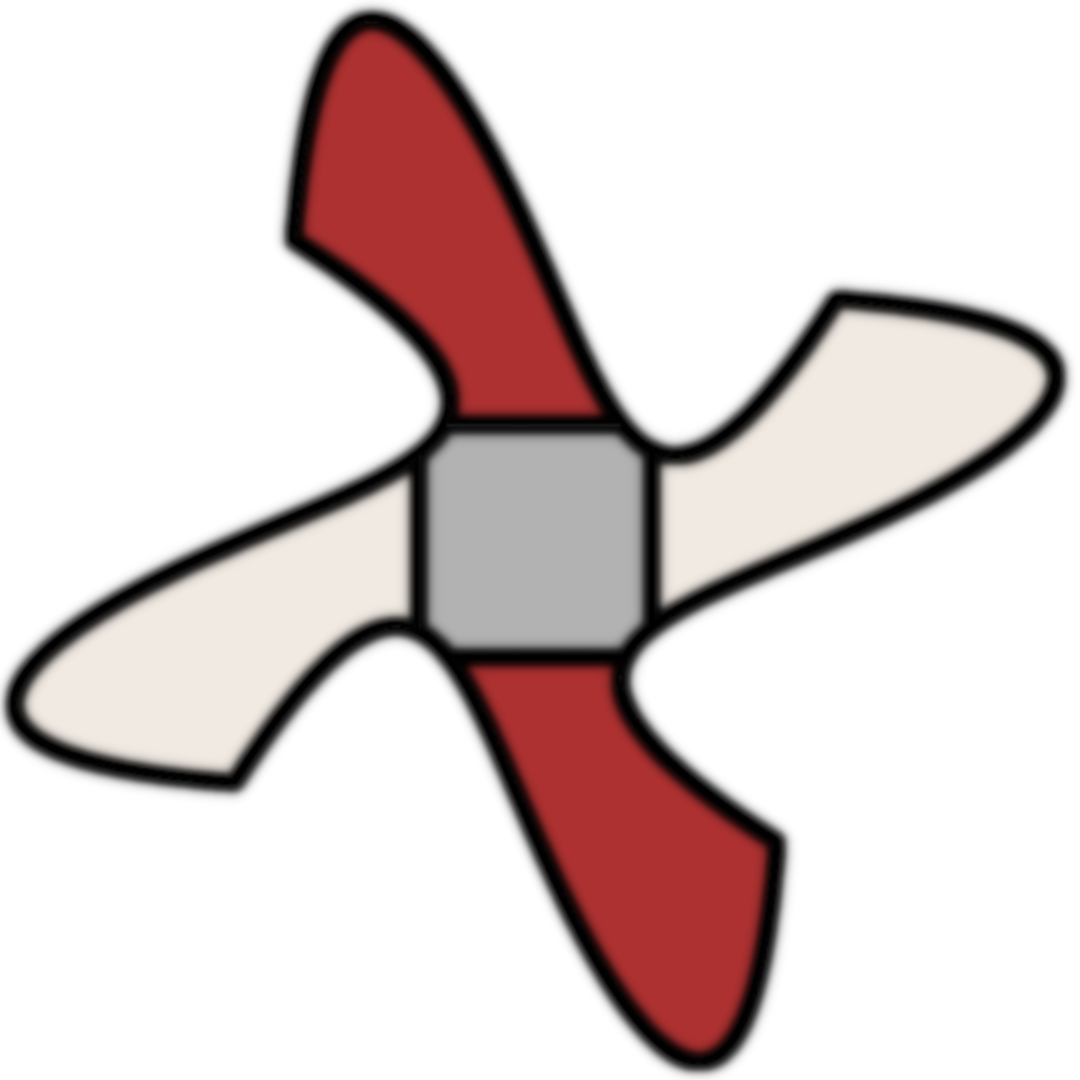}} \kern-.5em $\pmb{\rightarrow}$ \kern-.3em \raisebox{-.3\totalheight}{\includegraphics[width=.6cm]{figures/MTA4_State_2.png}}\\            
        
$^1L_{\theta 2}$ &  $[12, 13, 8, 10, 7]$ & \raisebox{-.3\totalheight}{\includegraphics[width=.6cm]{figures/MTA4_State_12}} \kern-.5em $\pmb{\rightarrow}$ \kern-.3em \raisebox{-.3\totalheight}{\includegraphics[width=.6cm]{figures/MTA4_State_13.png}} \kern-.5em $\pmb{\rightarrow}$ \kern-.3em  \raisebox{-.3\totalheight}{\includegraphics[width=.6cm]{figures/MTA4_State_8.png}} \kern-.5em $\pmb{\rightarrow}$  \kern-.3em  \raisebox{-.3\totalheight}{\includegraphics[width=.6cm]{figures/MTA4_State_10.png}} \kern-.5em $\pmb{\rightarrow}$  \kern-.3em  \raisebox{-.3\totalheight}{\includegraphics[width=.6cm]{figures/MTA4_State_7.png}}\\            
            
$^1L_{\theta 3}$ & \scriptsize{ $[16, 8, 10, 7, 12, 6, 1]$ }& \raisebox{-.3\totalheight}{\includegraphics[width=.5cm]{figures/MTA4_State_16}} \kern-.5em \tiny{$\pmb{\rightarrow}$} \kern-.3em \raisebox{-.3\totalheight}{\includegraphics[width=.5cm]{figures/MTA4_State_8.png}} \kern-.5em \tiny{$\pmb{\rightarrow}$} \kern-.3em  \raisebox{-.3\totalheight}{\includegraphics[width=.5cm]{figures/MTA4_State_10.png}} \kern-.5em \tiny{$\pmb{\rightarrow}$} \kern-.3em \raisebox{-.3\totalheight}{\includegraphics[width=.6cm]{figures/MTA4_State_7.png}} \kern-.5em \tiny{$\pmb{\rightarrow}$} \kern-.3em  \raisebox{-.3\totalheight}{\includegraphics[width=.5cm]{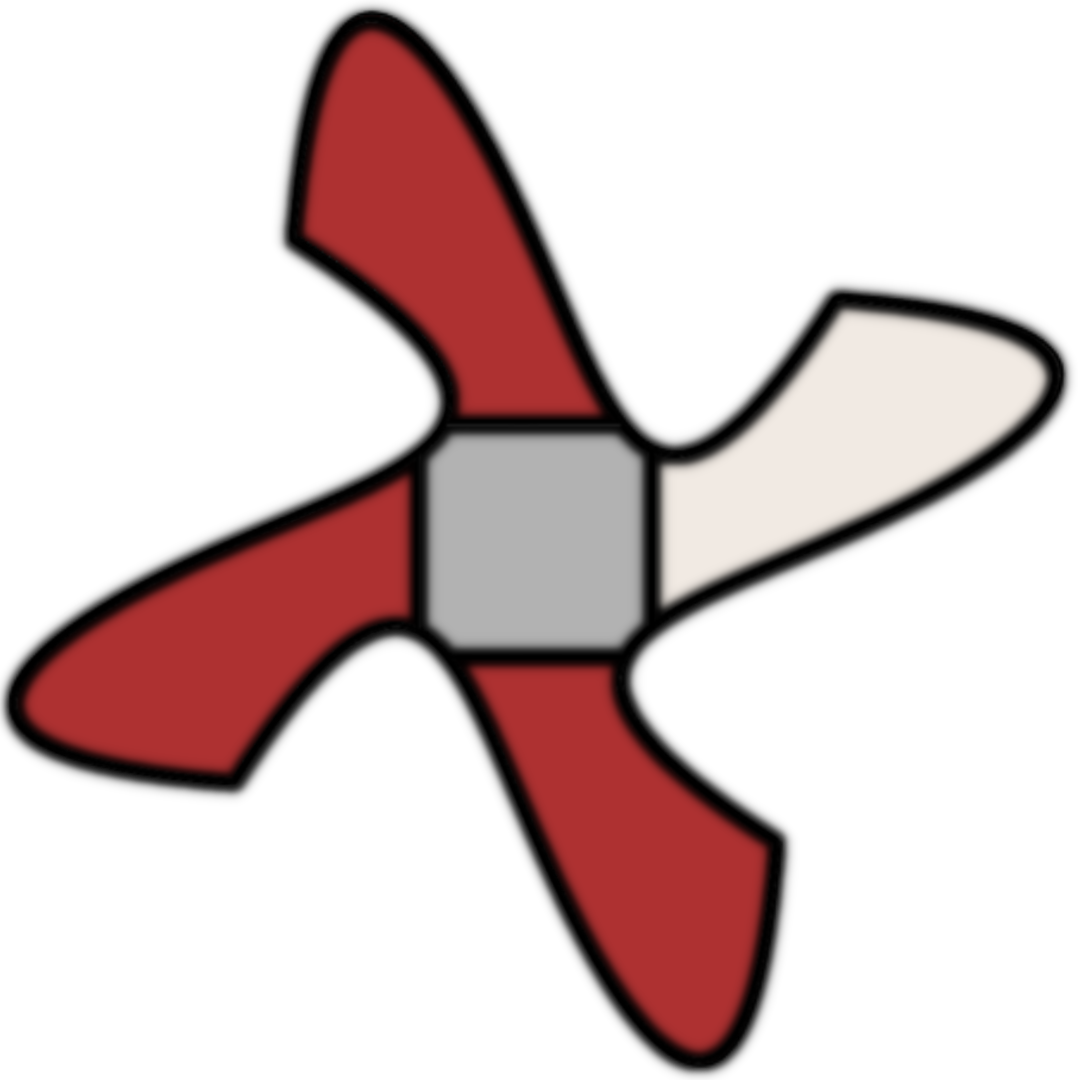}} \kern-.5em \tiny{$\pmb{\rightarrow}$} \kern-.3em \raisebox{-.3\totalheight}{\includegraphics[width=.5cm]{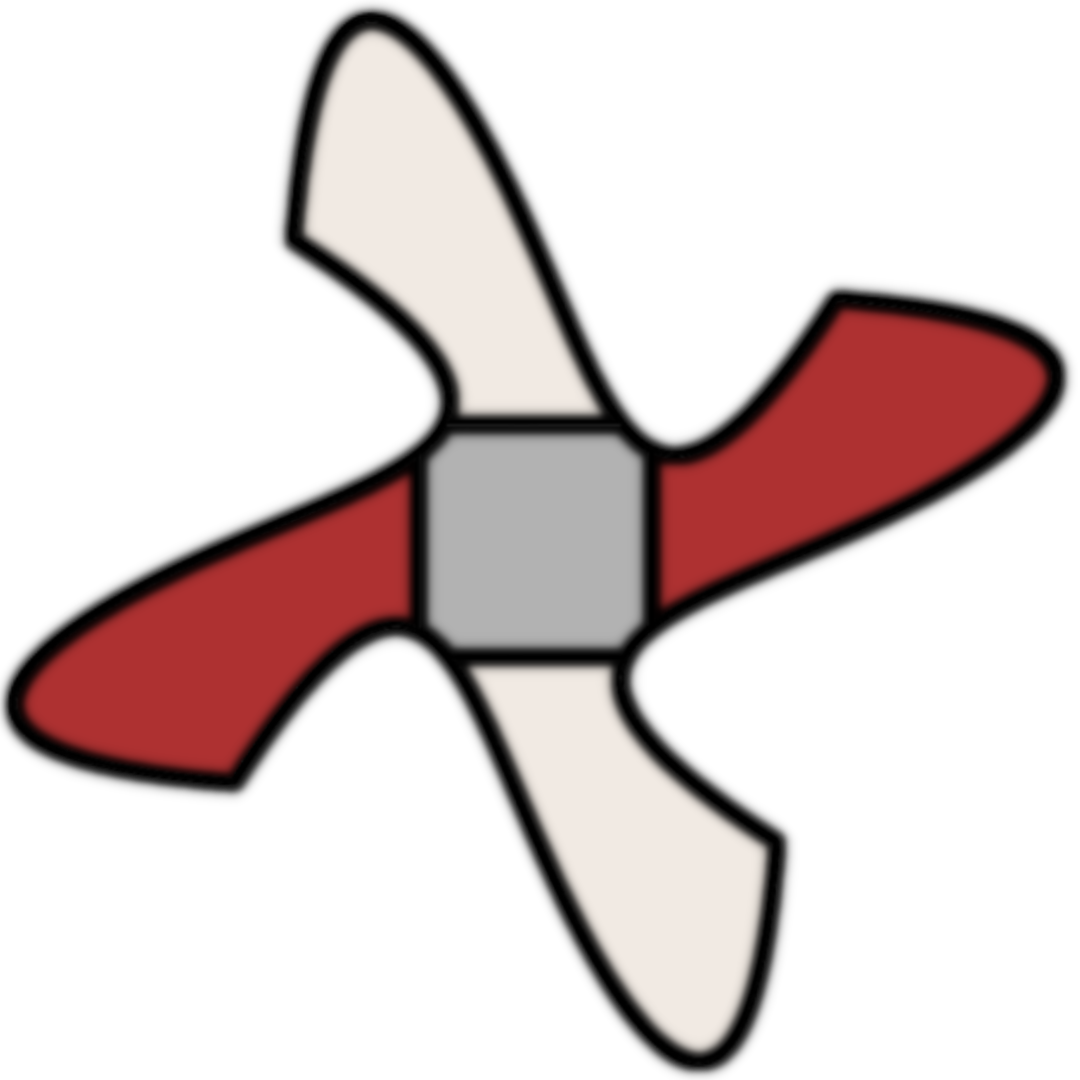}}  \kern-.5em \tiny{$\pmb{\rightarrow}$} \kern-.3em  \raisebox{-.3\totalheight}{\includegraphics[width=.5cm]{figures/MTA4_State_1.png}}\\            
\hline    
\end{tabular*}
\end{center}
\end{table}

One of the major motivations for data-driven control of soft robots is the difficulty in eliminating or minimizing motion artifacts due to small manufacturing inconsistencies. Many soft robots are designed to be symmetric but have significant biases in motion that are not captured in simulation. This further complicates motion planning, as symmetries are often exploited to simplify the control space. The presented four-limb robot is designed to be rotationally symmetric but is observed to exhibit biases in motion. To explore this, we tested symmetric permutations (i.e., 90 degree rotations of the actuation sequence) of the same translation-dominant gait $^1L_{t1}$ and plotted the results in Fig.~\ref{Fig:4perm}. Despite being expected to have identical translation and rotation at 90 degree rotations, the gaits, despite being individually fairly stable, possess vastly different magnitudes of translation, with one gait nearly doubling the rotation magnitude of the others. While simulation is a valuable tool for robotic control, these data emphasize a growing understanding in the soft robot community that data-driven methods (which can also augment existing simulation tools) are needed to address the unique challenges of soft robots, including that of manufacturing inconsistencies.

\begin{figure}[ht]
\centering
    \includegraphics[width=.7\columnwidth,trim=4cm 17.8cm 5cm 9.1cm, clip=true]{figures/1t_legend.pdf} \\[-1.5ex] 
    \subfloat[][]{\includegraphics[width = .43\columnwidth,trim=2cm 0.2cm 2.2cm .8cm, clip=true]{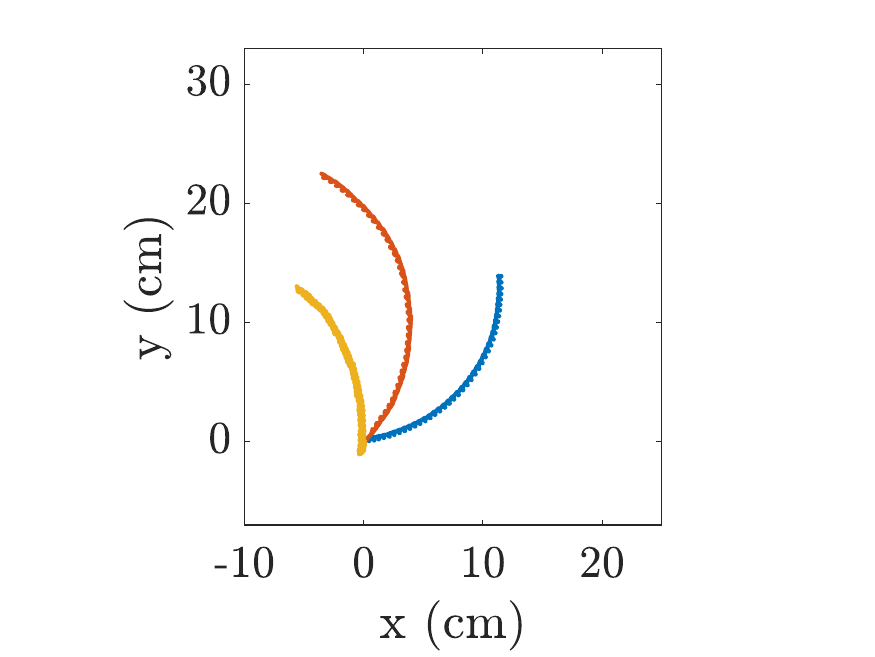}}\hfill
    \subfloat[][]{\includegraphics[width = .54\columnwidth,trim= 0cm 0.1cm 1.3cm .4cm, clip=true]{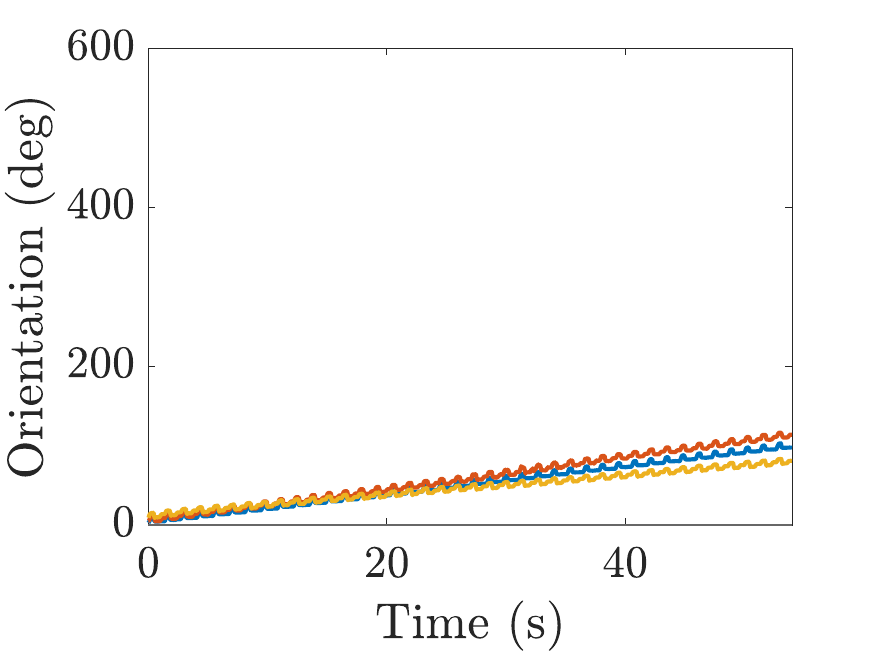}}\par
    
    \includegraphics[width=.6\columnwidth,trim=6cm 19.7cm 5cm 7cm, clip=true]{figures/1theta_legend.pdf} \\[-1.8ex] 
    \subfloat[][]{\includegraphics[width = .43\columnwidth,trim=2cm 0.2cm 2.2cm .8cm, clip=true]{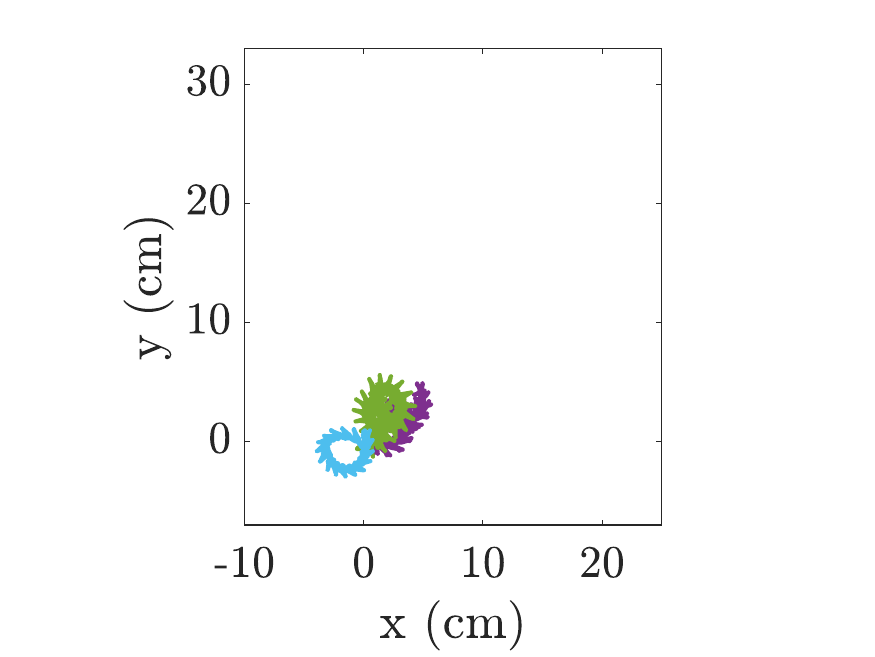}}\hfill
    \subfloat[][]{\includegraphics[width = .54\columnwidth,trim= 0cm 0.1cm 1.3cm .4cm, clip=true]{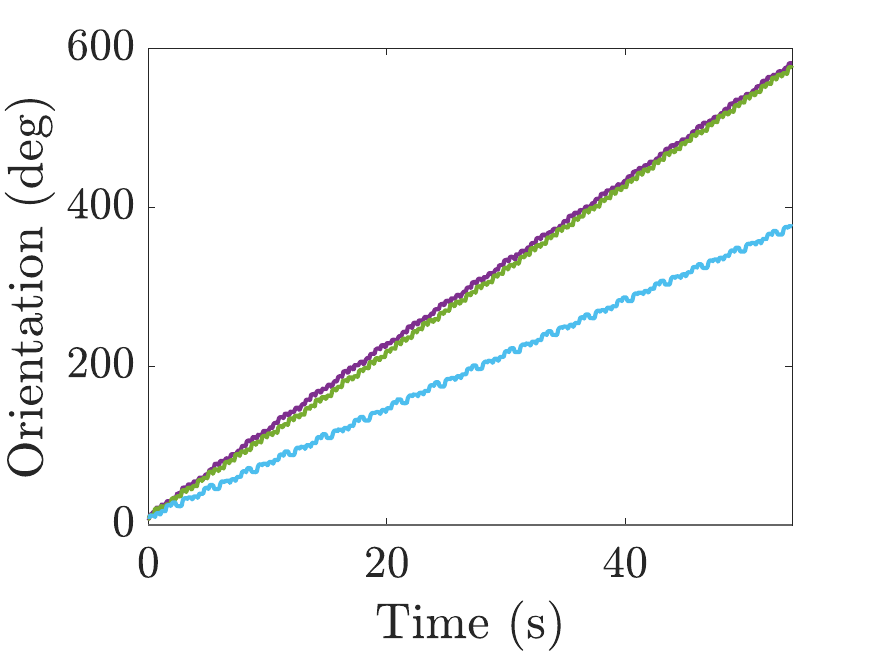}}\par

    \caption{\TetraSoRo~gaits synthesized for the rubber mat (substrate 1).  Experimental (a) trajectory and (b) rotation plots for the translation-dominant gaits ($^1L_{t1}, ^1L_{t2}, ^1L_{t3}$) are shown in addition to the (c) trajectory and (d) orientation plots for the rotation-dominant gaits ($^1L_{\theta 1}, ^1L_{\theta 2}, ^1L_{\theta 3}$). }
    \label{Fig:4synth}
\end{figure}

\begin{figure}[ht]
\centering
    \hspace{5mm}\fbox{
\parbox[c]{.75\columnwidth}{
\centering
\includegraphics[width=.07\columnwidth,trim=6.3cm 14.2cm 14.15cm 13.3cm, clip=true]{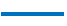} \raisebox{-.3\totalheight}{\includegraphics[width=.6cm]{figures/MTA4_State_3}} \kern-.5em $\pmb{\rightarrow}$ \kern-.3em \raisebox{-.3\totalheight}{\includegraphics[width=.6cm]{figures/MTA4_State_15.png}} \kern-.5em $\pmb{\rightarrow}$ \kern-.3em  \raisebox{-.3\totalheight}{\includegraphics[width=.6cm]{figures/MTA4_State_1.png}} 
\includegraphics[width=.07\columnwidth,trim=8.52cm 14.2cm 11.93cm 13.3cm, clip=true]{figures/perm_legend.pdf} \raisebox{-.3\totalheight}{\includegraphics[width=.6cm]{figures/MTA4_State_2}} \kern-.5em $\pmb{\rightarrow}$ \kern-.3em \raisebox{-.3\totalheight}{\includegraphics[width=.6cm]{figures/MTA4_State_8.png}} \kern-.5em $\pmb{\rightarrow}$ \kern-.3em  \raisebox{-.3\totalheight}{\includegraphics[width=.6cm]{figures/MTA4_State_1.png}} \\
\includegraphics[width=.07\columnwidth,trim=10.82cm 14.2cm 9.71cm 13.3cm, clip=true]{figures/perm_legend.pdf} \raisebox{-.3\totalheight}{\includegraphics[width=.6cm]{figures/MTA4_State_9}} \kern-.5em $\pmb{\rightarrow}$ \kern-.3em \raisebox{-.3\totalheight}{\includegraphics[width=.6cm]{figures/MTA4_State_12.png}} \kern-.5em $\pmb{\rightarrow}$ \kern-.3em  \raisebox{-.3\totalheight}{\includegraphics[width=.6cm]{figures/MTA4_State_1.png}} 
\includegraphics[width=.07\columnwidth,trim=12.95cm 14.2cm 7.5cm 13.3cm, clip=true]{figures/perm_legend.pdf} \raisebox{-.3\totalheight}{\includegraphics[width=.6cm]{figures/MTA4_State_5}} \kern-.5em $\pmb{\rightarrow}$ \kern-.3em \raisebox{-.3\totalheight}{\includegraphics[width=.6cm]{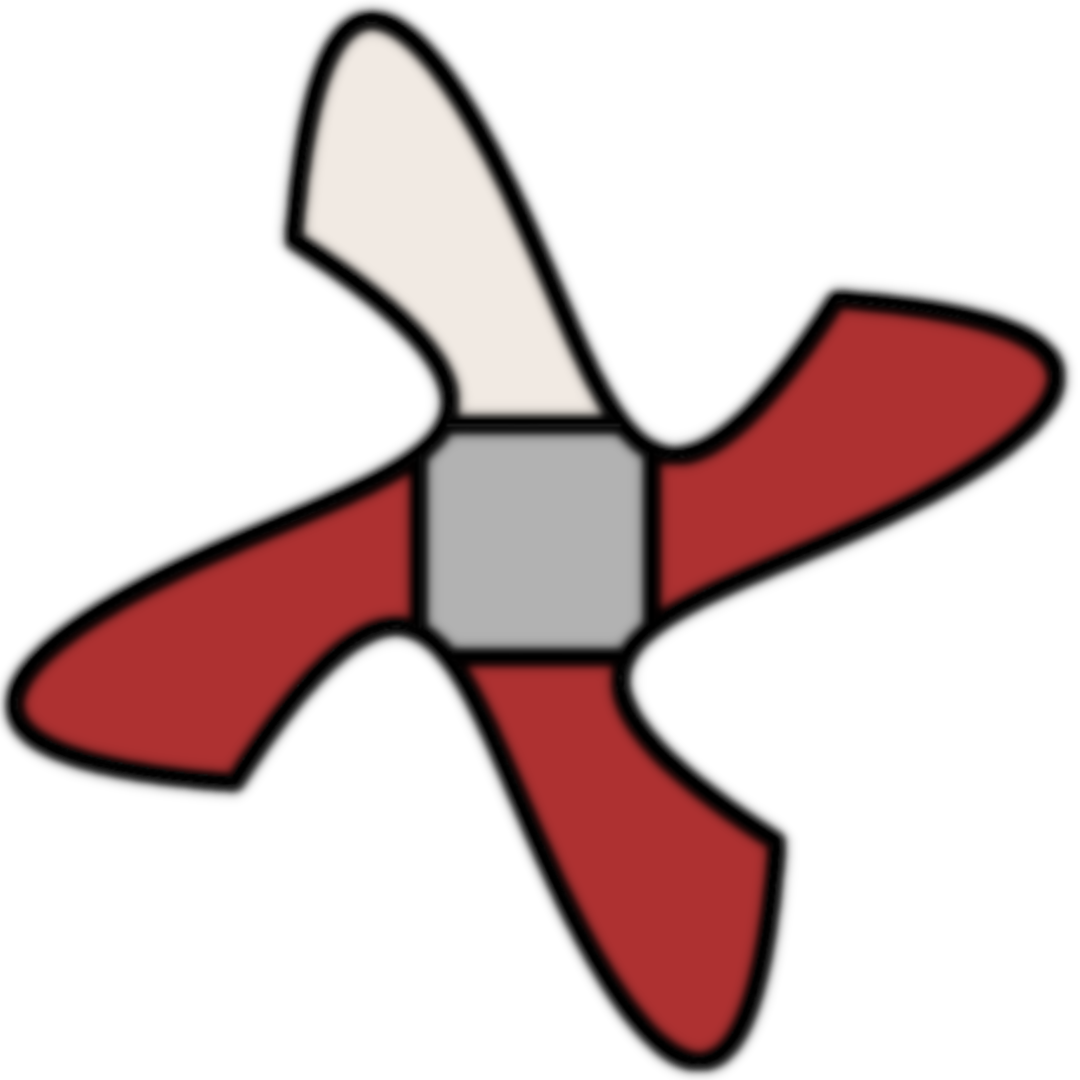}} \kern-.5em $\pmb{\rightarrow}$ \kern-.3em  \raisebox{-.3\totalheight}{\includegraphics[width=.6cm]{figures/MTA4_State_1.png}} 
}}\vspace*{-2mm}
     \subfloat[][]{\includegraphics[width = .43\columnwidth,trim=2cm 0.2cm 2.2cm .5cm, clip=true]{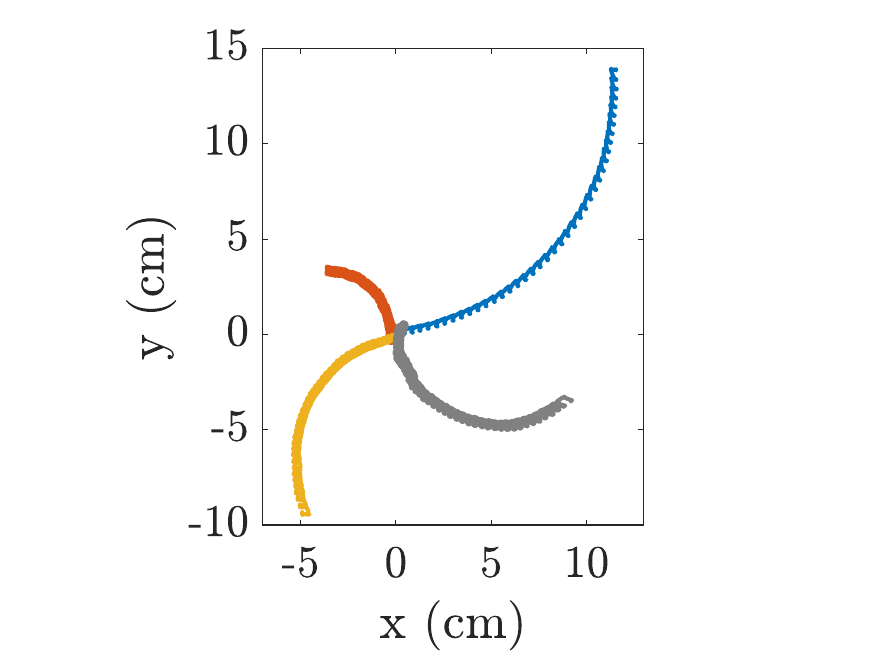}}\hfill
    \subfloat[][]{\includegraphics[width = .54\columnwidth,trim= 0cm 0.1cm 1.3cm .5cm, clip=true]{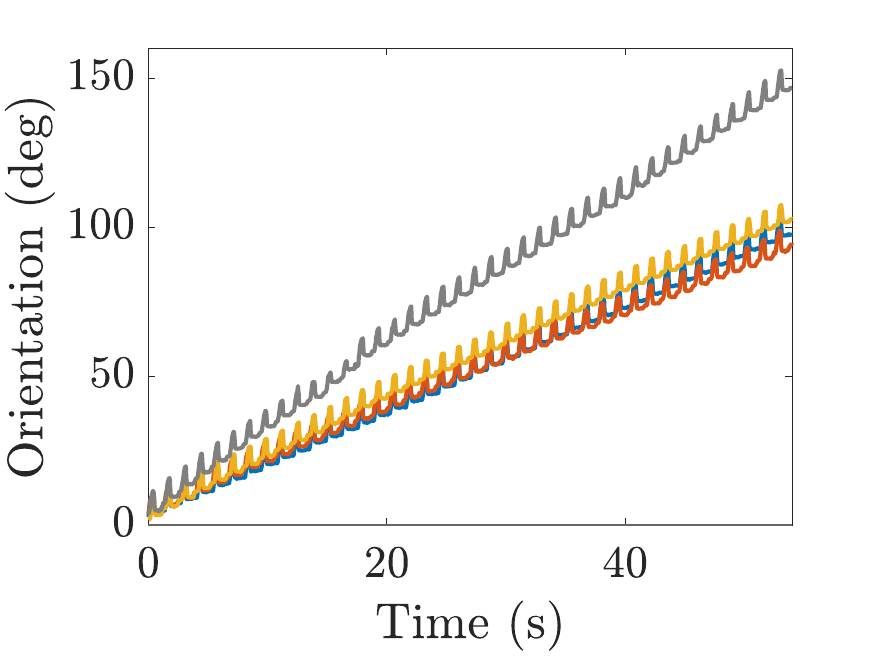}}\par

    \caption{\TetraSoRo~translation-dominant gait $^1L_{t1}$ is compared with its three rotationally symmetric permutation gaits in terms of experimental (a) trajectory and (b) orientation. }
    \label{Fig:4perm}
\end{figure}

\subsubsection{Loss-of-limb Scenario Experimental Results}

%\begin{figure}
%    \centering
%    \subfloat[][]{\includegraphics[page=1,width=.45\columnwidth,trim= 3.3cm 6.5cm 18.9cm 6.5cm, clip=true]{figures/loss_of_limb.pdf}} \hfill
%    \subfloat[][]{\includegraphics[page=1,width=.45\columnwidth,trim= 16.2cm 6.5cm 6cm 6.5cm, clip=true]{figures/loss_of_limb.pdf}}
%    \caption{Single actuator failure (loss of limb) resulting in reformulating from (a) original digraph to (b) pruned digraph. All lines connecting nodes represent bidirectional edges (i.e., two distinct robot state transitions) for ease of visualization.}
%\label{Fig:expsetup}
%\end{figure}

While the soft robotics field continues to mature, robustness remains a problem. Therefore, control algorithms should be designed to be resilient in the case of actuator failure (e.g., a tendon snapping). The proposed gait synthesis algorithm contains an input variable of functioning actuators; in the case of one or more failures, the gait synthesis can be easily and quickly recalculated by first pruning the unreachable states and edges of the digraph, as shown in Fig.~\ref{Fig:lol_graph}. \Tab \ref{Tab:4lol} and \Fig \ref{Fig:4lol} detail the results of the gait synthesis for loss-of-limb scenario. The ``lost" limb was varied for gait synthesis and the translation-dominant and rotation-dominant gaits that exhibited the best function evaluations (i.e., best predicted trajectories and rotation values) were chosen. One limb was excluded from the limb variation in translation-dominant synthesis to avoid repeated gaits in \Tab \ref{Tab:4synth}. We were able to successfully synthesize translation-dominant and rotation-dominant gaits in the case of actuator failure. One translation-dominant gait, $^1L_{t4}$, even exhibited higher translation magnitude with similar rotation, as compared to the previously synthesized gaits. This is expected to be due to variation between predicted and expected behavior.

\begin{figure}
    \centering
    \subfloat[][]{\includegraphics[page=2,width=.49\columnwidth,trim= 1.2cm 6.5cm 18.9cm 4.9cm, clip=true]{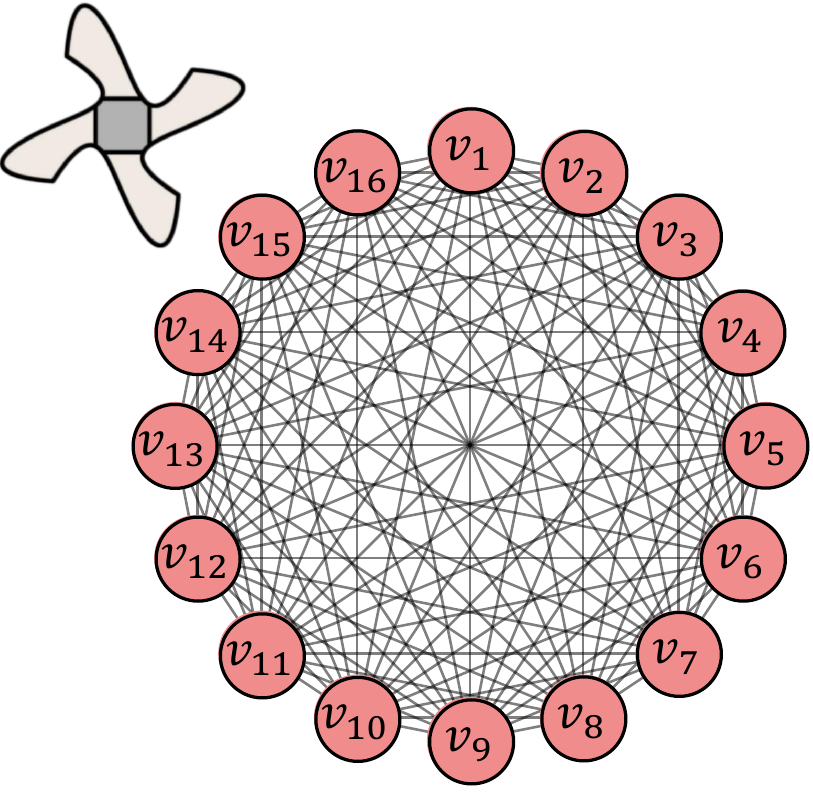}} \hfill
    \subfloat[][]{\includegraphics[page=2,width=.49\columnwidth,trim= 15.8cm 6.5cm 4cm 4.9cm, clip=true]{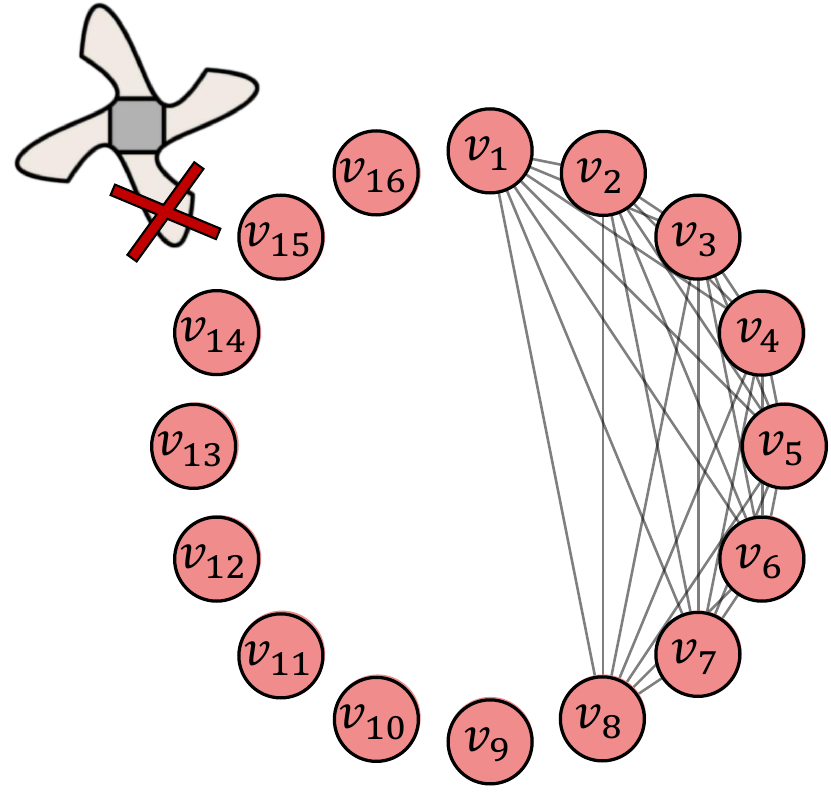}} \par
    \caption{Single actuator failure (loss of limb) resulting in reformulating from (a) original digraph to (b) pruned digraph. All lines connecting nodes represent bidirectional edges (i.e., two distinct robot state transitions) for ease of visualization.}
\label{Fig:lol_graph}
\end{figure}
\begin{table}
\renewcommand\arraystretch{2}
\begin{center}
\caption{\TetraSoRo~Loss-of-Limb Synthesized Gaits on Rubber Mat. }
\label{Tab:4lol}
\noindent\begin{tabular*}{\columnwidth}{@{\extracolsep{\fill}}ccc@{}}
\hline        
Gait & $V(L)$ & Robot States \\           \hline \hline  
$^1L_{t3}$ &  $[3,8,1]$ & \raisebox{-.3\totalheight}{\includegraphics[width=.6cm]{figures/MTA4_State_3}} \kern-.5em $\pmb{\rightarrow}$ \kern-.3em \raisebox{-.3\totalheight}{\includegraphics[width=.6cm]{figures/MTA4_State_8.png}} \kern-.5em $\pmb{\rightarrow}$ \kern-.3em  \raisebox{-.3\totalheight}{\includegraphics[width=.6cm]{figures/MTA4_State_1.png}}\\            
        
$^1L_{t4}$ &  $[5,3,8,1]$ & \raisebox{-.3\totalheight}{\includegraphics[width=.6cm]{figures/MTA4_State_5}} \kern-.5em $\pmb{\rightarrow}$ \kern-.3em \raisebox{-.3\totalheight}{\includegraphics[width=.6cm]{figures/MTA4_State_3.png}} \kern-.5em $\pmb{\rightarrow}$ \kern-.3em  \raisebox{-.3\totalheight}{\includegraphics[width=.6cm]{figures/MTA4_State_8.png}} \kern-.5em $\pmb{\rightarrow}$  \kern-.3em  \raisebox{-.3\totalheight}{\includegraphics[width=.6cm]{figures/MTA4_State_1.png}}\\            
                   
\hline    \hline
$^1L_{\theta 3}$ &  $[4, 5, 6, 3, 2]$ & \raisebox{-.3\totalheight}{\includegraphics[width=.6cm]{figures/MTA4_State_4}} \kern-.5em $\pmb{\rightarrow}$ \kern-.3em \raisebox{-.3\totalheight}{\includegraphics[width=.6cm]{figures/MTA4_State_5.png}} \kern-.5em $\pmb{\rightarrow}$ \kern-.3em  \raisebox{-.3\totalheight}{\includegraphics[width=.6cm]{figures/MTA4_State_6.png}} \kern-.5em $\pmb{\rightarrow}$ \kern-.3em \raisebox{-.3\totalheight}{\includegraphics[width=.6cm]{figures/MTA4_State_3.png}} \kern-.5em $\pmb{\rightarrow}$ \kern-.3em  \raisebox{-.3\totalheight}{\includegraphics[width=.6cm]{figures/MTA4_State_2.png}} \kern-.5em $\pmb{\rightarrow}$ \kern-.3em \raisebox{-.3\totalheight}{\includegraphics[width=.6cm]{figures/MTA4_State_2.png}}\\            
        
$^1L_{\theta 4}$ &  $[11, 13, 3, 1]$ & \raisebox{-.3\totalheight}{\includegraphics[width=.6cm]{figures/MTA4_State_11}} \kern-.5em $\pmb{\rightarrow}$ \kern-.3em \raisebox{-.3\totalheight}{\includegraphics[width=.6cm]{figures/MTA4_State_13.png}} \kern-.5em $\pmb{\rightarrow}$ \kern-.3em  \raisebox{-.3\totalheight}{\includegraphics[width=.6cm]{figures/MTA4_State_3.png}} \kern-.5em $\pmb{\rightarrow}$  \kern-.3em  \raisebox{-.3\totalheight}{\includegraphics[width=.6cm]{figures/MTA4_State_1.png}}\\                    
\hline    
\end{tabular*}
\end{center}
\end{table}

\begin{figure}[ht]
\centering
    \includegraphics[width=.6\columnwidth,trim=6cm 17.8cm 5cm 9.1cm, clip=true]{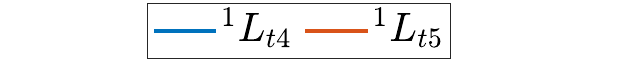} \\[-1.5ex] 
    \subfloat[][]{\includegraphics[width = .43\columnwidth,trim=2.7cm 0.2cm 2.8cm 2cm, clip=true]{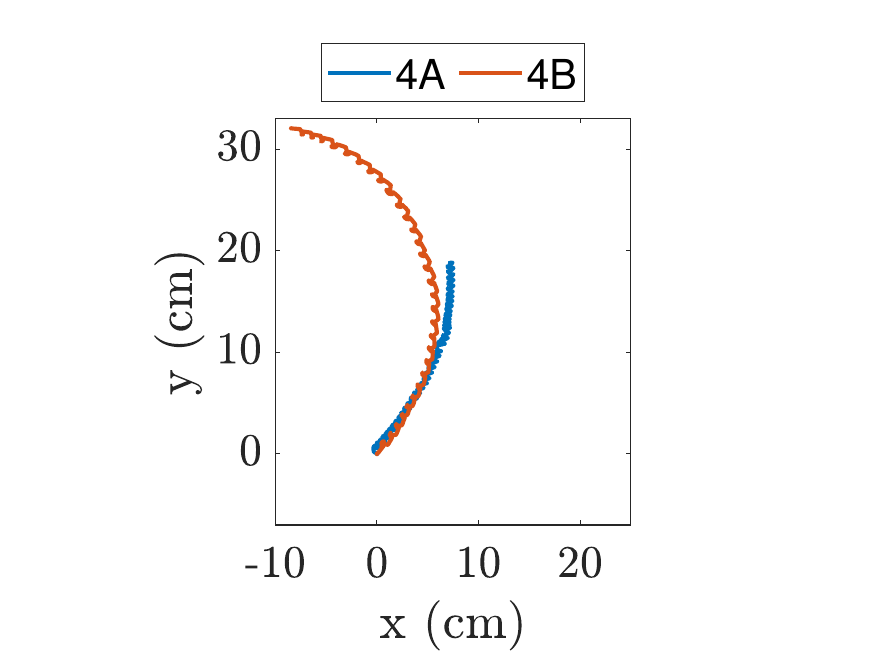}}\hfill
    \subfloat[][]{\includegraphics[width = .54\columnwidth,trim= 0cm 0.1cm 1.3cm .4cm, clip=true]{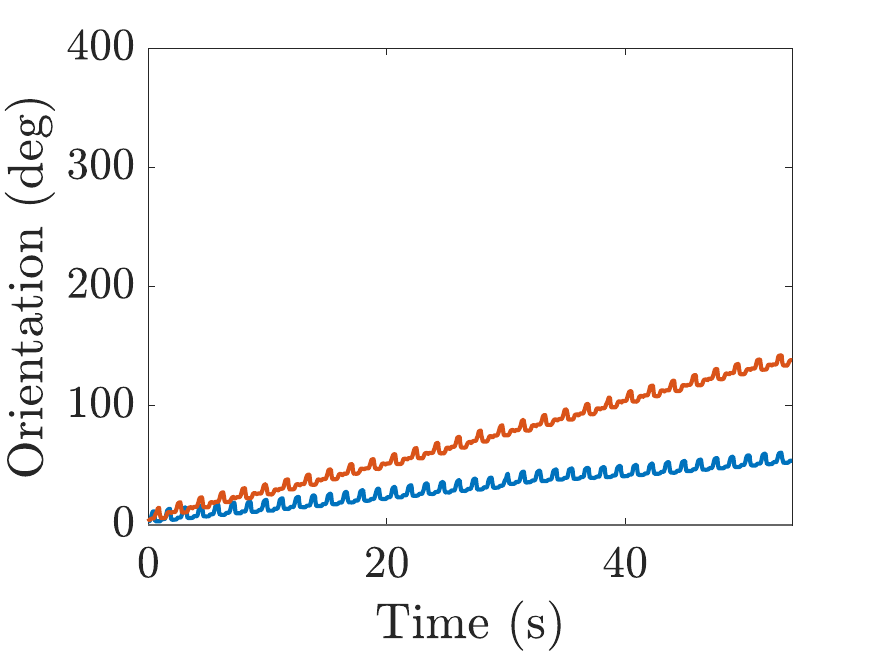}}\par
    
    \includegraphics[width=.62\columnwidth,trim=6cm 17.8cm 5cm 9.1cm, clip=true]{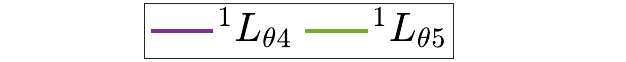} \\[-1.5ex] 
    \subfloat[][]{\includegraphics[width = .43\columnwidth,trim=1.6cm 0.2cm 2.9cm .4cm, clip=true]{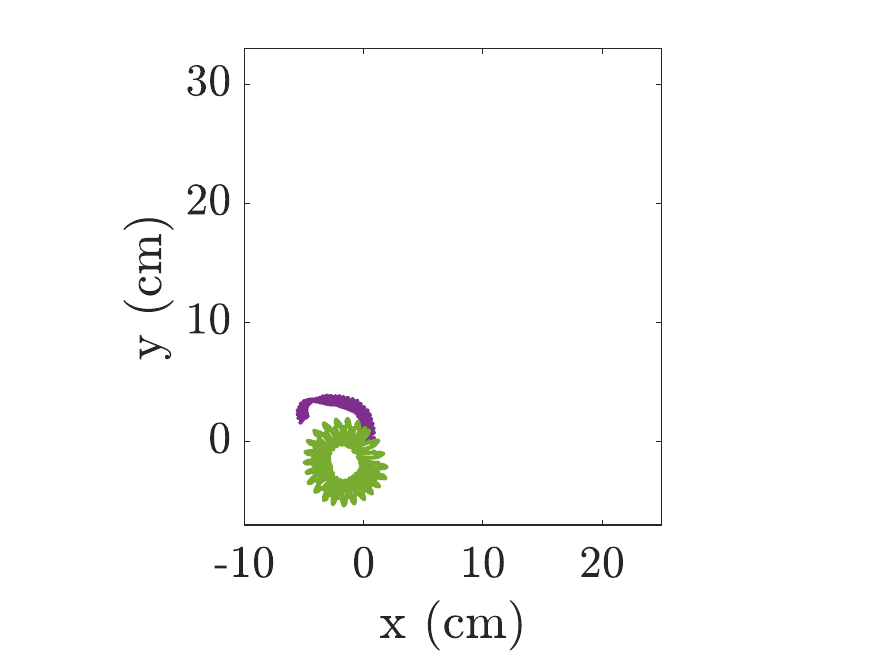}}\hfill
    \subfloat[][]{\includegraphics[width = .54\columnwidth,trim= 0cm 0.1cm 1.3cm .4cm, clip=true]{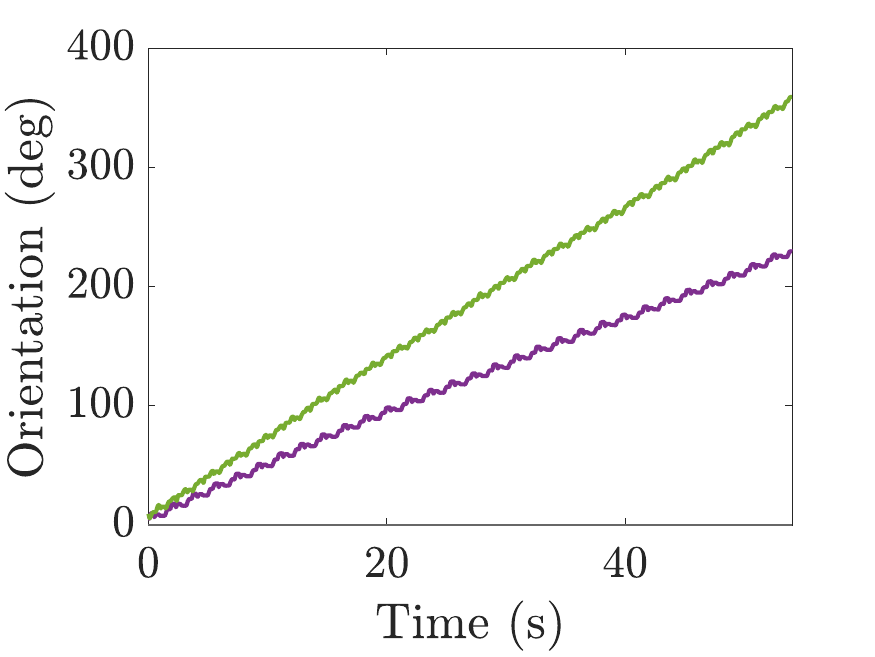}}\par

    \caption{\TetraSoRo~gaits synthesized for the rubber mat (substrate 1) in the case of single-actuator failure (loss of limb).  Experimental (a) trajectory and (b) orientation plots for the translation-dominant gaits ($^1L_{t4}, ^1L_{t5}$) are shown in addition to the (c) trajectory and (d) orientation plots for the rotation-dominant gaits ($^1L_{\theta 4}, ^1L_{\theta 5}$). }
    \label{Fig:4lol}
\end{figure}

\subsection{Gait Characterization}
To summarize the experimental results and analyze the success and applicability of the proposed gait synthesis method, we can construct \textit{gait characterization} plots. It can be difficult to define optimality of the gaits as different factors (e.g., translation magnitude, rotation magnitude, and variance) can impact their usefulness in practice. To attempt to visualize these factors, \Fig \ref{Fig:gaitchar} plots the experimental average translational speed versus the average rotational speed for every gait tested (both intuitive and synthesized). Bubbles (ellipses) are used to represent the data, where the centroid is the average of the data and the ellipse axes represent $\pm1$ standard deviation. Finally, a line that originates at $(0,0)$ and is proportionally scaled to the data is used to divide each plot into two regions: rotation-dominance (upper left) and translation-dominance (bottom right). Data points with the smallest bubbles (i.e., smallest variance) that lie furthest from these lines are considered optimal as they achieve the best separation between rotation and translation. 

There are a few observations of interest. 
\begin{enumerate}
    \item The synthesized rotation-dominant and translation-dominant gaits are generally well-separated, as indicated by the dotted line.
    \item The intuitive gaits tend to lie closer to line, suggesting a greater degree of translation/rotation coupling.
    \item The four-limb robot achieves the greatest speeds (both translational and rotational) with the lowest percentage of variance, but struggles to achieve near-zero translation or rotation, most likely due to its stronger rotational inclination.
    \item The three-limb robot and carpet experiments produce the lowest degree of separation with the highest degree of variance.
\end{enumerate}

 These plots provide a concise overview of the experiments, including the degree to which the robot-environment pairs conform to the model assumptions; the three-limb carpet experiments violated methodology assumptions (e.g., lack of repeatability of quasi-static states, inconsistent curling/uncurling, uneven and heterogeneous substrate), the effect of which is clearly reflected in the gait characterization plot. As can be seen in \Tab \ref{Tab:speeds}, the methodology achieves an average speed improvement over intuitive gaits of 82\% and 97\% for translation and rotation, respectively, achieving maximum translation speeds of 3.68 mm/s (0.017 BL/s)  and 7.15 mm/s (0.033 BL/s) as well as maximum rotation speeds of 4.23$^{\circ}$/s and 10.61$^{\circ}$/s for the \TriSoRo~and \TetraSoRo, respectively.  Nevertheless, these data show that the proposed gait method strategy is successful and resilient, capable of learning gaits for soft robots with unintuitive behavior, unexpected asymmetrical motion, coupled translational/rotational movements, and complex interaction with the environment even when testing for different robot morphologies, different surfaces, assumption violations, and simulated actuator failure. These experiments also emphasize the importance of considering robot-environment interactions rather than just the robot in free space, as even varying the frictional interaction between the soft robot and flat substrate can significantly alter the robot behavior and optimal locomotion gaits.
\begin{figure*}[ht]
\newcommand{\subfigwidth}{0.48\textwidth}
\newcommand{\figwidth}{0.6\textwidth}
\centering
    \includegraphics[width=\figwidth,trim=1cm 25.3cm 1cm 1.6cm, clip=true]{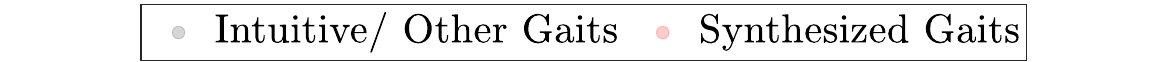} \\[-2.2ex] 
    \subfloat[][Three-limb substrate 1 (rubber mat)]{\includegraphics[width = \subfigwidth,trim= 0.2cm 0cm 1cm .3cm, clip=true]{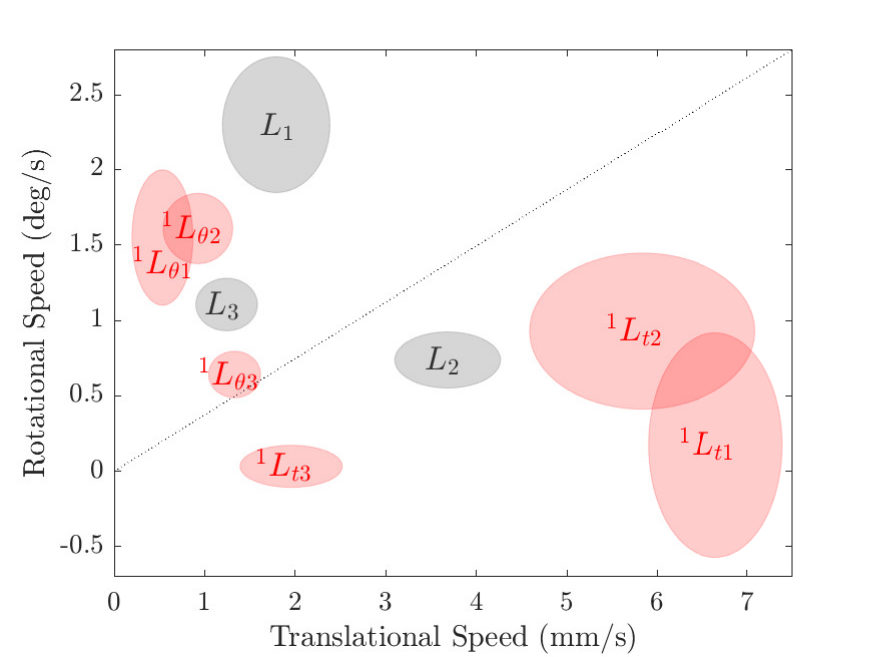}}\hfill
    \subfloat[][Three-limb substrate 2 (whiteboard)]{\includegraphics[width = \subfigwidth,trim= 0.2cm 0cm 1cm .3cm, clip=true]{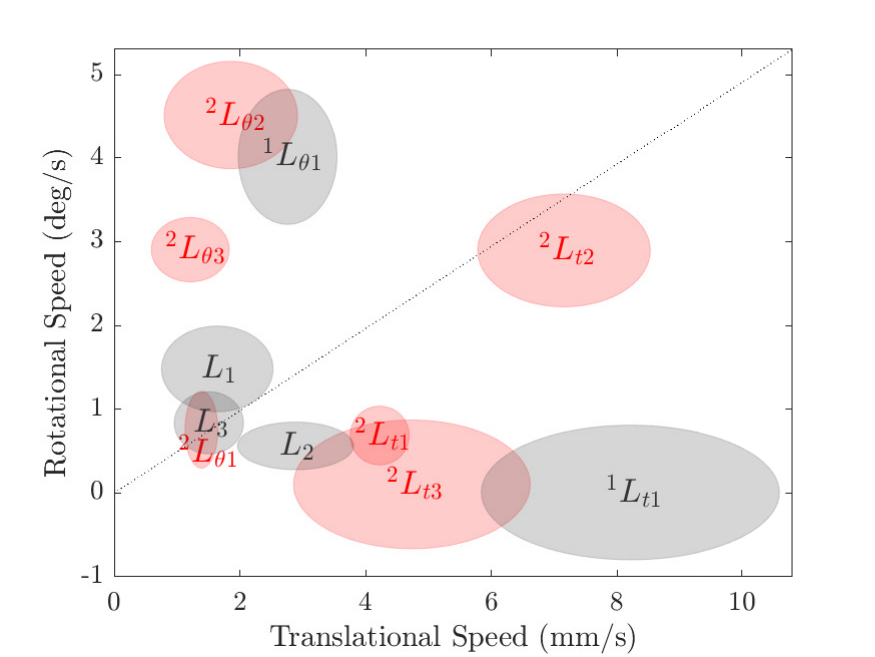}}\\[-2ex] 
    \subfloat[][Three-limb substrate 3 (carpet)]{\includegraphics[width = \subfigwidth,trim= 0.2cm 0cm 1cm .3cm, clip=true]{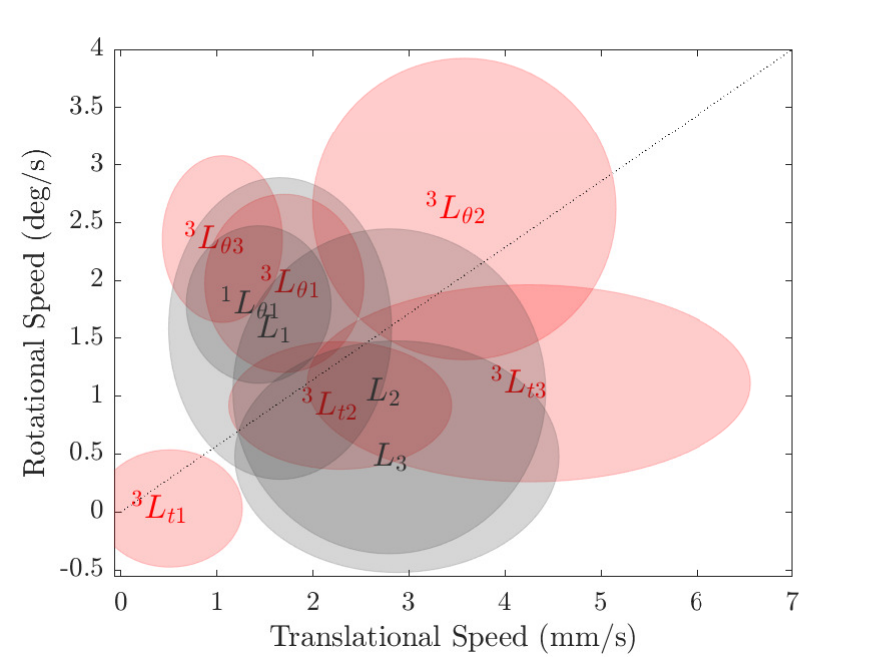}}\hfill
    \subfloat[][Four-limb substrate 1 (rubber mat)]{\includegraphics[width = \subfigwidth,trim= 0.18cm 0cm 1cm .3cm, clip=true]{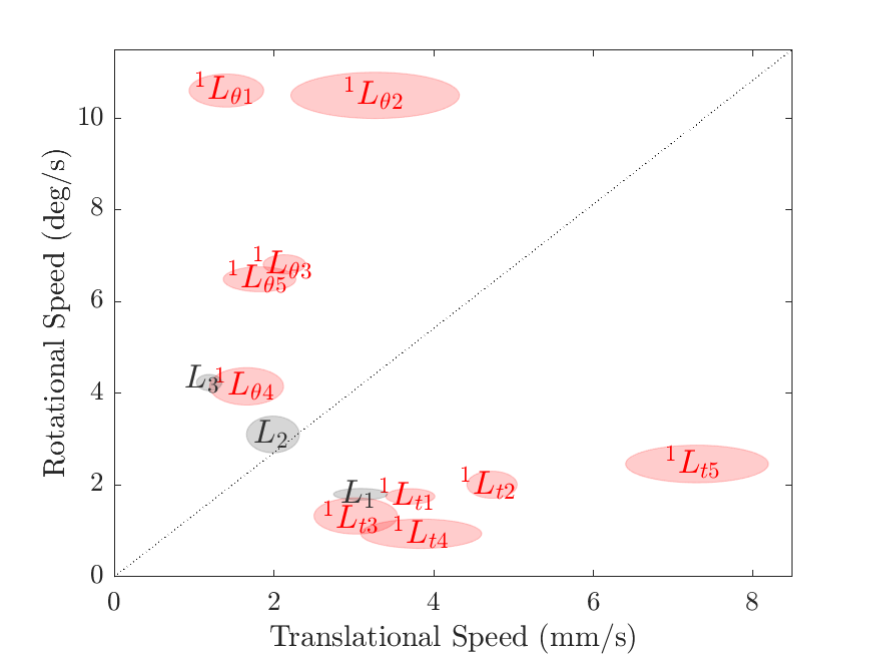}}\par  
    \caption{Gait characterization where the height and width of the ellipsoid bubbles represent $\pm1$ standard deviation of rotational speed and translation speed, respectively, of a gait cycle. Red bubbles correspond to gaits synthesized for the given robot-environment pair, and grey bubbles correspond to the remaining gaits.}
    \label{Fig:gaitchar}
\end{figure*}

\begin{table*}[ht]
\captionsetup{justification=centering, labelsep=newline,textfont=sc}
    \caption{Speed and Performance of Gaits. }
        \label{Tab:speeds}
\centering
\renewcommand{\arraystretch}{1.3}
    \noindent\begin{tabular*}{\textwidth}{@{\extracolsep{\fill}}*{9}{c}@{}}
    \hline
            \multirow{2}{*}{\hfil Robot} & \multirow{2}{*}{\hfil Substrate}  & \multicolumn{3}{c}{Intuitive Gaits} &  \multicolumn{3}{c}{Synthesized Gaits}& \multirow{2}{*}{\hfil \% Improvement}\\
            \cline{3-5} \cline{6-8}
           & & $\|\bm{\nu}\|_{\text{max}}$ (mm/s) & $\|\bm{\nu}\|_{\text{max}}$ (BL/s) & $\|\omega\|_{\text{max}}$ ($^{\circ}$/s)   & $\|\bm{\nu}\|_{\text{max}}$ (mm/s) & $\|\bm{\nu}\|_{\text{max}}$ (BL/s) & $\|\omega\|_{\text{max}}$ ($^{\circ}$/s)  &\\
            \hline \hline
            1 & 1 & 3.68 & 0.017 & - & 6.64 & 0.030 & - & 80.4\\

            1 & 2&  2.89 & 0.013 & - & 7.15 & 0.033 & - &148.0\\
            
            1 & 3 & 2.87 & 0.013 & - & 4.25 & 0.019 & - & 47.9\\
            
            2 & 1 & 3.08 & 0.009 & - & 4.73 & 0.013 & - & 53.4\\
            \hline

            1 & 1 & - & - & 2.30 & - & - & 1.61 & -30.1 \\

            1 & 2 & - & - & 1.48 & - & - & 4.51 & 204.8 \\

            1 & 3 & - & - & 1.58 & - & - & 2.62 & 65.2 \\
            
            2 & 1 & - & - & 4.23 & - & - & 10.61 & 150.6 \\
            \hline
            \multicolumn{9}{p{.98\textwidth}}{Note: The maximum translation speed $\|\bm{\nu}\|_{\text{max}}$ and maximum rotation speed $\|\omega\|_{\text{max}}$ for each table entry are chosen among the three intuitive gaits and the three synthesized gaits (translation-dominant and rotation-dominant gaits, respectively) for each robot-substrate pair without loss of limb. The body length (BL) is calculated as the maximum dimension of the robot in the flat (unactuated) state. BL = 220 mm for robot 1 and BL = 350 mm for robot 2.}\\
            \hline
    \end{tabular*}

\end{table*}
\section{Conclusion} \label{Sec:Conclusion}
This research demonstrates that statistical graph theory has the ability to play an instrumental role in locomotion control of soft robots in unknown environments. The probabilistic Model-Free Control (pMFC) framework is an improved environment-centric data-driven approach that affords a graphical representation to {locomotion control}. The structure of the digraph is robot-specific where the vertices and edges correspond to the discretized robot states and motion primitives. The probabilistic edge weights describe the robot-environment interactions and are optimally learned using stochastic Hierholzer's algorithm. The framework also affords mathematical definition of fixed gaits as binary column vectors. This facilitates definition of locomotion gaits as simple cycles that are transformation invariant, i.e., the translation and rotation of a simple cycle are preserved irrespective of the starting vertex. 

Thereafter, gaits are synthesized by optimizing a linearized cost function formulated as a BILP with linear constraints. This formulation provides a quick and tractable solution despite the problem being NP-hard and balances exploration of gaits (through LHS) with exploitation through the exact and optimal solutions obtained using \Matlab ~\verb+intlinprog+. Multiple translation and rotation gaits are synthesized using the experimentally learned graph edge weights. The resulting synthesized gaits are experimentally validated to demonstrate that the pMFC approach can synthesize locomotion gaits on unknown surfaces through optimal learning. In summary, the method focuses on synthesizing gaits and is effective for robots lacking biological analogues for biomimetic gait selection and/or robust dynamic models that facilitate simulation for use in traditional learning methods. Therefore, a great application for this method is in soft robot locomotion control due to the significant simulation-to-reality gaps that currently exist with model-based methods.

%Benefits to the methodology include quick and tractable learning and optimization robots with four actuators, the optimality and diversity of the solutions for agile and versatile open-loop gait-based control, and the applicability of the method to different morphologies, actuators, and gait styles.

The advantages to this approach include its versatility (as it is not dependent on any model or actuator type), explainability of the tunable parameters, resilience (in the case of assumption violations and actuator failure), and the optimality and diversity of the synthesized gaits. Limiting assumptions of the model include quasi-static motion primitives, flat ground (a common condition for fixed gaits), static stability of robot states following all transitions, and binary actuation states. These assumptions also exclude dynamic and more complex gaits which could potentially provide faster speed and more diverse motions, including through actuation patterns of varying voltage/time constants and continuous input functions.  Regarding the generalizability of this approach, the pMFC framework is particularly well suited for multi-limb crawling soft robots. These robots tend to use binary actuation control, have stable robot configurations throughout movement, and involve locomotion that can be modeled (or easily adapted to be modeled) as quasi-static. The framework is also applicable to a variety of different robot constructions, morphology, actuators, and number of limbs, provided that the robot locomotion adheres to the model assumptions. 

Limitations of this methodology include the exponentially increasing complexity with added states and limbs; more complex solvers or non-exact solvers can be used to circumvent this complexity as more actuators are added, while patterns and probabilistic methods may need to be exploited to decrease learning time. For example, direct application of our method to robots with five or six actuators would result in graph edge learning times (i.e., continuous experimental robot operation) of 22.5 minutes and 47.3 minutes, respectively. 
%The latter refers to the quasistatic assumption that excludes dynamic gaits which could provide faster speed, the assumption of stability for all robot configurations which is not true for all legged systems (though the framework can be adapted accordingly), the necessary condition of flat ground (a common condition for fixed gaits) and the binary nature of the actuation which is limited to one speed / voltage per experiment and does not account for alternate actuation patterns that may produce improved motion. 

The experimental results performed on a planar surface are extremely promising. They validate the idea of a data-driven probabilistic model-free control framework to optimally learn about an unknown surface and synthesize the locomotion gaits in real time. The synthesized gaits have the potential to serve as starting points for other locomotion control techniques, including Reinforcement Learning (RL) and Central Pattern Generators (CPG). Furthermore, the proposed framework is sufficiently generic to be extended and applied to a diverse family of multi-limb soft robots without requiring significant modifications. In summary, it is particularly relevant for soft robots where prior knowledge of their behavior is limited. Both the assumption of planar uniform locomotion surface and temporally constant edge weights of the graph limit the applicability of this method to more challenging terrains. Subsequently, the evolution of the framework for complex and dynamic environments will involve combining this methodology with feedback control. Future work will focus on applying the synthesized gaits to soft robot locomotion control and path planning; utilizing learning methods to further tune and optimize actuation and gait parameters; developing modifications to reduce learning time (e.g., through Gaussian Processes or reinforcement learning); evolving the methodology for application in more complex environments and dynamic systems; and investigating the factors contributing to increased variance in soft robot locomotion behavior, such as complex frictional interactions and the influence of the tether. 
% 
% %
\section*{Acknowledgment}
The authors would like to thank Patricio Vela and Alexander Chang for invaluable discussions.

% \input{GaitSynthesis.bbl}
%\bibliographystyle{ieeetr}
%\bibliography{references,references2}

\end{document}